\documentclass{article}
\usepackage{microtype}
\usepackage{graphicx}
\usepackage{subcaption}
\usepackage{booktabs} 
\usepackage{amsmath}
\usepackage{amssymb}
\usepackage{amsthm}
\usepackage{ts_command} 
\usepackage[fleqn]{mathtools}
\usepackage{float}
\usepackage{bm}
\usepackage{bbold}
\usepackage{aliascnt}
\usepackage{kotex}
\usepackage{titletoc}
\usepackage{hyperref}
\usepackage{tikz}
\usepackage{enumitem}
\usetikzlibrary{arrows.meta,angles,quotes}

\usepackage[accepted]{icml2026}

\usepackage[capitalize,noabbrev]{cleveref}
\usepackage{blindtext}
\usepackage{duckuments}
\usepackage{tcolorbox}
\usepackage[euler]{textgreek}
\usepackage{subcaption}

\theoremstyle{plain}
\newtheorem{theorem}{Theorem}[section]

\newaliascnt{lemma}{theorem}
\newtheorem{lemma}[lemma]{Lemma}
\aliascntresetthe{lemma}
\crefname{lemma}{Lemma}{Lemmas}
\Crefname{lemma}{Lemma}{Lemmas}

\newaliascnt{proposition}{theorem}
\newtheorem{proposition}[proposition]{Proposition}
\aliascntresetthe{proposition}
\crefname{proposition}{Proposition}{Propositions}
\Crefname{proposition}{Proposition}{Propositions}

\newaliascnt{corollary}{theorem}
\newtheorem{corollary}[corollary]{Corollary}
\aliascntresetthe{corollary}
\crefname{corollary}{Corollary}{Corollaries}
\Crefname{corollary}{Corollary}{Corollaries}

\theoremstyle{definition}
\newaliascnt{definition}{theorem}
\newtheorem{definition}[definition]{Definition}
\aliascntresetthe{definition}
\crefname{definition}{Definition}{Definitions}
\Crefname{definition}{Definition}{Definitions}

\newaliascnt{assumption}{theorem}
\newtheorem{assumption}[assumption]{Assumption}
\aliascntresetthe{assumption}
\crefname{assumption}{Assumption}{Assumptions}
\Crefname{assumption}{Assumption}{Assumptions}

\theoremstyle{definition}
\newaliascnt{remark}{theorem}
\newtheorem{remark}[remark]{Remark}
\aliascntresetthe{remark}
\crefname{remark}{Remark}{Remarks}
\Crefname{remark}{Remark}{Remarks}

\definecolor{sectorblue}{RGB}{144,149,217}
\definecolor{sectorred}{RGB}{202,135,134}

\usepackage[textsize=tiny]{todonotes}
\icmltitlerunning{Over-Alignment vs Over-Fitting: The Role of Feature Learning Strength in Generalization}

\begin{document}

\twocolumn[
  \icmltitle{Over-Alignment vs Over-Fitting:\\ The Role of Feature Learning Strength in Generalization}
  \icmlsetsymbol{equal}{*}

  \begin{icmlauthorlist}
    \icmlauthor{Taesun Yeom}{yyy}
    \icmlauthor{Taehyeok Ha}{yyy}
    \icmlauthor{Jaeho Lee}{yyy}
  \end{icmlauthorlist}
    \icmlaffiliation{yyy}{Pohang University of Science and Technology (POSTECH), Pohang, South Korea}


  \icmlcorrespondingauthor{Jaeho Lee}{jaeho.lee@postech.ac.kr}

  \icmlkeywords{Machine Learning, ICML}

  \vskip 0.3in
]

\printAffiliationsAndNotice{}  
\begin{abstract}
Feature learning strength (FLS), i.e., the inverse of the effective output scaling of a model, plays a critical role in shaping the optimization dynamics of neural nets. While its impact has been extensively studied under the asymptotic regimes---both in training time and FLS---existing theory offers limited insight into how FLS affects generalization in practical settings, such as when training is stopped upon reaching a target training risk. In this work, we investigate the impact of FLS on generalization in deep networks under such practical conditions. Through empirical studies, we first uncover the emergence of an \textit{optimal FLS}---neither too small nor too large---that yields substantial generalization gains. This finding runs counter to the prevailing intuition that stronger feature learning universally improves generalization. To explain this phenomenon, we develop a theoretical analysis of gradient flow dynamics in two-layer ReLU nets trained with logistic loss, where FLS is controlled via initialization scale. Our main theoretical result establishes the existence of an optimal FLS arising from a trade-off between two competing effects: An excessively large FLS induces an \emph{over-alignment} phenomenon that degrades generalization, while an overly small FLS leads to \emph{over-fitting}.
 
\end{abstract}

\section{Introduction}\label{sec:intro}
One of the key mysteries of deep learning is its ability to find well-generalizing solutions, even when severely overparametrized \citep{zhang2017understanding}. Because this behavior appears to contradict classical learning theory, a number of explanations have been proposed. A leading hypothesis is based on \textit{implicit bias}---the tendency of neural nets to favor learning certain solutions, even in the absence of an explicit regularization \citep{vardi2023implicit}. A growing body of work investigates the origins of this phenomenon, attributing it to various factors such as gradient-based optimization dynamics \citep{soudry2018implicit, Lyu2020Gradient}, model architecture \citep{teney2024neural, cao2023implicit}, or hyperparameters, e.g., the learning rate \citep{even2023sgd, wu2023implicit}.

Among many factors, the \emph{feature learning strength} (FLS) stands out as particularly important \citep{woodworth2020kernel,atanasov2025the}. FLS is defined as the inverse of the effective scaling applied to the model output, which is typically controlled by the initialization scale or an explicit output multiplier, such as the softmax temperature. Varying FLS leads to two qualitatively distinct training regimes. When FLS is large, features evolve nonlinearly throughout training, reflecting genuine feature learning \citep{woodworth2020kernel, atanasov2022neural}. In contrast, when FLS is small, training closely resembles kernel learning, with features remaining largely fixed \citep{jacot2018neural, chizat2019lazy}. A substantial body of prior works has shown that analyzing these two regimes yields valuable insights into the optimization dynamics and generalization of deep learning \citep{arora2019fine, allen2019convergence,sclocchi2023dissecting,atanasov2025the,domine2025from,simon2026there}.

However, our theoretical understanding of how FLS affects generalization remains poorly aligned with practical observations. This gap is twofold. First, existing theories offer little concrete guidance for tuning FLS-related hyperparameters to achieve optimal generalization. Their conclusions often reduce to the coarse message that ``stronger feature learning improves generalization'' \citep{woodworth2020kernel,atanasov2025the}, whereas in practice, intermediate levels of feature learning---neither too weak nor too strong---tend to perform best \citep{agarwala2023temperature,masarczyk2025unpacking}. Second, much of the theoretical literature focuses on properties of the limiting solution, which is rarely relevant in real training settings. In practice, training is typically halted once a target training risk is reached or a fixed optimization budget is exhausted. Since stronger feature learning generally requires more optimization steps, conclusions drawn from the limiting regime can be misleading when applied to finite-time training \citep{woodworth2020kernel}.

\textbf{Contribution.} In this work, we aim to narrow the gap between the FLS-based theoretical understanding and practice, by studying the following two research questions:
\begin{tcolorbox}[boxrule=1pt,arc=0.3em,boxsep=-0.5mm]
\begin{itemize}[leftmargin=*,topsep=0pt,parsep=0pt]
    \item \textbf{Q1.} Does stronger feature learning always help generalization, under practical setups?
    \item \textbf{Q2.} If not, can we explain such a gap theoretically?
\end{itemize}
\end{tcolorbox}
To address \textbf{Q1}, we conduct experiments on image classification tasks using VGG \citep{vgg} and ResNet \citep{he2016deep} architectures. We find that, even when models achieve perfect training accuracy or attain the same training risk, their generalizability differs significantly depending on the FLS. Surprisingly, across all datasets and architectures we consider, excessively large FLS values consistently harm generalization, and an intermediate optimal FLS emerges, in contrast with the prevailing belief (\cref{fig:fig_1}). 
Moreover, we find that the benefit of tuning FLS grows with task complexity: as the dataset's intrinsic dimensionality increases, selecting the optimal FLS yields increasingly large generalization gains.


Motivated by these empirical results, we proceed to address \textbf{Q2} by analyzing the optimization dynamics induced by varying the FLS. Building on recent work that studies gradient flow dynamics in the strong feature learning regime \citep{min2024early,boursier2025early}, we first establish that FLS---equivalently, the initialization scale in our setting---critically governs the angular deviation of the weights (or the induced predictor) throughout training (\Cref{lem:phase1_lb,lem:phase2_bound}). Leveraging this characterization, we derive an error bound for binary Gaussian mixtures and decompose it into two distinct components: a data-dependent \textit{over-alignment} term and an \textit{over-fitting} term (\Cref{thm:rb}). This decomposition exposes a fundamental trade-off between strong and weak feature learning and implies the existence of a data-dependent optimal FLS. Together, these results capture what is observed in practice, providing a fresh perspective on how FLS shapes generalization.


In summary, our work provides both empirical and theoretical results providing insights into practical implicit bias in classification tasks, which has been largely unexplored in prior studies. In particular, we emphasize the role of feature learning strength in shaping generalization behavior in deep learning. We hope our work serves as a step toward demystifying the generalization ability of neural networks.

\begin{figure}[!t]
    \centering
    \includegraphics[width=0.75\linewidth]{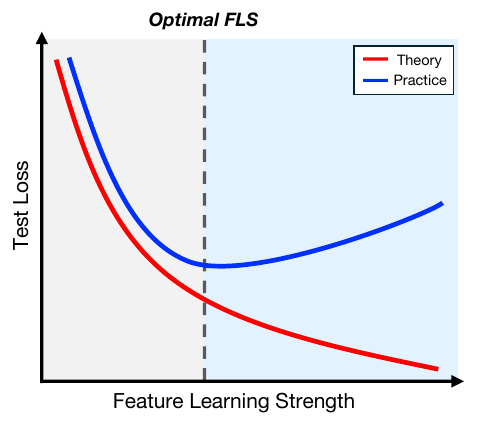}
\vspace{-0.5em}    \caption{\textbf{Emergence of an optimal FLS.} We empirically observe that, under standard classification setups, stronger feature learning tends to degrade generalization performance of the model when it exceeds a certain threshold, implying the existence of an ``optimal FLS'' that is neither too large nor too small.} 
    \label{fig:fig_1}
\end{figure}



\textbf{Notation.} Scalars, vectors, and matrices are denoted by lowercase (e.g., $a$), bold lowercase (e.g., $\mathbf{a}$), and bold uppercase letters (e.g., $\mathbf{A}$), respectively. The norm $\|\cdot\|$ denotes the Euclidean norm for vectors, and the spectral norm for matrices. $\|\cdot\|_F$ denotes the Frobenius norm. $\mathbf{I}_n$ denotes the $n\times n$ identity matrix, $\angle(\cdot,\cdot)$ denotes the angle between two vectors, $\Phi(\cdot)$ is the Gaussian CDF, and $\bbR_+:=\{x\in \bbR:x>0\}$.
\begin{figure*}[!htbp]
    \centering
    \begin{subfigure}[t]{0.25\textwidth}
        \centering
        \includegraphics[width=\linewidth]{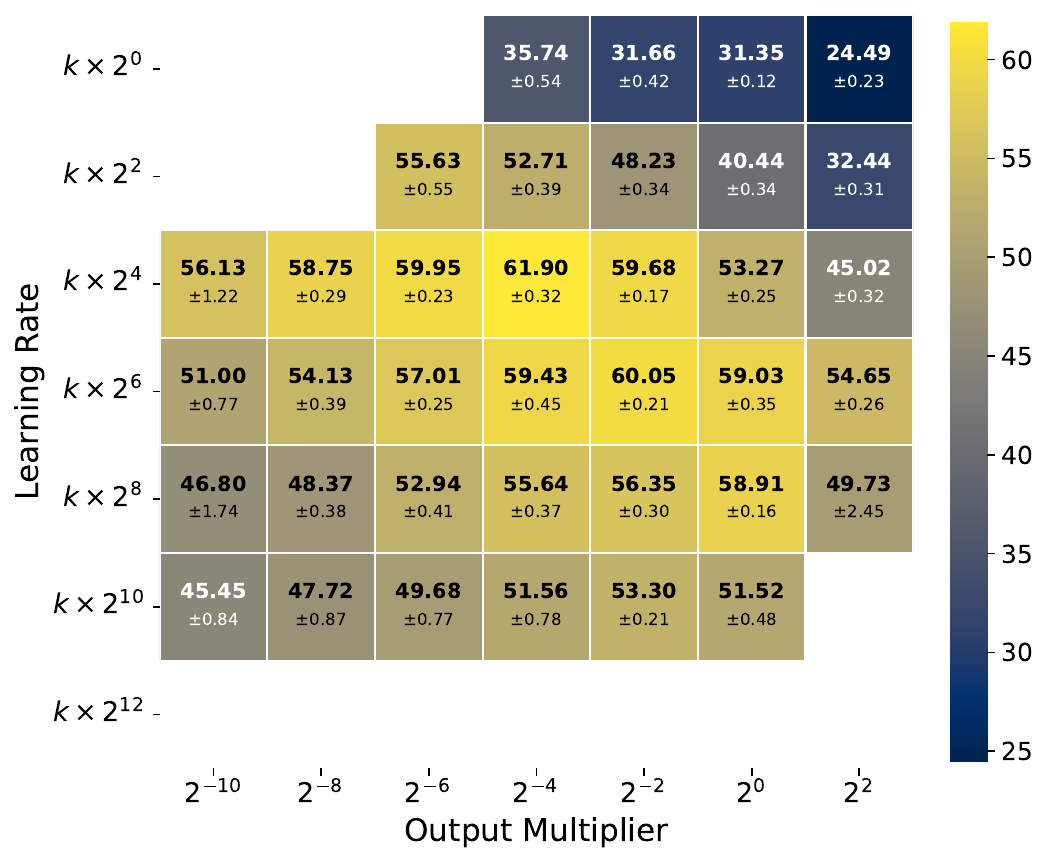}
        \caption{VGG19}
        \label{fig:main_cifar100_vgg19}
    \end{subfigure}\hfill
    \begin{subfigure}[t]{0.25\textwidth}
        \centering
        \includegraphics[width=\linewidth]{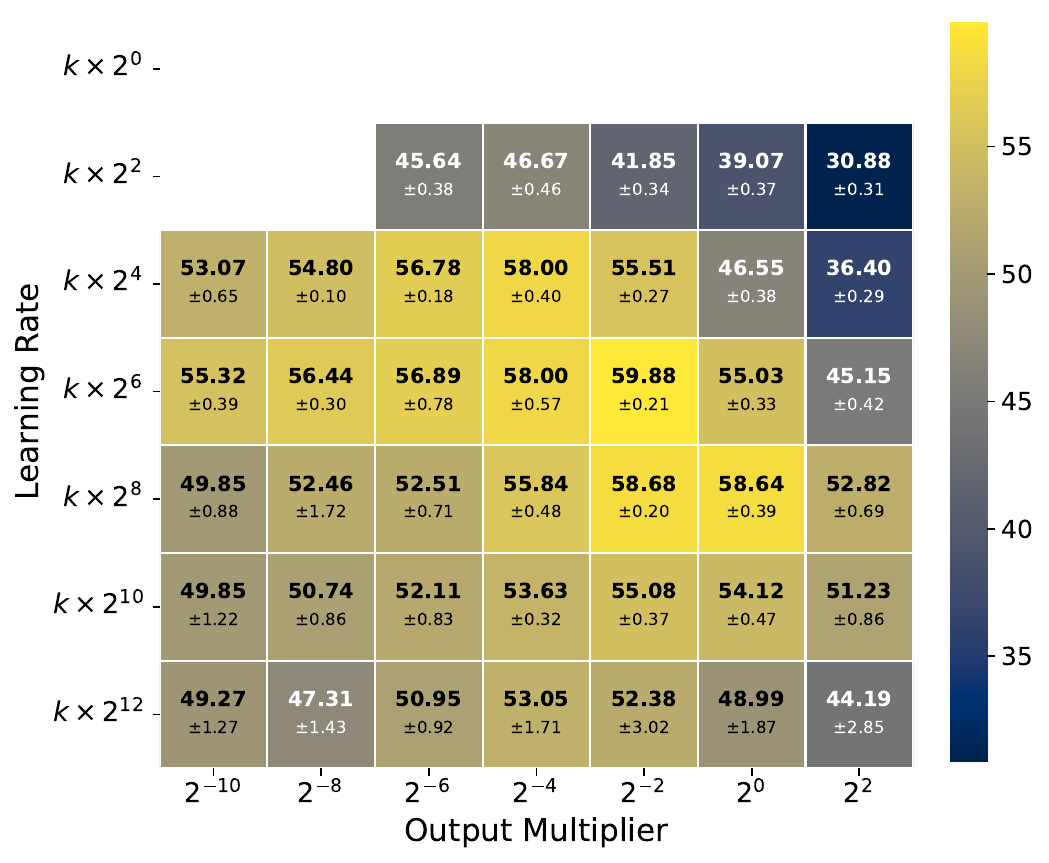}
        \caption{ResNet18}
        \label{fig:main_cifar100_resnet18}
    \end{subfigure}\hfill
    \begin{subfigure}[t]{0.25\textwidth}
        \centering
        \includegraphics[width=\linewidth]{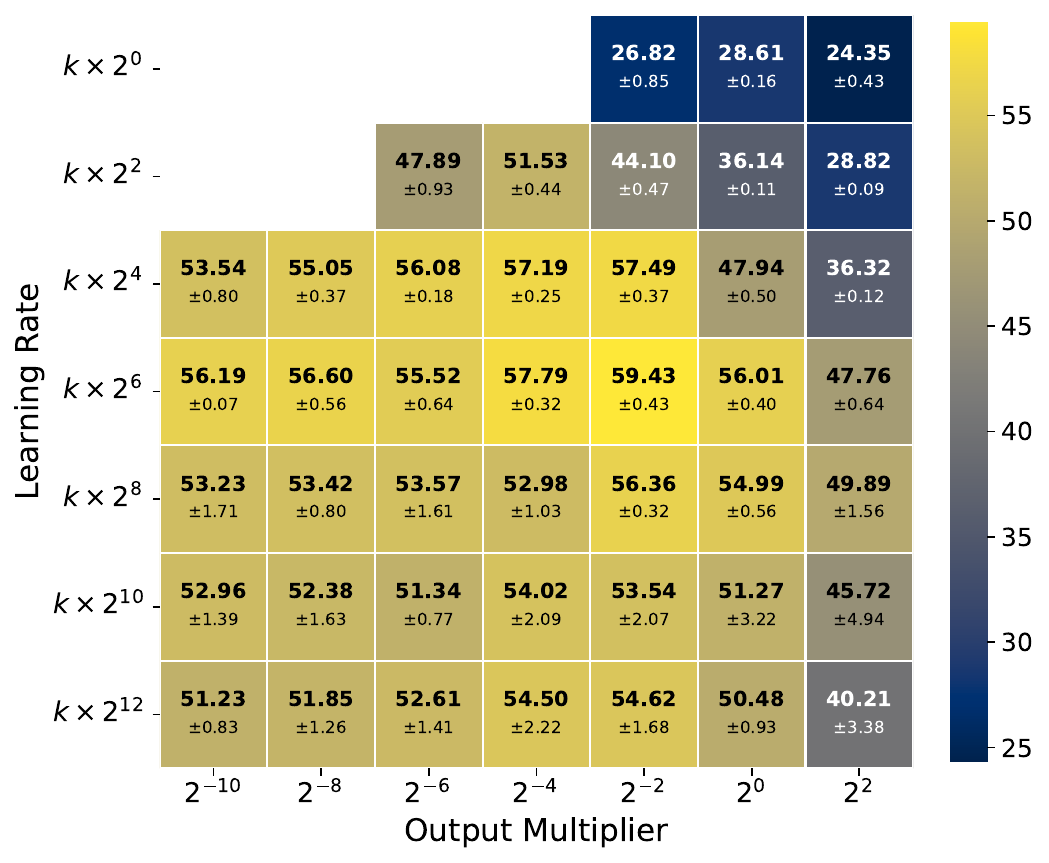}
        \caption{ResNet34}
        \label{fig:main_cifar100_resnet34}
    \end{subfigure}\hfill
    \begin{subfigure}[t]{0.25\textwidth}
        \centering
        \includegraphics[width=\linewidth]{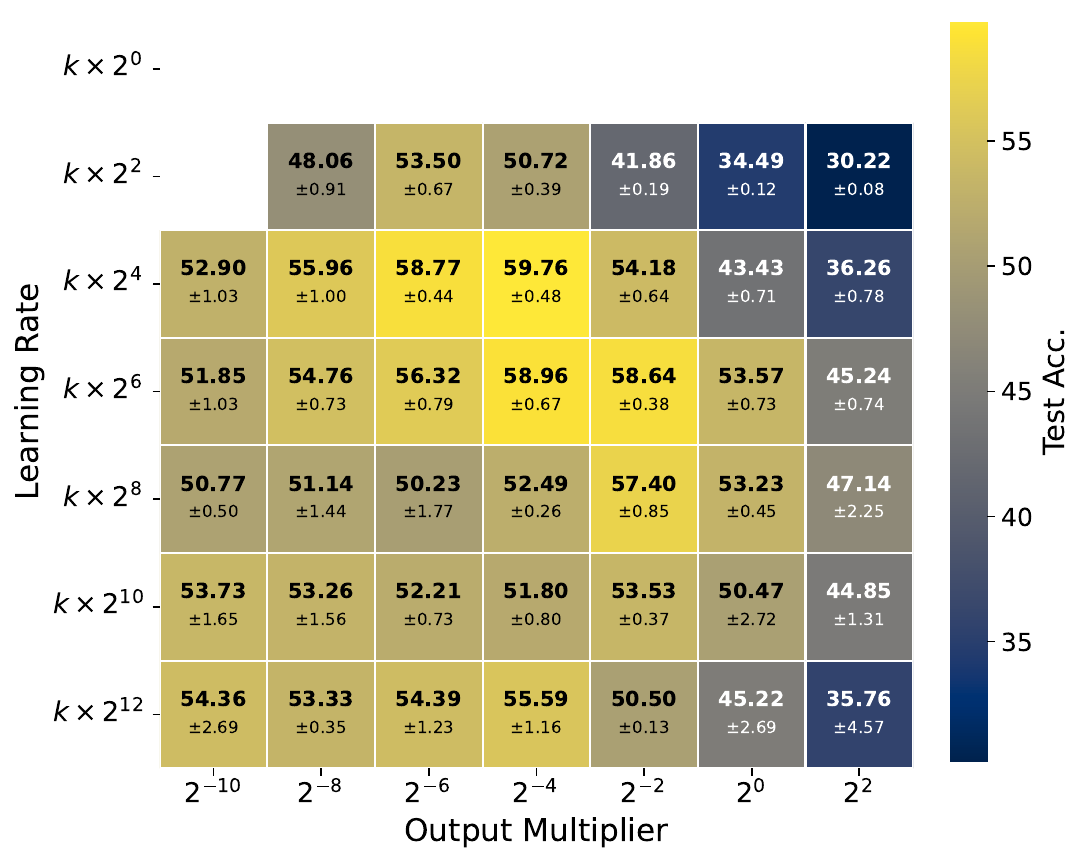}
        \caption{ResNet50}
        \label{fig:main_cifar100_resnet50}
    \end{subfigure}

    \vspace{-0.5em}
    \caption{\textbf{Emergence of optimal FLS in generalization.} Peak test accuracy (\%) of various networks trained on CIFAR-100. Blank grids indicate cases where, for at least one of the three seeds, the training accuracy does not exceed 99\%. For readability, the learning rate axis is labeled using the pre-normalized values (i.e., $\eta$), where $k=6.4\times10^{-4}$. Further details can be found in \cref{app:emp_exp_details}.}
    \label{fig:main_cifar100}
\end{figure*}

\section{Related Work}
\textbf{FLS in deep learning.} The optimization behavior of neural nets is highly sensitive to the \textit{feature learning strength} (FLS), i.e., the inverse of the effective scaling of the model output \citep{chizat2019lazy,atanasov2025the}. This scaling can be controlled through the weight initialization scheme \citep{woodworth2020kernel, kunin2024get, yeom2025fast} or by explicitly rescaling the model output \citep{chizat2019lazy, atanasov2025the}. Varying this scale induces a transition between two distinct regimes: a feature learning regime at small scales, and kernel regime at large scales. In the strong feature learning regime, the training dynamics are highly nonlinear, inducing phenomena such as neuron alignment \citep{maennel2018gradient, min2024early, boursier2025early} or saddle-to-saddle dynamics \citep{jacot2021saddle, kunin2025alternating}. In contrast, in the kernel regime, the network behaves approximately linearly with respect to its initialization, with little updates in features \citep{jacot2018neural, chizat2019lazy}. Majority of these works aim to provide a clear picture of learning dynamics itself, induced by the gradient descent. Our work, on the other hand, focuses on the generalization performance of the models induced by these training dynamics.

\textbf{Feature learning and generalization.} A widely held belief is that the stronger feature learning always leads to better generalization in standard---i.e., in-distribution---classification. In such regime, the training dynamics result in sparse features, which in turn leads to a better generalization \citep{woodworth2020kernel,li2021implicit,stoger2021small,li2023the}. Several prior works establish concrete connections to the generalization: \citet{sclocchi2023dissecting} study the phase diagram varying SGD noise and feature learning strength; in an online learning setup, \citet{atanasov2025the} empirically analyze the generalization behavior across varying feature learning strength and learning rate; most similar to our work, \citet{petrini2022learning} study generalization behavior in a spherical regression task. This work, however, mainly considers the  two extreme choices of FLS in infinite-width networks: mean-field \citep{mei2018mean} vs. neural tangent kernel. In contrast, our work primarily focuses on characterizing the \textit{optimal} feature learning strength which lies between these regimes, in the classification setup.

\begin{figure*}[!t]
    \centering
    \begin{subfigure}[t]{0.33\textwidth}
        \centering
        \includegraphics[width=\linewidth]{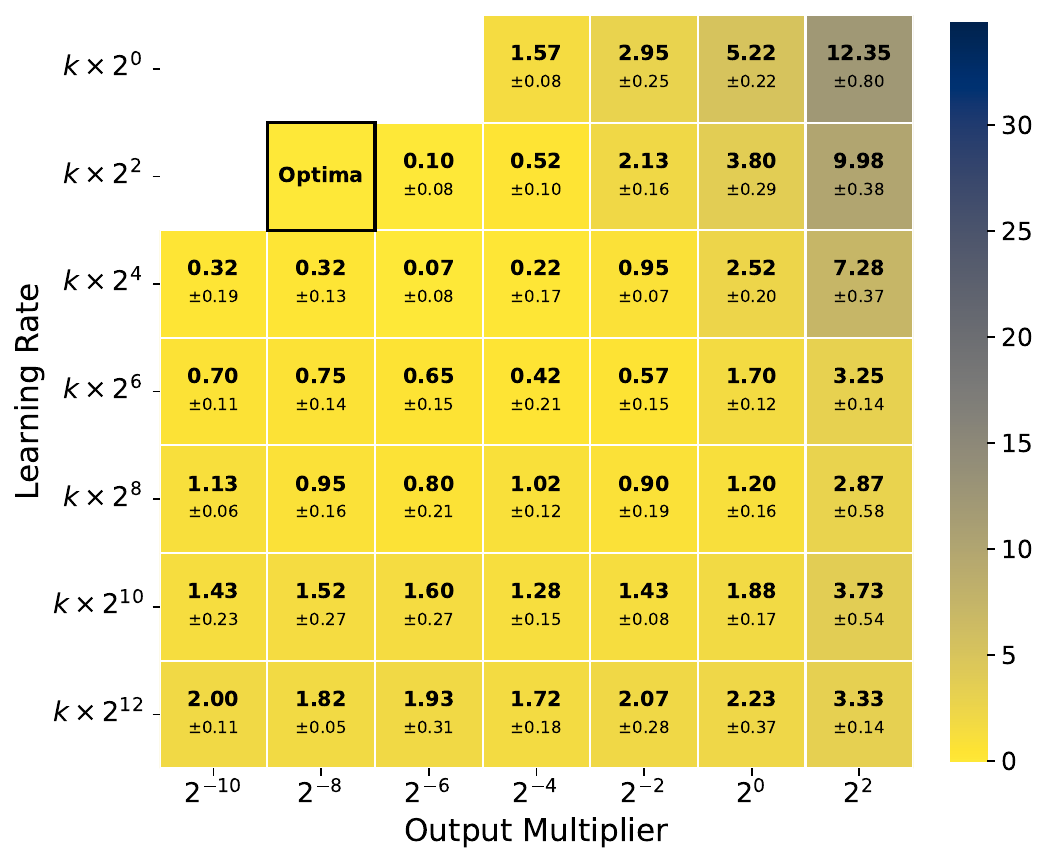}
        \caption{Effective dimension: 32}
        \label{fig:main_biggan_32}
    \end{subfigure}\hfill
    \begin{subfigure}[t]{0.33\textwidth}
        \centering
        \includegraphics[width=\linewidth]{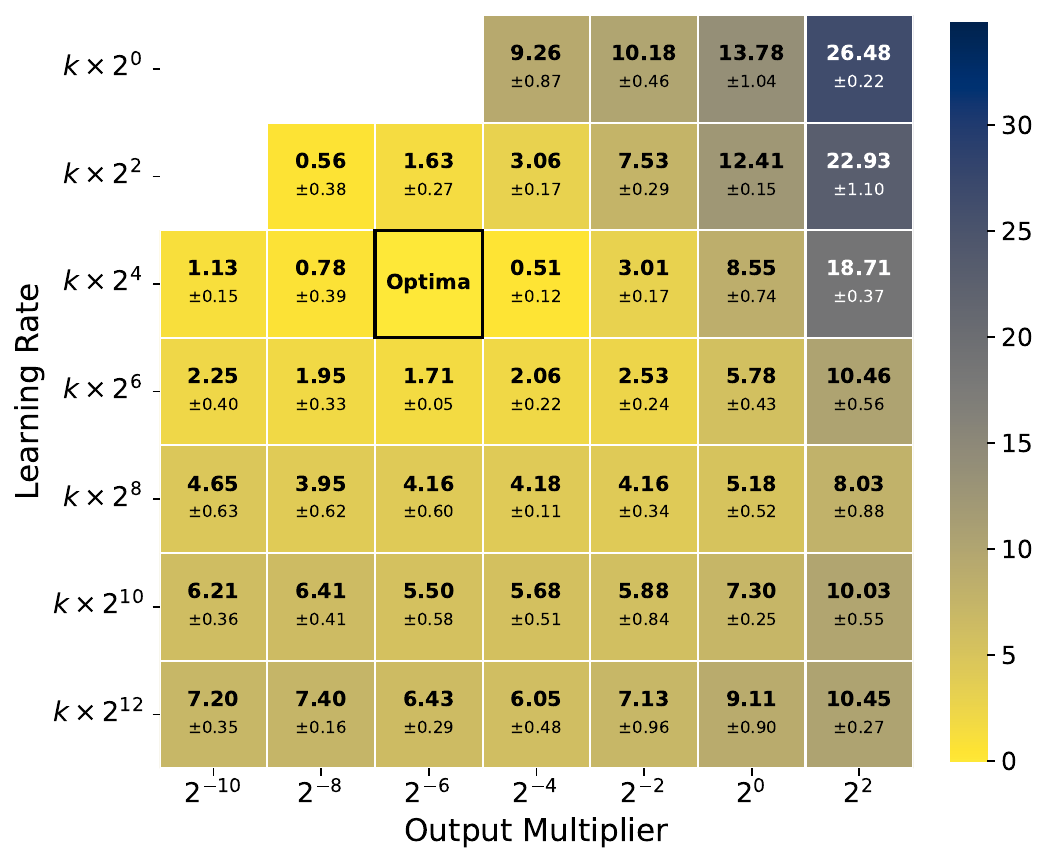}
        \caption{Effective dimension: 64}
        \label{fig:main_biggan_64}
    \end{subfigure}\hfill
    \begin{subfigure}[t]{0.33\textwidth}
        \centering
        \includegraphics[width=\linewidth]{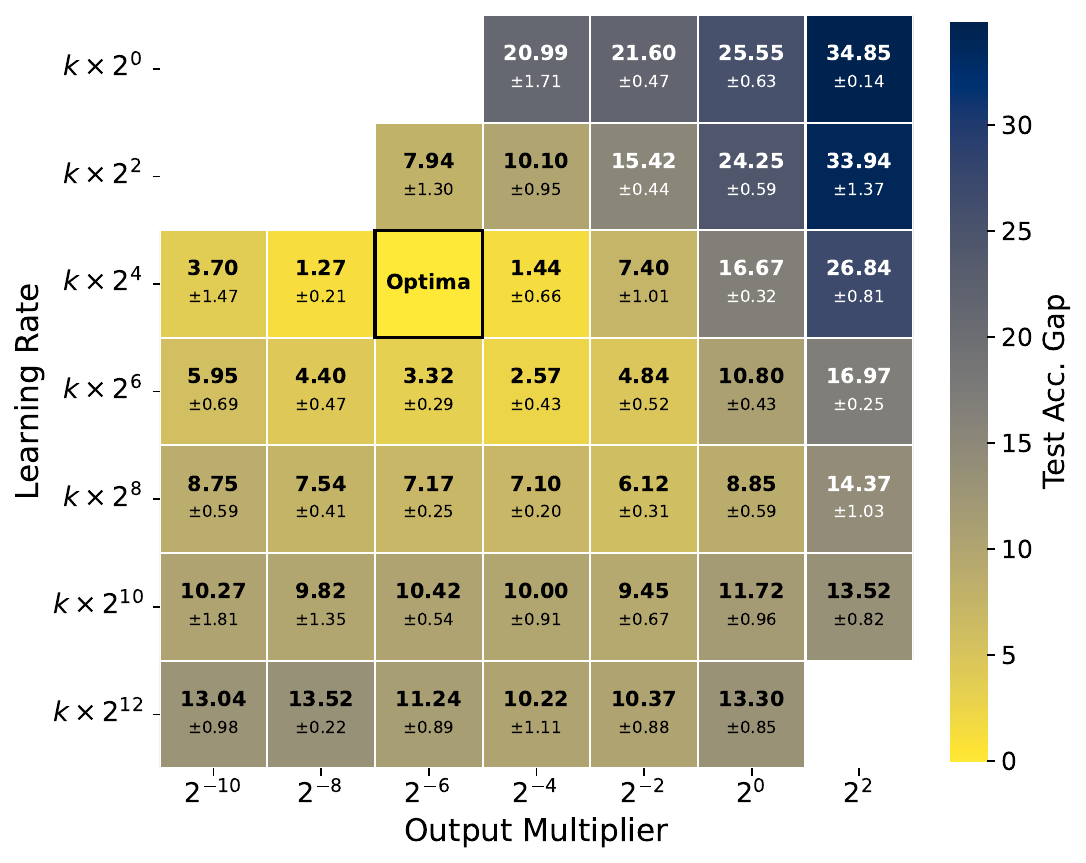}
        \caption{Effective dimension: 128}
        \label{fig:main_biggan_128}
    \end{subfigure}
\vspace{-0.5em}
    \caption{\textbf{Optimal FLS is more beneficial for the difficult dataset.} The gap of the peak test accuracy (\%) of ResNet18 trained on a BigGAN-generated dataset against the best FLS. We have varied the effective dimensionality of the samples generated, to control the task difficulty. Blank grids indicate the cases where, for at least one of the three seeds, the training accuracy does not exceed 99\%. For readability, the learning rate axis is labeled using the pre-normalized values (i.e., $\eta$), where $k=6.4\times10^{-4}$.}
    \label{fig:main_biggan}
\end{figure*}

\section{Empirical Takes on FLS \& Generalization}
\label{sec:dnn}

In this section, we empirically study how FLS affects generalization in deep networks and present a nontrivial observation that has not been discussed in prior literature: \emph{Larger FLS can hurt generalization in standard classification}.

\textbf{Controlling FLS in deep networks.} To control FLS, in this section we consider the following scaling rule for neural networks. Suppose we train a neural network $f$ via gradient descent, with learning rate $\eta$. Here, we rescale the function as $f \mapsto c f$ for some output multiplier $c>0$ and set the learning rate to $\eta/c$. Here, a smaller $c$ corresponds to a larger FLS. This scaling scheme is analogous to the way FLS is controlled in a widely used parameterization scheme, the so-called \textit{maximal update parameterization} \citep{geiger2020disentangling,yang2021tpiv,bordelon2022self} and temperature scaling \citep{agarwala2023temperature,masarczyk2025unpacking}; see \cref{app:flsvsmup} for more details.

\subsection{Emergence of the sweet spot in generalization}\label{ssec:sweet_spot} In previous empirical work, FLS is regarded as a trade-off quantity between computational resources and generalization \citep{woodworth2020kernel}; that is, larger FLS, which typically requires longer training time, leads to better generalization. To check whether this holds for deep networks, we conduct standard image classification experiments using widely used architectures: VGG19 with batch normalization, and ResNet\{18, 34, 50\}. We use CIFAR-10 and CIFAR-100 as representative datasets for the classification task. Here, we present the results of the CIFAR-100; see \cref{app:cifars} for additional experimental results, including CIFAR-10.

Since the FLS parameter $c$ is heavily influenced by the choice of the learning rate $\eta$ \citep{atanasov2025the}, we sweep over different values of $c$ and $\eta$ and present test-accuracy heatmaps on the $(c,\eta/c)$ plane in \cref{fig:main_cifar100}. Each grid point is averaged over three random seeds, and we report the mean with its standard deviation. For a fair comparison, all networks are trained until they achieve near-perfect training accuracy (i.e., above 99 percent).

Taking a closer look at \cref{fig:main_cifar100}, we observe that, as reported in prior work, training with a larger $c$ tends to degrade generalization \citep{woodworth2020kernel,mehta2021extreme}. On the other hand, the interesting observation here is that there exists an \emph{optimal FLS}: output multipliers below the optimum also hurt generalization, and this trend holds across all networks. Notably, these results have been obtained when the networks have already reached their peak generalization performance (i.e., further training leads to overfitting rather than improvement). Therefore, our findings directly refute the common claim that ``with sufficiently long training, a larger FLS is better in classification.''

Moreover, our empirical results reveal the practical benefits of using an optimal FLS. In typical settings, hyperparameter tuning does not explicitly include the FLS. However, in such cases---for example, when only the learning rate is tuned---one may fail to reach best generalization, even if the model appears well-optimized within the chosen search space. For instance, in \cref{fig:main_cifar100_resnet50}, the best test accuracy achieved with the default scale (i.e., $c=2^0$) is surpassed by that achieved with the optimal FLS (i.e., $c=2^{-4}$) by about 6\% (e.g., 53.57\% vs. 59.76\%). These observations suggest that FLS should be treated as a critical axis for hyperparameter tuning, alongside conventional choices.

\subsection{Benefits of optimal FLS and the task difficulty}

Having observed an optimal FLS in standard classification tasks, we now examine how this behavior changes under varying conditions, such as datasets with different levels of difficulty. In this subsection, we focus on the relationship between FLS and the \emph{intrinsic dimensionality} of a dataset, a notion for characterizing its complexity \citep{ansuini2019intrinsic,gong2019intrinsic}.

Recent work by \citet{pope2021the} studies the impact of a dataset's intrinsic dimensionality on generalization by explicitly varying the effective dimensionality of the latent vector: Specifically, by zeroing out a predefined subset of indices in each input random vector of deep generative models. They find that higher intrinsic dimensionality increases sample complexity. Following this work, we generate a synthetic dataset with 10 classes from the dog category in ImageNet using pretrained BigGAN \citep{brock2019large}. We vary the effective dimensionality among 32, 64, and 128; for brevity, we refer to each dataset by its effective dimensionality. See \cref{app:emp_exp_details} for more details.


In \cref{fig:main_biggan}, we present a heatmap of ``the gap in peak test accuracy'' for ResNet18 across datasets with different effective dimensionalities. The gap is defined as the difference between the highest accuracy on the heatmap (i.e., at `Optima' in each heatmap) and the accuracy of each grid cell. Consistent with \cref{fig:main_cifar100}, we again observe a ``sweet spot'' of generalization, and the network in this regime consistently outperforms other configuration across datasets. Most notably, as the effective dimensionality increases (i.e., as the task becomes more difficult), the benefit of using the optimal FLS becomes larger, as can be seen from the increased gap in the test accuracy.

As a takeaway, these results highlight the practical value of identifying the optimal FLS. The advantage of doing so is particularly more pronounced for challenging tasks.

We refer readers to \cref{app:biggans} for additional results, covering different architectures and evaluation metrics.



\section{Problem Formulation}\label{sec:theories}

To demystify the internal mechanisms of the phenomenon  observed in \cref{sec:dnn}---the emergence of the optimal FLS---we move onto a theoretical analysis. This section describes the problem formulation and relevant preliminaries, based on which we establish the theoretical results in \cref{sec:theory}.


\subsection{Preliminaries: Two-phase dynamics}

First, we describe some known results about the optimization dynamics of the models under various FLS. In particular, we focus on the case of \emph{large FLS}---the models with small feature learning strengths can be understood easily, as they can be approximated by their linearized functionals.

For large FLS, models with positively homogeneous activations (e.g., ReLU) exhibit an interesting learning dynamics. Roughly, their training consists of two distinct phases.
\begin{itemize}[leftmargin=*,topsep=0pt,parsep=0pt,itemsep=1pt]
    \item \emph{Phase 1: Neuron alignment.} In the early phase, weights are aligned to particular direction, with only slight growth in output scale and marginal decrease of the loss. This phenomenon is known as \emph{neuron alignment} (or \emph{directional convergence}) \citep{maennel2018gradient, ji2018gradient}.
    \item \emph{Phase 2: Margin maximization.} After neurons are aligned, the loss begins to decrease more noticeably. As the activation patterns have been stabilized in \emph{Phase 1}, the model in phase 2 behaves approximately as a linear model, whose optimization is well understood \citep{ji2018gradient,arora2018a,nacson2019convergence,Lyu2020Gradient}. 
\end{itemize}


In \cref{sec:theory}, we will demonstrate that an analysis on the large FLS regime alone suffices to establish the optimality of the intermediate FLS, which is neither too large nor too small. Nevertheless, we will also show that our theory can be extended to the case of small FLS, in agreement with the empirical results presented.

\subsection{Formulation}\label{ssec:setup}
Now we describe the theoretical setup we consider. 

Consider a binary classification task with a $d$-dimensional input $\mathbf{x} \in \mathbb{R}^d$ and a binary label $y \in \{-1,+1\}$. As the classifier, we consider a bias-free two-layer neural network using the ReLU activation:
\begin{align}
    f(\bfx;\theta) = \sum_{j=1}^h v_j \sigma(\langle \bfw_j, \bfx \rangle).
\end{align}
Here, $\sigma(x) = \max\{0,x\}$ denotes the ReLU activation and $\theta := (\mathbf{W},\mathbf{v})$ denotes the tuple of parameters, with the first layer parameters $\mathbf{W} =[\bfw_1, \cdots, \bfw_h]\in \bbR^{d \times h} $ and the second layer parameter $\mathbf{v} =[v_1, \cdots, v_h]^\top\in \bbR^{h}$.


The training dataset consists of $n$ independently drawn samples $D = \{(\bfx_i,y_i)\}_{i=1}^n$ and define $\bfx_{\max}:=\max_{i}\|\bfx_i\|$. Using this dataset, we minimize the \textit{training risk} over the samples, defined as the following.
\begin{align}
\hat{L}(\theta) := \frac{1}{n}\sum_{i=1}^n \ell(f(\bfx_i;\theta),y_i).
\end{align}
We use the logistic loss, i.e., $\ell(\hat{y},y) = \log(1+\exp(-y\hat{y}))$, where $\hat{y}$ denotes the model output.\footnote{We note, however, that all arguments in this paper holds for any exponentially-tailed loss function \citep{soudry2018implicit}.} The training risk is minimized via gradient flow. More concretely, we conduct
\begin{align}
    d\bfW/dt \in -\partial_{\bfW} \hat{L}(\theta), \qquad  d\bfv/dt \in -\partial_{\bfv} \hat{L}(\theta),
\end{align}
where $\partial$ denotes the Clarke subdifferential \citep{clarke1975generalized}.

\textbf{Feature learning strength.} The FLS is controlled with the \textbf{\textit{scale factor}} $\alpha > 0$ of the initialization. Precisely, the first layer weight $\mathbf{W}$ is initialized in two steps: First, we sample the entries of the reference matrix $\mathsf{W}$ from some distribution $\mathcal{P}$. Then, we scale this weight by $\alpha$ to initialize $\mathbf{W}$, i.e.,
\begin{align}
\mathbf{W}(0)  =\alpha \mathsf{W}.
\end{align}
Here, we define the quantity $\mathsf{W}_{\max}:=\max_{j\in [h]}\|\mathsf{W}_j\|$.
The second layer weights are determined as 
\begin{align}
v_j(0) \sim \mathrm{Unif}(\{\|\bfw_j(0)\|, -\|\bfw_j(0)\|\})
\end{align}
This initialization scheme enables us to utilize existing tools for gradient flow analysis, e.g., balancedness or sign preservation properties. See \cref{lem:gf_property} for details.

Note that controlling FLS via the initialization scale factor $\alpha$ is essentially equivalent to using an output multiplier, as in \cref{sec:dnn}. We formally show this point in \cref{prop:equivalence}.


\textbf{Data model.} As the data-generating distribution, we consider a simple Gaussian mixture in $\mathbb{R}^d$ with two classes. Precisely, each sample $(\mathbf{x}_i, y_i)$ is generated as
\begin{align}
\bfx_i =  \kappa y_i \bfs_i + \sigma \bfz_i,\qquad \bfz_i \sim \mathcal{N}(\mathbf{0}, \mathbf{I}_d),\label{eq:data_model}
\end{align}
where $\bfs_i \in \{\bfs_+, \bfs_-\}$ is the signal vector with $\|\bfs_i\|=1$, chosen according to the corresponding class label. The separability parameter $\kappa \in \mathbb{R}_+$ and the noise level $\sigma \in \mathbb{R}_+$ control the signal strength and the noise magnitude, respectively. Note that in our theoretical analyses, we set $\kappa=1$ and consider symmetric Gaussian mixture for simplicity, i.e., $\bfs_+ = \bfs_-$.

We further assume that our training dataset satisfies the following condition 
\citep{phuong2021the}.

\begin{assumption}[Orthogonal separability]\label{assm:ortho_sep}
There exists some constant $\lambda\in \mathbb{R}_{+}$ such that for all distinct pairs of training data $(\bfx,y),(\tilde{\bfx},\tilde{y}) \in D$, the following holds:
\begin{align}
    \frac{y\tilde{y}\langle\bfx,\tilde{\bfx}\rangle}{\|\bfx \| \|\tilde{\bfx}\|} \ge \lambda.
\end{align}
\end{assumption}
\cref{assm:ortho_sep} is a sufficient condition under which the early-phase ODE admits an \textit{interpretable} stationary point (see \cref{lem:informal_min}), a property that does not hold for general datasets \citep{glasgow2024sgd,boursier2025early}. Since our primary focus is the generalization, we impose this assumption only on the training set, and not the population distribution. Here, we can ensure that \cref{assm:ortho_sep} holds with high probability in our data model (\cref{prop:assm_just}).

\subsection{Key definitions}

\citet{phuong2021the} and \citet{min2024early} show that there exists some \emph{trapping time}
\begin{align}
t_1=O(\log n /\sqrt{\lambda})
\end{align}
independent of the scale factor $\alpha$, after which every neuron becomes permanently specialized to a single class or becomes dead. Formally, consider the \emph{data-dependent cones}
\begin{align}
&\mathcal{C}_+ := \left\{\bfw: \mathbb{1}[\langle  \bfw, \bfx_i \rangle>0] = \mathbb{1}[y_i>0], \forall i\right\}, \\
&\mathcal{C}_\oslash := \left\{\bfw:\langle\bfw, \bfx_i \rangle \leq 0, \forall i\right\},
\end{align}
corresponding to neurons that activate only on the positive class, or never activated, respectively. We can also define $\mathcal{C}_-$ analogously. At time $t_1$, we can partition the indices of the neurons of the given model by which cone they belong to. More formally, we define \emph{neuron index partition} as
\begin{align}
    V_+ &:= \{j\in[h]: \bfw_j(t_1)\in \mathcal{C}_+\}, \\
    V_\oslash &:= \{j\in[h]: \bfw_j(t_1)\in \mathcal{C}_\oslash\},
\end{align}
where $V_-$ can be defined analogously.

\citet{min2024early} shows that this partition remains the same for all $t\ge t_1$. Thus, for any such $t$, one can analyze class-wise learning dynamics by decoupling neurons into linear subnetworks indexed by the positive class $V_{+}$ and the negative class $V_{-}$ ($V_{\oslash}$ does not affect training).

In what follows, we focus only on the \textit{positive-class data}, for $t \ge t_1$. The negative class can be handled similarly,\footnote{See, for example, \citet[Section~3]{min2024early}} and we are not interested in $t < t_1$ as we are interested in the generalization of models that achieve low training risk.


Now, we can define the effective predictor as follows.
\begin{definition}[Effective predictor] For $t \ge t_1$, the effective (linear) predictor for the positive-class is defined as
\begin{align}
    \hat{\bfw}_{\alpha}(t) :=\sum_{j\in {V}_+}v_j(t)\bfw_j(t)
\end{align}
\end{definition}
Here, the activation function is linear as each neuron is activated only for the data from its corresponding class.

Given this effective predictor, we are interested in the angular alignment (i.e., normalized inner product) between the direction of the effective predictor and the class mean. More concretely, consider the following definition.
\begin{definition}[Angular alignment] For $t \ge t_1$, the angular alignment between the effective predictor and some reference direction $\mathbf{r} \in \mathbb{R}^d$ is defined as
\begin{align}
    \Psi(t):=\left\langle \frac{\hat{\bfw}_\alpha(t)}{\|\hat{\bfw}_\alpha(t)\|}, \mathbf{r}\right\rangle
\end{align}
\end{definition}
In particular, we are interested in analyzing the angular alignment where the reference direction is the class mean
\begin{align}
\mathbf{r} = \bfx_+/\|\bfx_+\|,\quad \text{where}\quad\bfx_+:= \sum_{i:y_i=+1}\mathbf{x}_i.
\end{align}
In \cref{sec:theory}, we derive a lower bound on this angular alignment, as a function of the scale factor $\alpha$. 


\section{Theoretical Analysis}\label{sec:theory}
Based on the formulation described in \cref{sec:theories}, we now provide our main theoretical results. We first provide lower bounds in the neuron alignment during two distinct phases of training (\cref{lem:phase1_lb,lem:phase2_bound}). Then, based on the results we provide an upper bound on the excess error (\cref{thm:rb}). All proofs in this section are deferred to the Appendix.

\subsection{Neuron alignment in phase 1}\label{ssec:phase1}


Under the setup specified in \cref{sec:theories}, \citet{min2024early} provides the weight space ODE that governing the Phase 1.
\begin{lemma}[Lemma 3 and 4 of \citet{min2024early}, informal]\label{lem:informal_min} Suppose that the scale factor satisfies
\begin{align}
\alpha\leq {1}/{4\sqrt{h}\bfx_{\max}\mathsf{W}^2_{\max}}.
\end{align}
Then, for any $t\leq t_\alpha$ (where $t_\alpha\geq t_1$), the alignment ODE (\cref{eq:alignment_gf}) holds. The stationary point of \cref{eq:alignment_gf} is given by $\bfx_+/\|\bfx_+\|$.\\ (A more formal statement can be found in \cref{lem:lem3_min}.)
\end{lemma}

Here, we are particularly interested in the behavior of the alignment at a scale-dependent critical timestamp $t_\alpha=\Theta(\log(1/\alpha)/n)$, which marks the time threshold up to which \cref{lem:informal_min} holds. To analyze this, we first derive the angular alignment between the normalized first-layer weight vector $\bfw_j(t)/\|\bfw_j(t)\|$, where $j\in V_+$,  and the normalized class mean $\bfx_+/\|\bfx_+\|$:
\begin{align}
\nang(t_\alpha) := \left\langle  \frac{\bfw_j(t_\alpha)}{\|\bfw_j(t_\alpha)\|}, \frac{\bfx_+}{\|\bfx_+\|}\right\rangle.
\end{align}
In turn, following results provide a lower bound on $\nang(t_\alpha)$ and reveals how it depends on $\alpha$.
\begin{lemma}\label{lem:phase1_lb}
For any $j\in V_+$, we have
\begin{align}
    \nang(t_\alpha) \ge \sqrt{\zeta(\alpha)}
\tanh\left((t_{\alpha}-t_1)\|\bfx_+\|\sqrt{\zeta(\alpha)}\right), \label{eq:p1_nang}
\end{align}
and consequently,
\begin{align}
     \Psi(t_\alpha)\geq\sqrt{\zeta(\alpha)}
\tanh\left((t_{\alpha}-t_1)\|\bfx_+\|\sqrt{\zeta(\alpha)}\right), \label{eq:predictor_lb}
\end{align}
where $\zeta(\alpha):= 1-{4\alpha n \sqrt{h} \bfx_{\max}^2 \mathsf{W}_{\max}^2}/\|\bfx_+\|$.
\end{lemma}

\begin{corollary}\label{cor:phase1_angle} Suppose that $\|\bfx_+\|/n < 4\bfx_{\max}$ holds. Then, the angle between $\bfx_+$ and $\bfw_j(t_\alpha)$ is proportional to $\sqrt{\alpha}$.
\end{corollary}




As shown in \cref{lem:phase1_lb} and \cref{cor:phase1_angle}, in the vanishing scale limit (i.e., $\alpha \to 0$), the alignment becomes stronger, for both weights and the effective predictor, with the angular deviation approaching zero. Conversely, as the scale increases, the direction of $\bfw_j$ deviates more significantly from the class mean. However, this phase does not capture the behavior at reasonable convergence, since the loss has not yet decreased significantly at this phase \citep{min2024early}. Building on these results, we proceed to Phase 2.




\subsection{Evolution of the alignment in phase 2}\label{ssec:phase2}

Now, we analyze how the results from the phase 1 affect the subsequent training. In particular, we consider the non-asymptotic case where we continue training until the (positive-class) training risk reaches some designated threshold $\eta > 0$. Precisely, we define the \emph{stopping time}\footnote{Although we use the training risk as the stopping criterion, our empirical results still hold when the validation risk is used as the stopping criterion; see \cref{supp:valid_risk}.} as
\begin{align}
    t_{\eta,\alpha}:=\inf\{t \geq t_\alpha: \hat{L}_+(\theta_{t}) \leq \eta \},
\end{align}
where $\hat{L}_+(\cdot)$ denotes the training risk computed only on positive-class samples. Due to the decoupling of the neurons, the training risk can be decomposed as $\hat{L}(\theta_{t})=\hat{L}_+(\theta_{t}) + \hat{L}_-(\theta_{t})$, where each class-wise risk affects only the corresponding subnetwork. Thus, for simplicity, we will write $\hat{L}(\theta)=\hat{L}_+(\theta)$ in what follows. In the same spirit, we will replace $n$ with $n_+=\sum_{i}\mathbb{1}[y_i=+1]$, since this modification does not affect the results.

The reason why we consider such $t_{\eta,\alpha}$ is twofold: (1) This choice closely aligns with the common practice, e.g., early stopping; (2) It allows us to go beyond the well-known optimization-generalization trade-off to investigate whether stronger feature learning hurts generalization \citep{woodworth2020kernel}. Specifically, we examine the behavior at comparable training risk $\eta$ (i.e., at different GF timesteps) for various initialization scales, challenging the view that small initialization is universally beneficial for generalization.

Note that the Phase 2 dynamics are driven by (or, more precisely, initiated by) the result of \cref{lem:phase1_lb}; consequently, the behavior at $t_\eta$ depends on \cref{lem:phase1_lb}. Utilizing such results, we derive the lower bound on $\psi_j(t_{\eta,\alpha})$.



\begin{lemma}\label{lem:phase2_bound} Let $\beta:=\lambda^2\bfx_{\min}^2/32\bfx_{\max}$, where $\bfx_{\min}$ denotes the minimum value of all $\|\bfx\|$. Also let $t_2\geq t_\alpha = O(\log(1/\alpha)/n)$. Then, for any threshold $\eta >0$, we have
    \begin{align}
        \psi_j(t_{\eta,\alpha}) \geq \lambda +  m(\alpha)\exp(-g(\alpha)),\label{eq:phase2_lb_eq}
    \end{align}
    where $m(\alpha):=\psi_j(t_\alpha)-\lambda$ and
    \begin{align}
g(\alpha)\leq\bfx_{\max}n\left((t_2-t_\alpha)\hat L(t_\alpha)+\frac{1}{\beta}\log\frac{\hat{L}(t_2)}{\eta}\right). \nonumber
    \end{align}
\end{lemma}

Here, same as in \cref{eq:predictor_lb} of \cref{lem:phase1_lb}, we can derive the same lower bound on $\Psi(t_{\eta,\alpha})$ using properties of the conic hull, a derivation we omit for brevity.

From \cref{lem:phase1_lb}, the results imply that $m(\alpha)$ in \cref{eq:phase2_lb_eq} increases as $\alpha$ decreases. In contrast, $g(\alpha)$ depends on a non-asymptotic timescale, which makes it difficult to interpret directly. Nevertheless, we can analyze it indirectly: Since $t_2$ is defined as the timescale at which the loss decreases \emph{significantly} \citep{min2024early}, we expect $t_2 - t_\alpha \approx 0$. Consequently, we may (approximately) bound $g(\alpha)\lesssim \bfx_{\max}n(\log({\hat L(t_2)}/{\eta}) / \beta)\approx O(1)$. Plugging this estimate into \cref{eq:phase2_lb_eq} suggests that $\exp(-g(\alpha)) \approx 1$, which yields
\begin{align}
    \Psi(t_{\eta,\alpha}) \approx   \Psi(t_\alpha).\label{eq:p1andp2}
\end{align}
We validate this empirically and present results in \cref{app:p2_discuss}, with further discussions. There, we observe that the results aligns with our analysis. \cref{eq:p1andp2} suggests that the alignment is \emph{almost consistent} in the phase 2, as they mainly follow the results from the phase 1.



\textbf{Comparison with prior work.} Prior works on implicit bias in (deep) linear classification show that the predictor converges asymptotically to the $\ell_2$ max-margin direction of the training set, which is considered desirable for linearly separable datasets (e.g., the hard-margin SVM solution) \citep{gunasekar2018implicit,ji2018gradient,yun2021a,phuong2021the,min2025neural}. In contrast, we analyze how much the effective predictor can deviate from the reference direction after a finite number of GF iterations; this perspective is particularly useful for analyzing generalization in Gaussian mixtures.

\subsection{Over-alignment vs. Over-fitting}\label{ssec:bound}

So far, we have analyzed the alignment of the neurons (and the effective predictor) to the class mean direction, dependent on the scale factor $\alpha$. In this subsection, we connect these results to provide an upper bound on the population error. Our main result (\cref{thm:rb}) reveals the pitfalls of overly small initialization, which we term \emph{over-alignment}.

We begin by defining the population error as follows.
\begin{definition}[Population error]
The (zero-one) population error rate of a predictor $\hat{\bfw}_{\alpha} \in \mathbb{R}^d$ is
\begin{align}
    \mathcal{E}(\hat{\bfw}_{\alpha}) := \Pr\left(\mathrm{sgn}\left(\hat{\bfw}_{\alpha}^\top \bfx\right) \neq y\right),
\end{align}
where the $\Pr(\cdot)$ denotes the probability with respect to the data distribution described in \cref{eq:data_model}.
\end{definition}
Let us denote the Bayes optimal error achievable for the same dataset as $\mathcal{E}^*$. Then, we can decompose the excess error of the given predictor $\hat{\bfw}_{\alpha}$ into two terms.
\begin{align}
     &\mathcal{E}(\hat{\bfw}_\alpha) - \mathcal{E}^* \nonumber\\ 
     &=\underbrace{\inf_{\bfv \in H(\alpha)} \mathcal{E}(\bfv) - \mathcal{E}^*}_{=:\mathsf{OA}(\alpha)}+ \underbrace{\mathcal{E}(\hat{\bfw}_\alpha) -  \inf_{\bfv \in H(\alpha)} \mathcal{E}(\bfv)}_{=:\mathsf{OF}(\alpha)}.\label{eq:of_oa}
\end{align}
Here, the set $H(\alpha)$ denotes the circular cone around the one-side class mean, characterizing the region where the effective predictor resides after the time $t_{\eta,\alpha}$ has elapsed:
\begin{align}
H(\alpha)=\{\bfv\in \mathbb{S}^{d-1}: \langle \bfx_+/\|\bfx_+\|, \bfv \rangle \geq \Psi(t_{\eta,\alpha})\}.    
\end{align}
Note that we have constrained the $\ell_2$ norm of $\bfv$ to be one. This is because the zero-one error is invariant to scalar multiplication of the weight, which simplifies the analysis. 

In \cref{eq:of_oa}, we have introduced two terms, $\mathsf{OA}(\alpha)$ and $\mathsf{OF}(\alpha)$, which we refer to as the degrees of \emph{over-alignment} and \emph{over-fitting}, respectively. Intuitively, each term can be interpreted as follows:
\begin{itemize}[leftmargin=*,topsep=0pt,parsep=0pt]
\item $\mathsf{OA}(\alpha)$, which we dubbed \textit{over-alignment}, is the gap between the minimum achievable error among all predictors in $H(\alpha)$ and the Bayes error. As we will show below, this quantity \textbf{decreases as $\alpha$ increases}.
\item $\mathsf{OF}(\alpha)$, referred to as \textit{over-fitting}, is the gap between the population error of the learned predictor $\hat{\bfw}_\alpha$ and the minimum achievable error within $H(\alpha)$. We will show below that this quantity \textbf{increases as $\alpha$ increases}. 
\end{itemize}

\begin{figure}[t]
    \centering
    \begin{subfigure}{0.495\columnwidth}
        \centering
        \includegraphics[width=\linewidth]{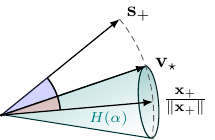}
        \caption{Case 1}
        \label{fig:image1}
    \end{subfigure}
    \hfill
    \begin{subfigure}{0.495\columnwidth}
        \centering
        \includegraphics[width=\linewidth]{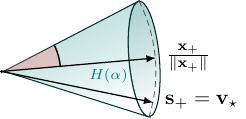}
        \caption{Case 2}
        \label{fig:image2}
    \end{subfigure}

    \caption{\textbf{Visual explanation.} In our analysis, $\mathsf{OA}(\alpha)$ is determined by the {\color{sectorblue} angle}, whereas $\mathsf{OF}(\alpha)$ is determined by the {\color{sectorred} angle}. (a) Case 1: When $\alpha$ is sufficiently small, there exists an irreducible gap between $\bfs_+$ and $\mathbf{v}_\star \in {\color{teal}H(\alpha)}$. In this case, the error is determined by both $\mathsf{OA}(\alpha)$ and $\mathsf{OF}(\alpha)$. (b) Case 2: When $\alpha$ is sufficiently large, we have $\bfs_+ = \mathbf{v}_\star$. In this case, the error is governed by the ``volume'' of ${\color{teal}H(\alpha)}$, i.e., by $\mathsf{OF}(\alpha)$ alone.}
    \label{fig:mech}
\end{figure}

The theorem below provides an upper bound on the excess error of the learned predictor, where the bound characterizes the trade-off between $\mathsf{OA}(\alpha)$ and $\mathsf{OF}(\alpha)$. Below, for simplicity, we write $\phi:=\angle(\bfx_+, \bfs_+)$.

\begin{theorem}\label{thm:rb}
Suppose that \cref{assm:ortho_sep} holds, $\|\hat{\bfw}_{\alpha}(t_{\eta,\alpha})\|\leq 1$, and $\langle \hat{\bfw}_{\alpha}(t_{\eta,\alpha}), \bfx_i\rangle \geq 0$ for all $i$ with $y_i = +1$. Let $G_\epsilon$ be the grid defined by
$G_\epsilon := \left\{ -1+k\epsilon : k\in \mathbb{Z}\right\} \cap \left[-1,1\right]$. Then, for any $\delta\in(0,1)$ and $\epsilon\in (0,0.25)$, we have
\begin{align}
    &\mathcal{E}(\hat{\bfw}_\alpha(t_{\eta,\alpha}) ) - \mathcal{E}^* 
 \leq \underbrace{\Phi\left(-\frac{\bfv_\star^{\top} \mathbf{s}_+}{\sigma}\right) -\Phi\left(-\frac{1}{\sigma}\right)}_{=\mathsf{OA(\alpha)}}\label{eq:risk_bound} \\&\!+\! \underbrace{\frac{2(1+e)}{\sigma\sqrt{2\pi}}\!\left(\!g(\alpha) h(n,d)\! +\! \eta \!+\!  C\!\left(1+\sigma\sqrt{d}\right) \!\sqrt{\frac{\log(6/\delta\epsilon)}{n}} \right)}_{\geq\mathsf{OF}(\alpha)} \nonumber 
\end{align}
\vspace{-0.3em}
with probability at least $1-\delta$, where:
\vspace{-0.3em}
\begin{align}
&\bfv_\star := \arg\max_{\bfv\in H(\alpha)}\bfv^\top \bfs_+, \\
& g(\alpha) : =2\sqrt{2\pi}\cdot\sqrt{{1-r(\alpha)^2}},\\
&h(n,d) := \sqrt{\frac{d}{n}}\left(\sigma\sqrt{\frac{d}{n}}+\sigma+1\right),\\
&r(\alpha):= \max\left\{r \in G_\epsilon : 0< r\leq \min_{\bfv \in H(\alpha)} \bfv^\top \mathbf{s}_+\right\}
\end{align}
and for some constants $C>0$.
\end{theorem}

To understand what \cref{thm:rb}, let us take a closer look at the dependencies of $\mathsf{OA}(\alpha)$ and $\mathsf{OF}(\alpha)$ on $\alpha$.

\textbf{Over-alignment.} For $\mathsf{OA}(\alpha)$, there exists two regimes. First, suppose that the scale factor $\alpha$ is sufficiently large (while still satisfying the condition from \cref{lem:informal_min}), so that $\Psi(t_{\eta,\alpha}) \leq \cos\phi$ holds. Then, we know that $\mathbf{v}_\star = \mathbf{s}_+$ holds and thus the over-alignment term becomes equal to zero (e.g., Case 2 in \cref{fig:mech}).


In the second regime, we consider the case where $\alpha$ is smaller than this threshold. Then, for moderate number of data $n$ and data dimension $d$, we know that $\mathsf{OA}(\alpha)$ is a non-increasing function of $\alpha$. Furthermore, let $\bar{\Psi}(\cdot)=\arccos\Psi(\cdot)$, then we have
\begin{align}
    \bfv_\star = \frac{\sin(\phi-\bar{\Psi}(t_{\eta,\alpha})) }{\|\bfx_+\|\sin\phi} {\bfx_+}+ \frac{\sin (\bar{\Psi}(t_{\eta,\alpha}))}{\sin \phi} \bfs_+,
\end{align}
which corresponds to the spherical linear interpolation (e.g., Case 1 in \cref{fig:mech}).


\begin{figure}[!t]
    \centering
    \includegraphics[width=0.90\linewidth]{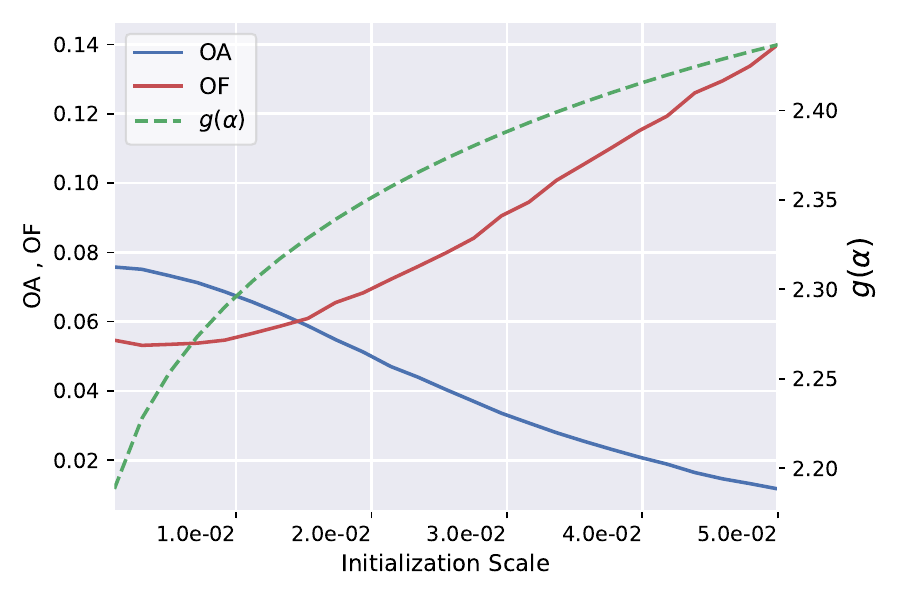}
    \vspace{-0.8em}
    \caption{\textbf{Numerical simulation of $\mathsf{OA}(\alpha)$, $\mathsf{OF}(\alpha)$, and $g(\alpha)$}: Note that $\mathsf{OA}(\alpha)+\mathsf{OF}(\alpha)$ recovers the excess error.}
        \label{fig:oa_of}
    \vspace{-1.5em}
\end{figure}


Together, these results demonstrate the phenomenon which we call \textit{over-alignment}: When $\alpha$ is small, there is a generalization gap due to the discrepancy between the best achievable predictor in the cone $H(\alpha)$ (i.e., $\mathbf{v}_\star$) and the Bayes optimal predictor (i.e., $\mathbf{s}_+$). This gap becomes more severe as $\alpha$ decreases, leading to an increase in generalization error in the large FLS regime.

\textbf{Over-fitting.} We now turn our attention to $\mathsf{OF}(\alpha)$. Since the exact equality for $\mathsf{OF}(\alpha)$ cannot be derived (due to $\mathcal{E}(\hat{\bfw}_\alpha)$ term), we upper bound the term as in \cref{eq:risk_bound}: Here, we study the behavior of $g(\alpha)$, since the only $\alpha$-dependent term among three terms in $\mathsf{OF}(\alpha)$. 

Here, the term $r(\alpha) \approx \min_{\bfv \in H(\alpha)} \bfv^\top \bfs_+$ captures the geometry of the cone $H(\alpha)$ (more precisely, the angle of the cone). For instance, when $n$ and $d$ are fixed and $\alpha$ is sufficiently small (e.g., Case 1 in \cref{fig:mech}), $r(\alpha)$ increases, hence $g(\alpha)$ decreases. Otherwise, for large $\alpha$, we have an increased $g(\alpha)$ (e.g., Case 2 in \cref{fig:mech}). 


We refer to this term---represented by $g(\alpha)$---as \textit{over-fitting}, since this phenomenon is consistent with the traditional notion of over-fitting in learning theory: As $\alpha$ increases, the volume of the hypothesis space of the learned predictor (i.e., $H(\alpha)$) also increases, leading to poor generalization.

\textbf{Numerical experiments.} To validate our theoretical results, in \cref{fig:oa_of}, we plot $\mathsf{OA}(\alpha)$, $\mathsf{OF}(\alpha)$, and $g(\alpha)$. We observe that $\mathsf{OA}(\alpha)$ and $\mathsf{OF}(\alpha)$ exhibit a trade-off, yielding an optimal FLS for the excess error (see \cref{app:u_shape}). Also, the trend of estimated $g(\alpha)$ correlates well with $\mathsf{OF}(\alpha)$. These results suggest that, even in large FLS regime, we can clearly formalize the optimal FLS. For details, see \cref{app:simul_detail}. 

\textbf{Small-norm regime.} We also note that \cref{thm:rb} holds in the small-norm regime, i.e., $\|\hat{\bfw}\left(t_{\eta,\alpha}\right)\|\le 1$, in contrast to most prior work on implicit bias in classification, which studies the norm-exploding regime where $\|\hat{\bfw}\left(t_{\eta,\alpha}\right)\|\to\infty$; see, e.g., \citet{soudry2018implicit}. Since we are mainly interested in (1) small $\alpha$ and (2) finite-time training, this condition can be deemed realistic (e.g., see  \cref{fig:norm}).

\textbf{Asymptotic analysis.} One might ask how these phenomena appears as $d$ and $n$ change. Unlike our work---which focuses on the finite-sample/dimensional regime---several prior works on Gaussian mixture classification have investigated the generalization error particularly in the \textit{proportional asymptotic regime}, where $d,n \to \infty$ with $d/n \in \left(0,\infty\right)$ \citep{mignacco2020role,refinetti2021classifying}. Applying the similar idea to our setting, under the modified proportional limit (induced by \cref{assm:ortho_sep}), we can obtain the following results.
\begin{proposition}\label{prop:propor}
Let
$
\gamma_1
:=
d/\left(n^2\log n\right)
$
and
$
\gamma_2
:=
\kappa^2/\left(\sigma^2\sqrt{d\log n}\right)
$.
Suppose $d,n\to\infty$ and
$
\gamma_2\to\gamma_{2,\infty}\in(0,\infty)
$.
Then, with probability tending to one, we have
\begin{align}
    \tan^2\phi
    \to
    {\sqrt{\gamma_1}}/{\gamma_2}.
\end{align}
Consequently, we can divide the regimes as follows:
\begin{enumerate}[leftmargin=*,topsep=0pt,parsep=0pt]
    \item Data-abundant: If $\gamma_1\to0$, then $\phi\to0$.
    \item Moderate: If $\gamma_1\to\gamma_{1,\infty}\in\left(0,\infty\right)$, then $\phi\to
        \arctan\left(
        {
        \gamma_{1,\infty}^{1/4}
        }/{
        \gamma_{2,\infty}^{1/2}
        }
        \right)$.
    \item High-dimensional: If $\gamma_1\to\infty$, then $\phi \to {\pi}/{2}$.
\end{enumerate}
\end{proposition}
\vspace{-0.5em}
From \cref{prop:propor}, we can notice that the trade-off between over-alignment and over-fitting arises only in the ``moderate'' regime, since $\phi$ converges to a nonzero angle between $0$ and $\pi/2$: For example, in the ``data-abundant'' regime, the empirical mean recovers the population signal direction---i.e., $\bfx_+/\|\bfx_+\| = \bfs_+$---and the optimal $\alpha$ converges to zero. Hence, the emergence of an optimal FLS can be viewed as a byproduct of the practical training regime, namely the finite-sample and finite-dimensional setting.

\subsection{Transferring the Optimal FLS}\label{ssec:transfer}
Can we leverage this phenomenon in practice? Indeed, recent works in deep learning theory have shown that, in some cases, optimal hyperparameters (HPs) can be transferred across different architectural configurations, making HP tuning more efficient at scale \citep{yang2021tuning,mlodozeniec2026completed}. Motivated by this perspective, in 5-layer vanilla CNNs, we show that the optimal FLS---as an instance of HPs---is transferable across widths and training dataset size; see \cref{app:hpt} for experimental details.

In \cref{fig:width,fig:dataset}, we plot the optimal output multiplier $c_\star$ for generalization by varying the width and training set size, respectively. Here, the results suggest that $c_\star$ are closely aligns with a numerical scaling law with respect to each factor. We note that these results align closely with our theory, which predicts that the $c_\star \appropto  O(n^{-2}h^{-1})$, obtained from differentiating the error bound. Extending the transfer argument beyond simple networks requires understanding the nontrivial effects arising from various factors, thus we leave this as future work.

\begin{figure}[!t]
    \centering
    \begin{subfigure}{0.495\columnwidth}
        \centering
        \includegraphics[width=\linewidth]{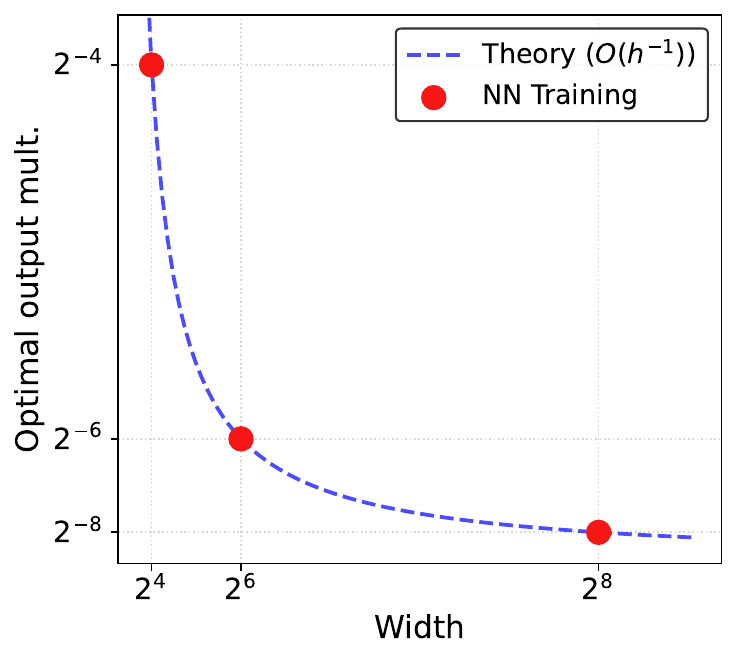}
        \caption{Varying width ($h$)}
        \label{fig:width}
    \end{subfigure}
    \hfill
    \begin{subfigure}{0.495\columnwidth}
        \centering
        \includegraphics[width=\linewidth]{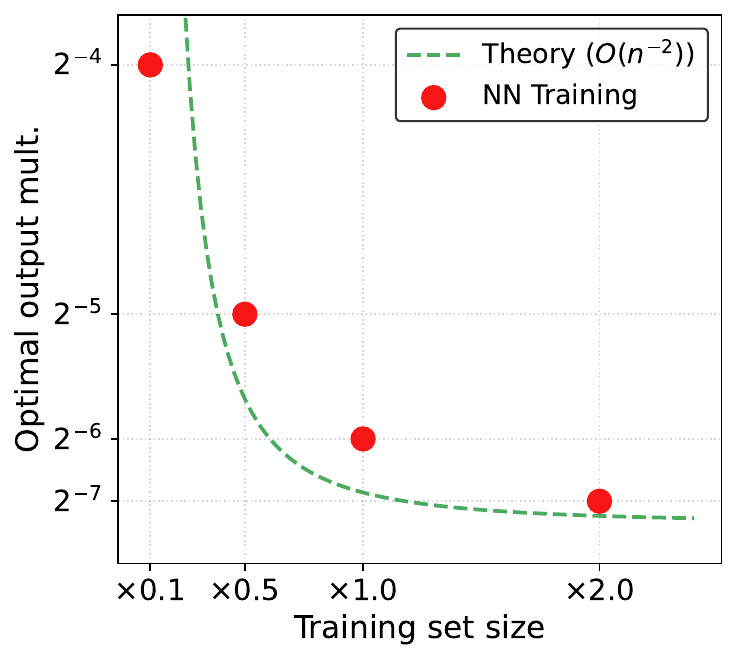}
        \caption{Varying dataset size ($n$)}
        \label{fig:dataset}
    \end{subfigure}
    \caption{\textbf{Optimal FLS (output multiplier) is predictable across width \& dataset size.} The dotted line indicates the scaling predicted by our theory, while the red dots represent the values obtained from training.}
        \label{fig:hpt}
    \vspace{-1.5em}
\end{figure}

\section{Conclusion}
\vspace{-0.3em}
In this work, we study how feature learning strength (FLS) affects generalization in classification tasks. Empirically, we find that not only small FLS, but also extremely large FLS can hurt generalization in deep networks, leading to a U-shaped curve in generalization performance (e.g., test accuracy or loss). To understand this phenomenon theoretically, we investigate its origin through the notions of \emph{over-alignment} and \emph{over-fitting}, which are obtained by decomposing the excess error based on the finite-time training dynamics of two-layer ReLU networks. We show that these two quantities exhibit a trade-off as a function of FLS, and that the optimal FLS arises from balancing them.

\textbf{Limitation and future direction.}
Our results do not capture the effects of many techniques used in practice, such as stochastic and adaptive gradient methods, data augmentation, and so on. Additionally, our theoretical framework relies on a strict constraint on the training dataset (i.e., orthogonal separability); Relaxing the assumption and extending the analysis to more practical regimes would be a important future work. From a practical perspective, one promising direction is to develop a rigorous  framework for analyzing the effect of FLS in larger models (e.g., Transformers).

\newpage
\section*{Acknowledgements}
This work was supported by Institute of Information \& Communications Technology Planning \& Evaluation (IITP) grant funded by the Korea government (MSIT) (No.RS-2024-00457882, No.RS-2019-II191906, No.RS-2022-II220713), the National Research Foundation of Korea (NRF) grant funded by the Korea government (MSIT) (No.RS-2024-00453301, No.RS-2025-24873016, No.RS-2026-25494004), and Basic Science Research Program through the National Research Foundation of Korea (NRF) funded by the Ministry of Education (No.RS-2025-25421671).

\section*{Impact Statement}
This paper presents work whose goal is to advance the field of Machine Learning. There are many potential societal consequences of our work, none which we feel must be specifically highlighted here.
\bibliography{bibtex}

@book{vershynin2018high,
author = {Vershynin, Roman},
title = {High-Dimensional Probability: An Introduction with Applications in Data Science},
edition = {2},
publisher = {Cambridge University Press},
address = {Cambridge},
series = {Cambridge Series in Statistical and Probabilistic Mathematics},
year = {2026}
}

@inproceedings{
min2025neural,
title={Neural Collapse under Gradient Flow on Shallow {R}e{LU} Networks for Orthogonally Separable Data},
author={Hancheng Min and Zhihui Zhu and Rene Vidal},
booktitle={NeurIPS},
year={2025}
}

@inproceedings{
pope2021the,
title={The Intrinsic Dimension of Images and Its Impact on Learning},
author={Phil Pope and Chen Zhu and Ahmed Abdelkader and Micah Goldblum and Tom Goldstein},
booktitle={ICLR},
year={2021}
}

@inproceedings{
yeom2025fast,
title={Fast Training of Sinusoidal Neural Fields via Scaling Initialization},
author={Taesun Yeom and Sangyoon Lee and Jaeho Lee},
booktitle={ICLR},
year={2025}
}

@inproceedings{
phuong2021the,
title={The inductive bias of {ReLU} networks on orthogonally separable data},
author={Mary Phuong and Christoph H Lampert},
booktitle={ICLR},
year={2021}
}

@inproceedings{
min2024early,
title={Early Neuron Alignment in Two-layer {ReLU} Networks with Small Initialization},
author={Hancheng Min and Enrique Mallada and Rene Vidal},
booktitle={ICLR},
year={2024},
}

@article{boursier2025early,
  title={Early alignment in two-layer networks training is a two-edged sword},
  author={Boursier, Etienne and Flammarion, Nicolas},
  journal={JMLR},
  year={2025}
}

@inproceedings{boursier2022gradient,
  title={Gradient flow dynamics of shallow relu networks for square loss and orthogonal inputs},
  author={Boursier, Etienne and Pillaud-Vivien, Loucas and Flammarion, Nicolas},
  booktitle={NeurIPS},
  year={2022}
}

@article{soudry2018implicit,
  title={The implicit bias of gradient descent on separable data},
  author={Soudry, Daniel and Hoffer, Elad and Nacson, Mor Shpigel and Gunasekar, Suriya and Srebro, Nathan},
  journal={JMLR},
  year={2018}
}

@inproceedings{
Lyu2020Gradient,
title={Gradient Descent Maximizes the Margin of Homogeneous Neural Networks},
author={Kaifeng Lyu and Jian Li},
booktitle={ICLR},
year={2020}
}

@article{vardi2023implicit,
  title={On the implicit bias in deep-learning algorithms},
  author={Vardi, Gal},
  journal={Communications of the ACM},
  year={2023},
}

@inproceedings{teney2024neural,
  title={Neural redshift: Random networks are not random functions},
  author={Teney, Damien and Nicolicioiu, Armand Mihai and Hartmann, Valentin and Abbasnejad, Ehsan},
  booktitle={CVPR},
  year={2024}
}

@inproceedings{
zhang2017understanding,
title={Understanding deep learning requires rethinking generalization},
author={Chiyuan Zhang and Samy Bengio and Moritz Hardt and Benjamin Recht and Oriol Vinyals},
booktitle={ICLR},
year={2017}
}

@inproceedings{
even2023sgd,
title={({S}){GD} over Diagonal Linear Networks: Implicit bias, Large Stepsizes and Edge of Stability},
author={Mathieu Even and Scott Pesme and Suriya Gunasekar and Nicolas Flammarion},
booktitle={NeurIPS},
year={2023},
}

@inproceedings{
wu2023implicit,
title={Implicit Bias of Gradient Descent for Logistic Regression at the Edge of Stability},
author={Jingfeng Wu and Vladimir Braverman and Jason D. Lee},
booktitle={NeurIPS},
year={2023},
}

@inproceedings{cao2023implicit,
  title={The implicit bias of batch normalization in linear models and two-layer linear convolutional neural networks},
  author={Cao, Yuan and Zou, Difan and Li, Yuanzhi and Gu, Quanquan},
  booktitle={COLT},
  year={2023},
}

@inproceedings{jacot2018neural,
  title={Neural tangent kernel: Convergence and generalization in neural networks},
  author={Jacot, Arthur and Gabriel, Franck and Hongler, Cl{\'e}ment},
  booktitle={NeurIPS},
  year={2018}
}

@inproceedings{chizat2019lazy,
  title={On lazy training in differentiable programming},
  author={Chizat, Lenaic and Oyallon, Edouard and Bach, Francis},
  booktitle={NeurIPS},
  year={2019}
}

@inproceedings{arora2019fine,
  title={Fine-grained analysis of optimization and generalization for overparameterized two-layer neural networks},
  author={Arora, Sanjeev and Du, Simon and Hu, Wei and Li, Zhiyuan and Wang, Ruosong},
  booktitle={ICML},
  year={2019},
}

@inproceedings{allen2019convergence,
  title={A convergence theory for deep learning via over-parameterization},
  author={Allen-Zhu, Zeyuan and Li, Yuanzhi and Song, Zhao},
  booktitle={ICML},
  year={2019},
}

@inproceedings{woodworth2020kernel,
  title={Kernel and rich regimes in overparametrized models},
  author={Woodworth, Blake and Gunasekar, Suriya and Lee, Jason D and Moroshko, Edward and Savarese, Pedro and Golan, Itay and Soudry, Daniel and Srebro, Nathan},
  booktitle={COLT},
  year={2020},
}

@inproceedings{
atanasov2022neural,
title={Neural Networks as Kernel Learners: The Silent Alignment Effect},
author={Alexander Atanasov and Blake Bordelon and Cengiz Pehlevan},
booktitle={ICLR},
year={2022},
}

@inproceedings{yang2021tpiv,
  title = 	 {Tensor Programs IV: Feature Learning in Infinite-Width Neural Networks},
  author =       {Yang, Greg and Hu, Edward J.},
  booktitle = 	 {ICML},
  year = 	 {2021},
}

@inproceedings{
yang2021tuning,
title={Tuning Large Neural Networks via Zero-Shot Hyperparameter Transfer},
author={Greg Yang and Edward J Hu and Igor Babuschkin and Szymon Sidor and Xiaodong Liu and David Farhi and Nick Ryder and Jakub Pachocki and Weizhu Chen and Jianfeng Gao},
booktitle={NeurIPS},
year={2021},
}

@inproceedings{
yun2021a,
title={A unifying view on implicit bias in training linear neural networks},
author={Chulhee Yun and Shankar Krishnan and Hossein Mobahi},
booktitle={ICLR},
year={2021},
}

@inproceedings{
atanasov2025the,
title={The Optimization Landscape of {SGD} Across the Feature Learning Strength},
author={Alexander Atanasov and Alexandru Meterez and James B Simon and Cengiz Pehlevan},
booktitle={ICLR},
year={2025},
}

@inproceedings{petrini2022learning,
  title={Learning sparse features can lead to overfitting in neural networks},
  author={Petrini, Leonardo and Cagnetta, Francesco and Vanden-Eijnden, Eric and Wyart, Matthieu},
  booktitle={NeurIPS},
  year={2022}
}

@inproceedings{nacson2019convergence,
  title={Convergence of gradient descent on separable data},
  author={Nacson, Mor Shpigel and Lee, Jason and Gunasekar, Suriya and Savarese, Pedro Henrique Pamplona and Srebro, Nathan and Soudry, Daniel},
  booktitle={AISTATS},
  year={2019},
}

@inproceedings{
ji2018gradient,
title={Gradient descent aligns the layers of deep linear networks},
author={Ziwei Ji and Matus Telgarsky},
booktitle={ICLR},
year={2019},
}

@inproceedings{
glasgow2024sgd,
title={{SGD} Finds then Tunes Features in Two-Layer Neural Networks with near-Optimal Sample Complexity: A Case Study in the {XOR} problem},
author={Margalit Glasgow},
booktitle={ICLR},
year={2024},
}

@inproceedings{du2018algorithmic,
  title={Algorithmic regularization in learning deep homogeneous models: Layers are automatically balanced},
  author={Du, Simon S and Hu, Wei and Lee, Jason D},
  booktitle={NeurIPS},
  year={2018}
}

@inproceedings{bordelon2022self,
  title={Self-consistent dynamical field theory of kernel evolution in wide neural networks},
  author={Bordelon, Blake and Pehlevan, Cengiz},
  booktitle={NeurIPS},
  year={2022}
}

@article{geiger2020disentangling,
  title={Disentangling feature and lazy training in deep neural networks},
  author={Geiger, Mario and Spigler, Stefano and Jacot, Arthur and Wyart, Matthieu},
  journal={Journal of Statistical Mechanics: Theory and Experiment},
  year={2020},
}

@inproceedings{kunin2024get,
  title={Get rich quick: exact solutions reveal how unbalanced initializations promote rapid feature learning},
  author={Kunin, Daniel and Ravent{\'o}s, Allan and Domin{\'e}, Cl{\'e}mentine and Chen, Feng and Klindt, David and Saxe, Andrew and Ganguli, Surya},
  booktitle={NeurIPS},
  year={2024}
}

@article{laurent2000adaptive,
  title={Adaptive estimation of a quadratic functional by model selection},
  author={Laurent, Beatrice and Massart, Pascal},
  journal={Annals of statistics},
  year={2000},
}

@article{
agarwala2023temperature,
title={Temperature check: theory and practice for training models with softmax-cross-entropy losses},
author={Atish Agarwala and Samuel Stern Schoenholz and Jeffrey Pennington and Yann Dauphin},
journal={TMLR},
year={2023},
}

@article{jacot2021saddle,
  title={Saddle-to-saddle dynamics in deep linear networks: Small initialization training, symmetry, and sparsity},
  author={Jacot, Arthur and Ged, Fran{\c{c}}ois and {\c{S}}im{\c{s}}ek, Berfin and Hongler, Cl{\'e}ment and Gabriel, Franck},
  journal={arXiv preprint arXiv:2106.15933},
  year={2021}
}

@inproceedings{
kunin2025alternating,
title={Alternating Gradient Flows: A Theory of Feature Learning in Two-layer Neural Networks},
author={Daniel Kunin and Giovanni Luca Marchetti and Feng Chen and Dhruva Karkada and James B Simon and Michael R DeWeese and Surya Ganguli and Nina Miolane},
booktitle={NeurIPS},
year={2025},
}

@article{maennel2018gradient,
  title={Gradient descent quantizes relu network features},
  author={Maennel, Hartmut and Bousquet, Olivier and Gelly, Sylvain},
  journal={arXiv preprint arXiv:1803.08367},
  year={2018}
}

@inproceedings{stoger2021small,
  title={Small random initialization is akin to spectral learning: Optimization and generalization guarantees for overparameterized low-rank matrix reconstruction},
  author={St{\"o}ger, Dominik and Soltanolkotabi, Mahdi},
  booktitle={NeurIPS},
  year={2021}
}

@inproceedings{li2021implicit,
  title={Implicit sparse regularization: The impact of depth and early stopping},
  author={Li, Jiangyuan and Nguyen, Thanh and Hegde, Chinmay and Wong, Ka Wai},
  booktitle={NeurIPS},
  year={2021}
}

@inproceedings{sclocchi2023dissecting,
  title={Dissecting the effects of SGD noise in distinct regimes of deep learning},
  author={Sclocchi, Antonio and Geiger, Mario and Wyart, Matthieu},
  booktitle={ICML},
  year={2023},
}

@article{clarke1975generalized,
  title={Generalized gradients and applications},
  author={Clarke, Frank H},
  journal={Transactions of the American Mathematical Society},
  year={1975}
}

@inproceedings{
arora2018a,
title={A Convergence Analysis of Gradient Descent for Deep Linear Neural Networks},
author={Sanjeev Arora and Nadav Cohen and Noah Golowich and Wei Hu},
booktitle={ICLR},
year={2019},
}

@inproceedings{refinetti2021classifying,
  title={Classifying high-dimensional gaussian mixtures: Where kernel methods fail and neural networks succeed},
  author={Refinetti, Maria and Goldt, Sebastian and Krzakala, Florent and Zdeborov{\'a}, Lenka},
  booktitle={ICML},
  year={2021},
}

@inproceedings{gunasekar2018implicit,
  title={Implicit bias of gradient descent on linear convolutional networks},
  author={Gunasekar, Suriya and Lee, Jason D and Soudry, Daniel and Srebro, Nati},
  booktitle={NeurIPS},
  year={2018}
}

@inproceedings{
li2023the,
title={The Lazy Neuron Phenomenon: On Emergence of Activation Sparsity in Transformers},
author={Zonglin Li and Chong You and Srinadh Bhojanapalli and Daliang Li and Ankit Singh Rawat and Sashank J. Reddi and Ke Ye and Felix Chern and Felix Yu and Ruiqi Guo and Sanjiv Kumar},
booktitle={ICLR},
year={2023},
}

@article{bartlett2002rademacher,
  title={Rademacher and gaussian complexities: Risk bounds and structural results},
  author={Bartlett, Peter L and Mendelson, Shahar},
  journal={JMLR},
  year={2002}
}

@inproceedings{maurer2021concentration,
  title={Concentration inequalities under sub-gaussian and sub-exponential conditions},
  author={Maurer, Andreas and Pontil, Massimiliano},
  booktitle={NeurIPS},
  year={2021}
}

@inproceedings{vgg,
  author={Karen Simonyan and Andrew Zisserman},
  title={Very Deep Convolutional Networks for Large-Scale Image Recognition},
  year={2015},
  booktitle={ICLR},
}

@inproceedings{he2016deep,
  title={Deep residual learning for image recognition},
  author={He, Kaiming and Zhang, Xiangyu and Ren, Shaoqing and Sun, Jian},
  booktitle={CVPR},
  year={2016}
}

@misc{mjt_dlt,
author = {Matus Telgarsky},
title = {Deep learning theory lecture notes},
howpublished = {\url{https://mjt.cs.illinois.edu/dlt/}},
year = {2021}
}

@article{masarczyk2025unpacking,
  title={Unpacking Softmax: How Temperature Drives Representation Collapse, Compression, and Generalization},
  author={Masarczyk, Wojciech and Ostaszewski, Mateusz and Cheng, Tin Sum and Lucchi, Aurelien and Pascanu, Razvan and others},
  journal={arXiv preprint arXiv:2506.01562},
  year={2025}
}

@inproceedings{
domine2025from,
title={From Lazy to Rich: Exact Learning Dynamics in Deep Linear Networks},
author={Cl{\'e}mentine Carla Juliette Domin{\'e} and Nicolas Anguita and Alexandra Maria Proca and Lukas Braun and Daniel Kunin and Pedro A. M. Mediano and Andrew M Saxe},
booktitle={ICLR},
year={2025},
}

@inproceedings{ansuini2019intrinsic,
  title={Intrinsic dimension of data representations in deep neural networks},
  author={Ansuini, Alessio and Laio, Alessandro and Macke, Jakob H and Zoccolan, Davide},
  booktitle={NeurIPS},
  year={2019}
}

@inproceedings{gong2019intrinsic,
  title={On the intrinsic dimensionality of image representations},
  author={Gong, Sixue and Boddeti, Vishnu Naresh and Jain, Anil K},
  booktitle={CVPR},
  year={2019}
}

@inproceedings{
brock2019large,
title={Large Scale {GAN} Training for High Fidelity Natural Image Synthesis},
author={Andrew Brock and Jeff Donahue and Karen Simonyan},
booktitle={ICLR},
year={2019},
}

@inproceedings{
mehta2021extreme,
title={Extreme Memorization via Scale of Initialization},
author={Harsh Mehta and Ashok Cutkosky and Behnam Neyshabur},
booktitle={ICLR},
year={2021},
}

@article{mei2018mean,
  title={A mean field view of the landscape of two-layer neural networks},
  author={Mei, Song and Montanari, Andrea and Nguyen, Phan-Minh},
  journal={PNAS},
  year={2018}
}

@book{ledoux1991probability,
  title={Probability in Banach Spaces: isoperimetry and processes},
  author={Ledoux, Michel and Talagrand, Michel},
  volume={23},
  year={1991},
  publisher={Springer}
}

@article{simon2026there,
  title={There will be a scientific theory of deep learning},
  author={Simon, Jamie and Kunin, Daniel and Atanasov, Alexander and Boix-Adser{\`a}, Enric and Bordelon, Blake and Cohen, Jeremy and Ghosh, Nikhil and Guth, Florentin and Jacot, Arthur and Kamb, Mason and others},
  journal={arXiv preprint arXiv:2604.21691},
  year={2026}
}

@inproceedings{mignacco2020role,
  title={The role of regularization in classification of high-dimensional noisy gaussian mixture},
  author={Mignacco, Francesca and Krzakala, Florent and Lu, Yue and Urbani, Pierfrancesco and Zdeborova, Lenka},
  booktitle={ICML},
  year={2020},
}

@inproceedings{
mlodozeniec2026completed,
title={Completed Hyperparameter Transfer across Modules, Width, Depth, Batch and Duration},
author={Bruno Kacper Mlodozeniec and Pierre Ablin and Louis B{\'e}thune and Dan Busbridge and Michal Klein and Jason Ramapuram and Marco Cuturi},
booktitle={ICLR},
year={2026}
}

@article{karkada2024lazy,
  title={The lazy (NTK) and rich ($\mu$P) regimes: a gentle tutorial},
  author={Karkada, Dhruva},
  journal={arXiv preprint arXiv:2404.19719},
  year={2024}
}
\bibliographystyle{icml2026}

\newpage
\appendix
\onecolumn
\startcontents[app]
\renewcommand\contentsname{Appendix}
\printcontents[app]{}{1}[3]{\section*{\contentsname}}

\newpage
\section{Experimental Details and Additional Results}\label{app:exp_det_res}

In this section, we provide further experimental details and omitted results. For all training runs, we use a single GPU of NVIDIA RTX 3090/4090 or A6000.

\subsection{Details About  Experiments in \cref{sec:dnn}}\label{app:emp_exp_details}
\textbf{Training details.} Across all training runs, we use a batch size of 128 and the vanilla SGD optimizer (without momentum). For training iterations, we trained until 80 epochs for all runs, which we found to be enough (i.e., all networks reach the best test accuracy and the lowest test loss within these epochs). To achieve near-perfect training accuracy (and near-zero training loss), we do not use data augmentation. Moreover, we do not use other training techniques, such as a learning rate scheduler or weight decay.

\textbf{Dataset details.} Here, we describe specific details for the image datasets used in the experiments.
\begin{itemize}[leftmargin=*,topsep=0pt,parsep=0pt]
    \item \textbf{CIFAR-10} and \textbf{CIFAR-100} each consist of 50k training images and 10k test images, with 10 and 100 classes, respectively.
    \item For the \textbf{BigGAN-generated dataset}, we generate 1k images per class (i.e., 10k images in total for 10 classes).\footnote{We additionally generate 10k images for the experiments in \cref{ssec:transfer}.} We then randomly select 8k images for training, ensuring class balance, and use the remaining 2k images as the test set. We choose 10 classes from the dog category of ImageNet: `\texttt{basenji},' `\texttt{basset},' `\texttt{beagle},' `\texttt{borzoi},' `\texttt{keeshond},' `\texttt{standard poodle},' `\texttt{vizsla},' `\texttt{weimaraner},' `\texttt{whippet},' and `\texttt{yorkshire terrier}.' All images are resized to $32\times 32$ resolution. We provide example images in \cref{fig:example_biggan}.
\end{itemize}
\begin{figure}[h]
    \centering
    \includegraphics[width=0.37\linewidth]{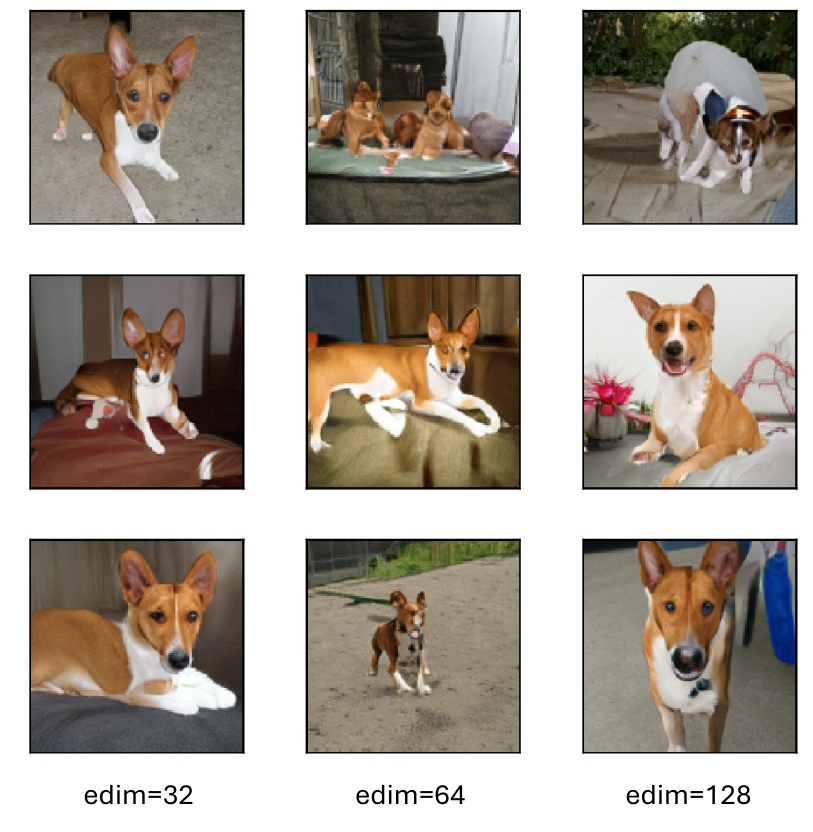}
    \caption{\textbf{Example images from the BigGAN-generated datasets.} As the effective dimensionality (i.e., \textsf{edim}) of the input increases, BigGAN produces more diverse images, thereby making the task more difficult.} 
    \label{fig:example_biggan}
\end{figure}

\newpage
\subsection{Details About  Experiments in \cref{ssec:transfer}}\label{app:hpt}
We use a 5-layer vanilla CNN as the neural network and a BigGAN-generated dataset with an effective dimension of 128 as the training dataset. As in \cref{sec:dnn}, we use a batch size of 128 and the vanilla SGD optimizer without momentum. In the width-scaling experiments, we vary the number of channels in the CNN. The optimal output multiplier is defined as the value of the multiplier that minimizes the test loss.

\vspace{60pt}
\subsection{Validation risk as stopping criterion}\label{supp:valid_risk}
In this subsection, we provide additional experimental results obtained when the stopping criterion is changed to \textit{validation risk}. To this end, we construct a fixed validation set using 20\% of the original training set.

The results are shown in \cref{tab:valid}. We observe that the optimal output multiplier does not change, even when the network is early-stopped once the optimal validation loss is achieved.

\begin{table}[h]
\centering
\caption{\textbf{Test accuracy (\%) over 80 epochs, with early stopping based on the minimum validation loss.} We train ResNet-18 on a BigGAN-generated dataset (edim=128).}
\label{tab:valid}
\begin{tabular}{lccccccc}
\toprule
Output mult. & $2^{-10}$ & $2^{-8}$ & $2^{-6}$ & $2^{-4}$ & $2^{-2}$ & $2^{0}$ & $2^{2}$ \\
\midrule
Peak acc. (80 epochs) & 72.92 & 75.35 & \textbf{76.62} & 75.18 & 69.22 & 59.95 & 49.78 \\
ES w/ val. loss        & 67.63 & 70.63 & \textbf{73.35} & 71.53 & 66.68 & 56.83 & 47.20 \\
\bottomrule
\end{tabular}
\end{table}

\newpage
\subsection{Experiments on CIFARs}\label{app:cifars}
In this subsection, we present additional results on CIFAR-10 and CIFAR-100 with varying effective dimensionality, as a follow-up to \cref{fig:main_cifar100} in the main paper. Specifically, we report the peak test accuracy and the best (i.e., lowest) test loss achieved during training for VGG19 and ResNet\{18,34,50\}.

\begin{figure*}[!htbp]
    \centering
    \begin{subfigure}[t]{0.25\textwidth}
        \centering
        \includegraphics[width=\linewidth]{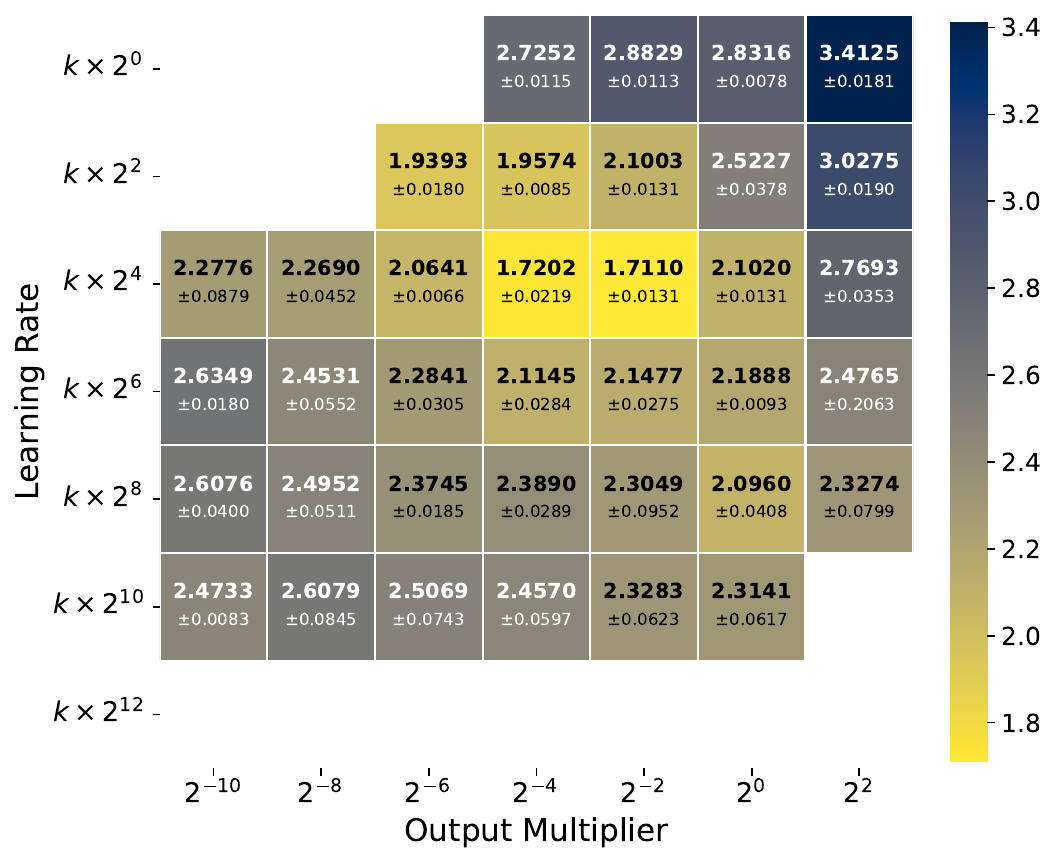}
        \caption{VGG19}
    \end{subfigure}\hfill
    \begin{subfigure}[t]{0.25\textwidth}
        \centering
        \includegraphics[width=\linewidth]{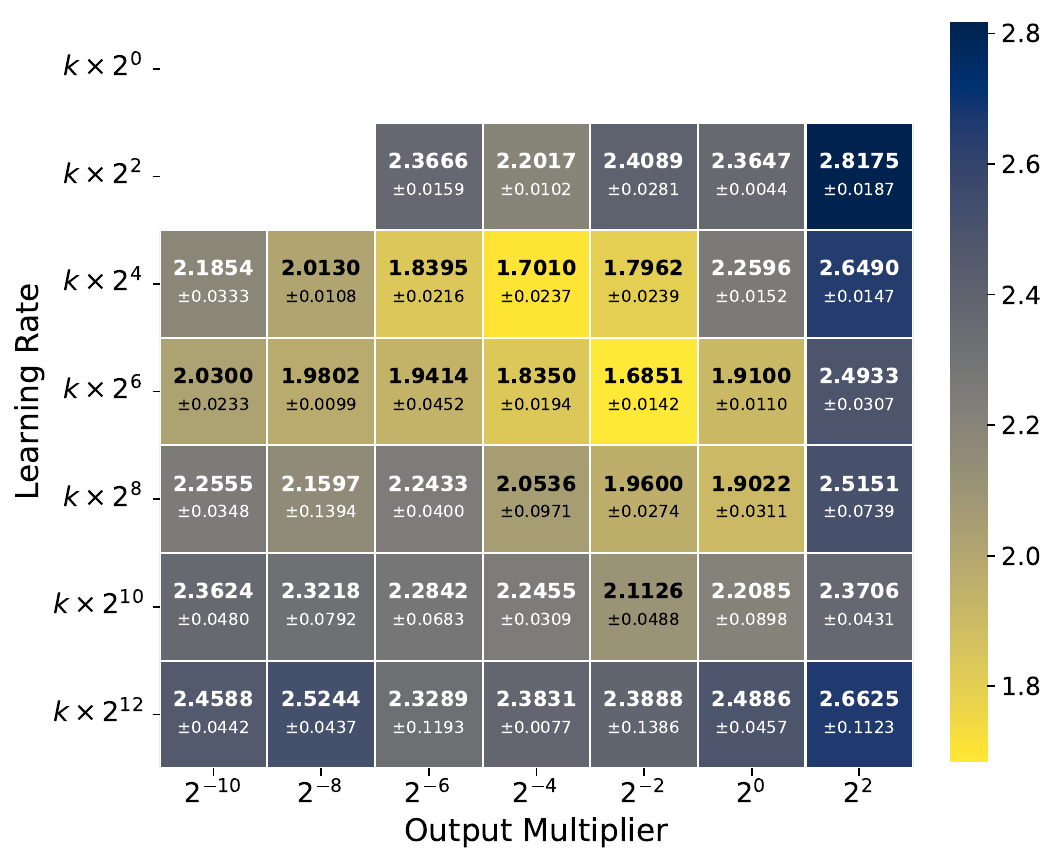}
        \caption{ResNet18}
    \end{subfigure}\hfill
    \begin{subfigure}[t]{0.25\textwidth}
        \centering
        \includegraphics[width=\linewidth]{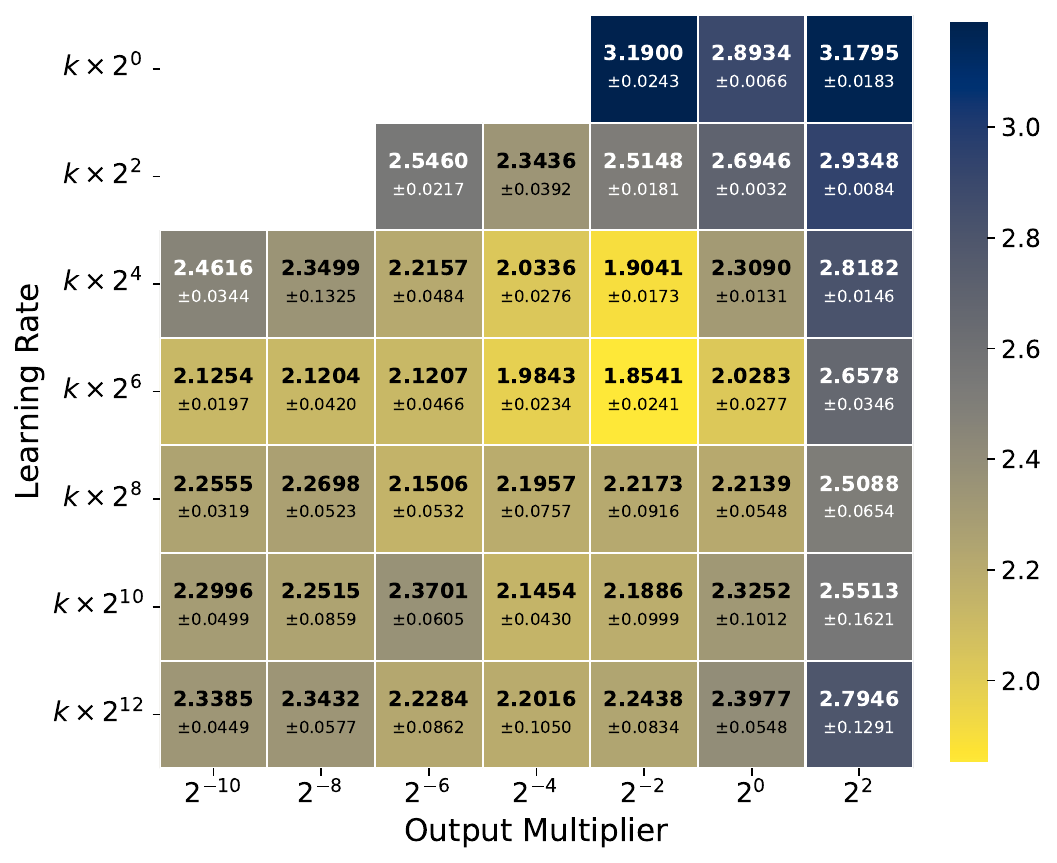}
        \caption{ResNet34}
    \end{subfigure}\hfill
    \begin{subfigure}[t]{0.25\textwidth}
        \centering
        \includegraphics[width=\linewidth]{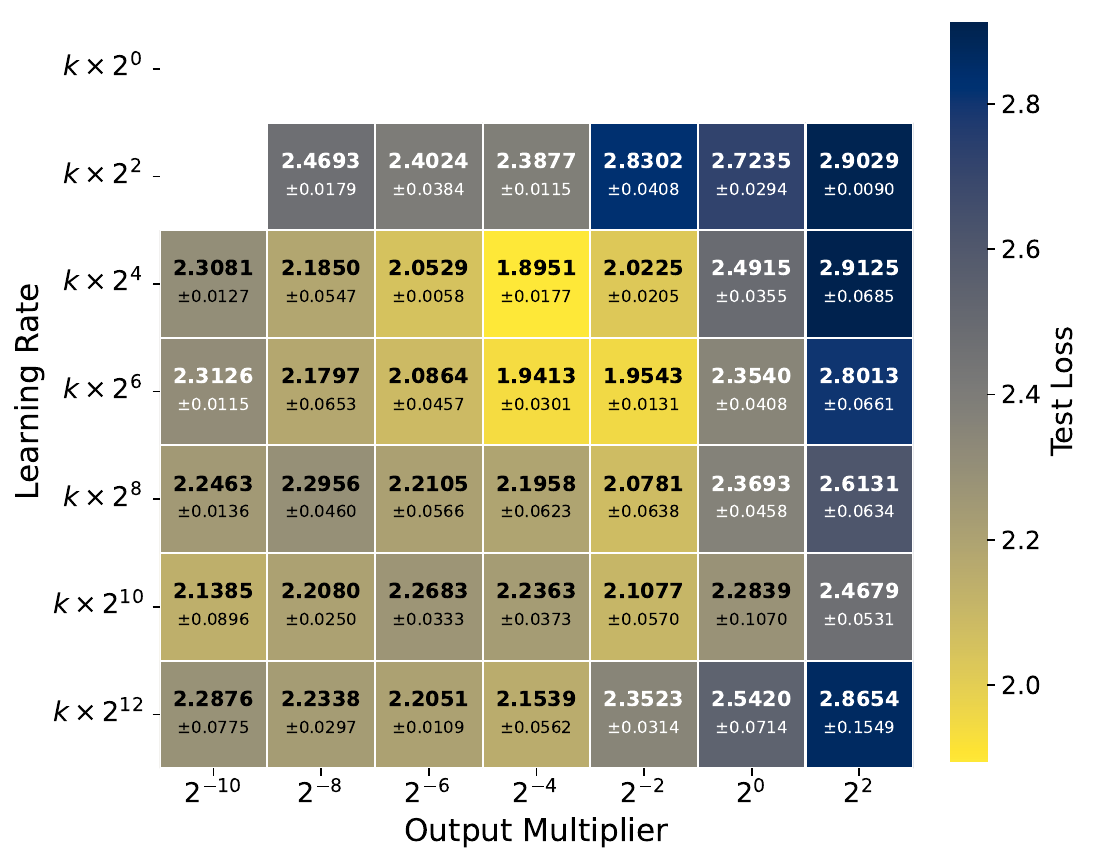}
        \caption{ResNet50}
    \end{subfigure}

    \caption{Best test loss (CIFAR-100).}
    \label{fig:app_cifar100_testloss}
\end{figure*}

\begin{figure*}[!htbp]
    \centering
    \begin{subfigure}[t]{0.25\textwidth}
        \centering
        \includegraphics[width=\linewidth]{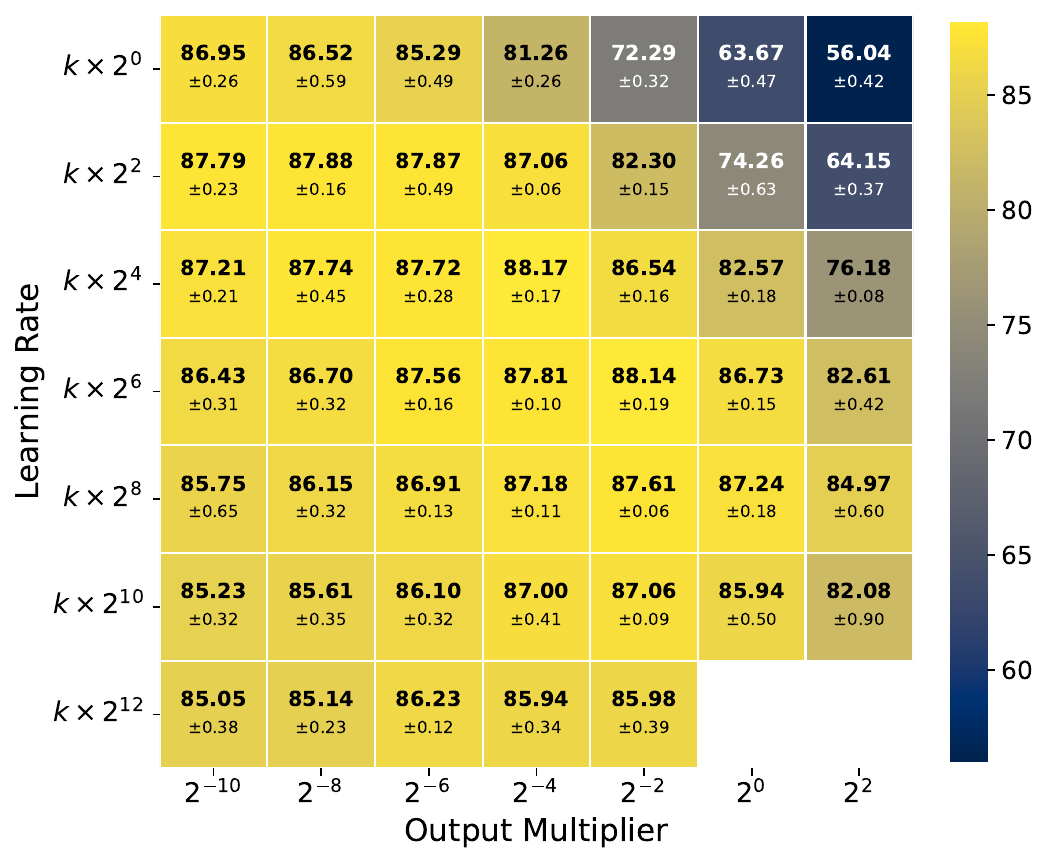}
        \caption{VGG19}
    \end{subfigure}\hfill
    \begin{subfigure}[t]{0.25\textwidth}
        \centering
        \includegraphics[width=\linewidth]{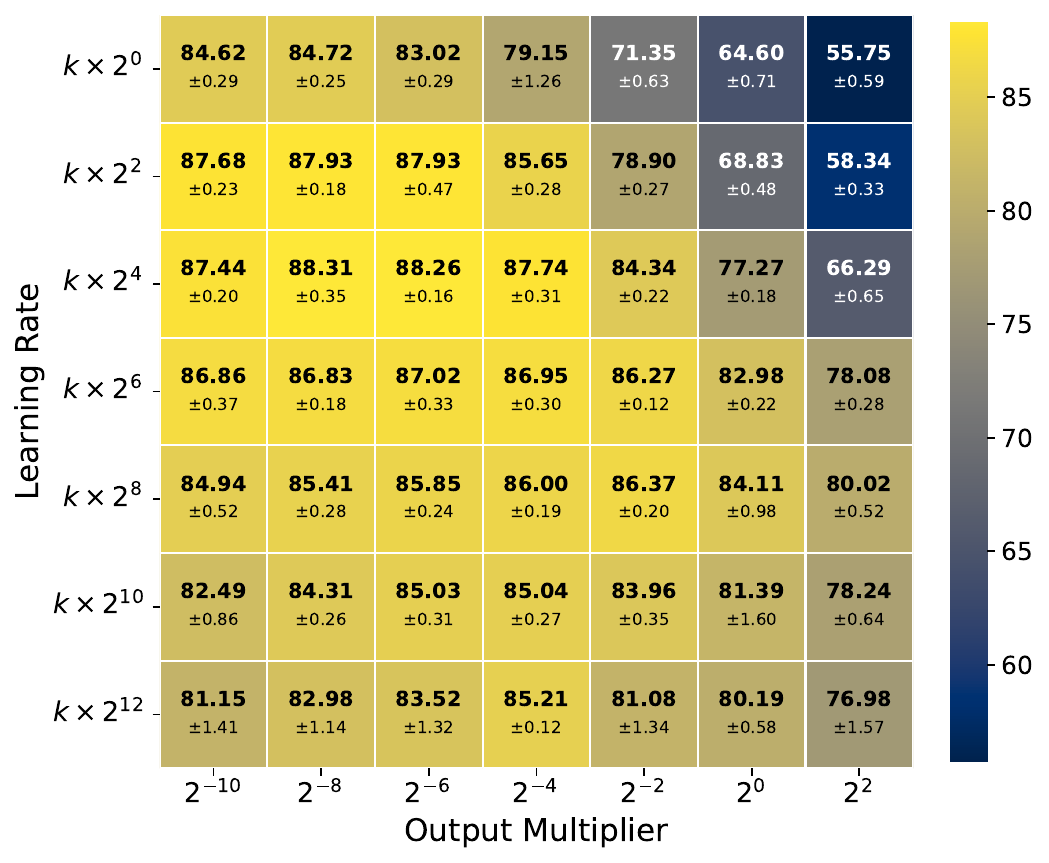}
        \caption{ResNet18}
    \end{subfigure}\hfill
    \begin{subfigure}[t]{0.25\textwidth}
        \centering
        \includegraphics[width=\linewidth]{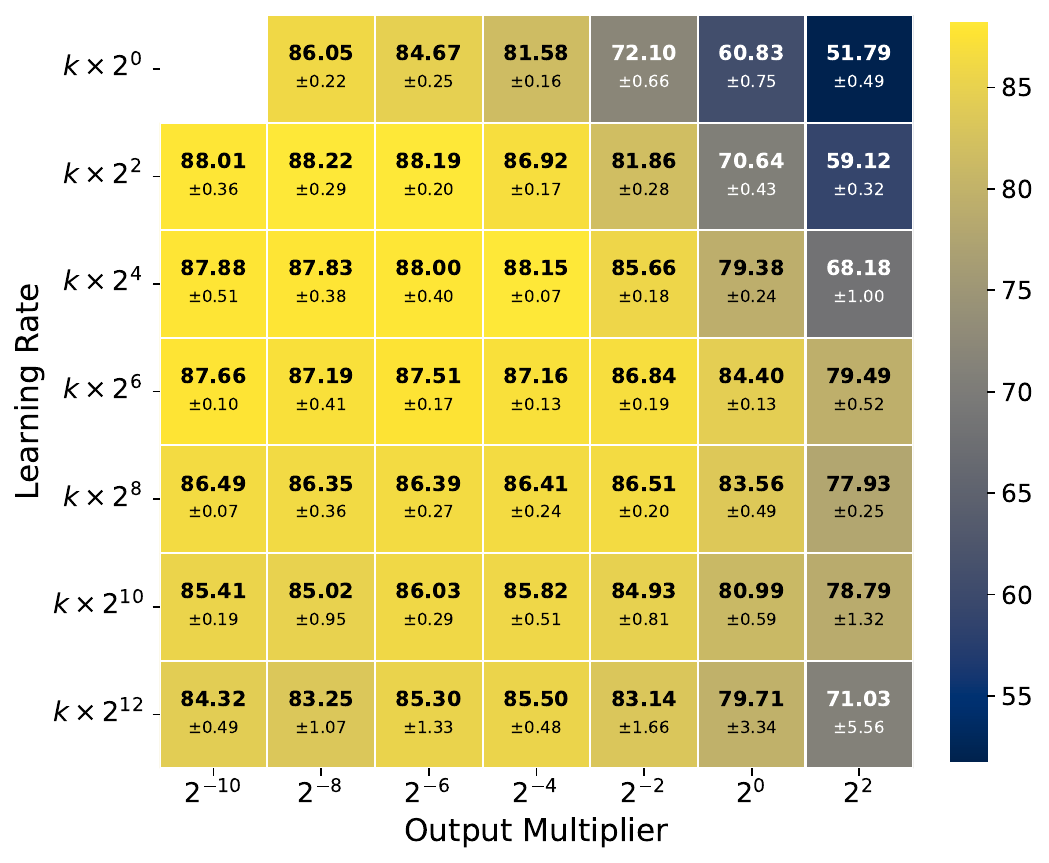}
        \caption{ResNet34}
    \end{subfigure}\hfill
    \begin{subfigure}[t]{0.25\textwidth}
        \centering
        \includegraphics[width=\linewidth]{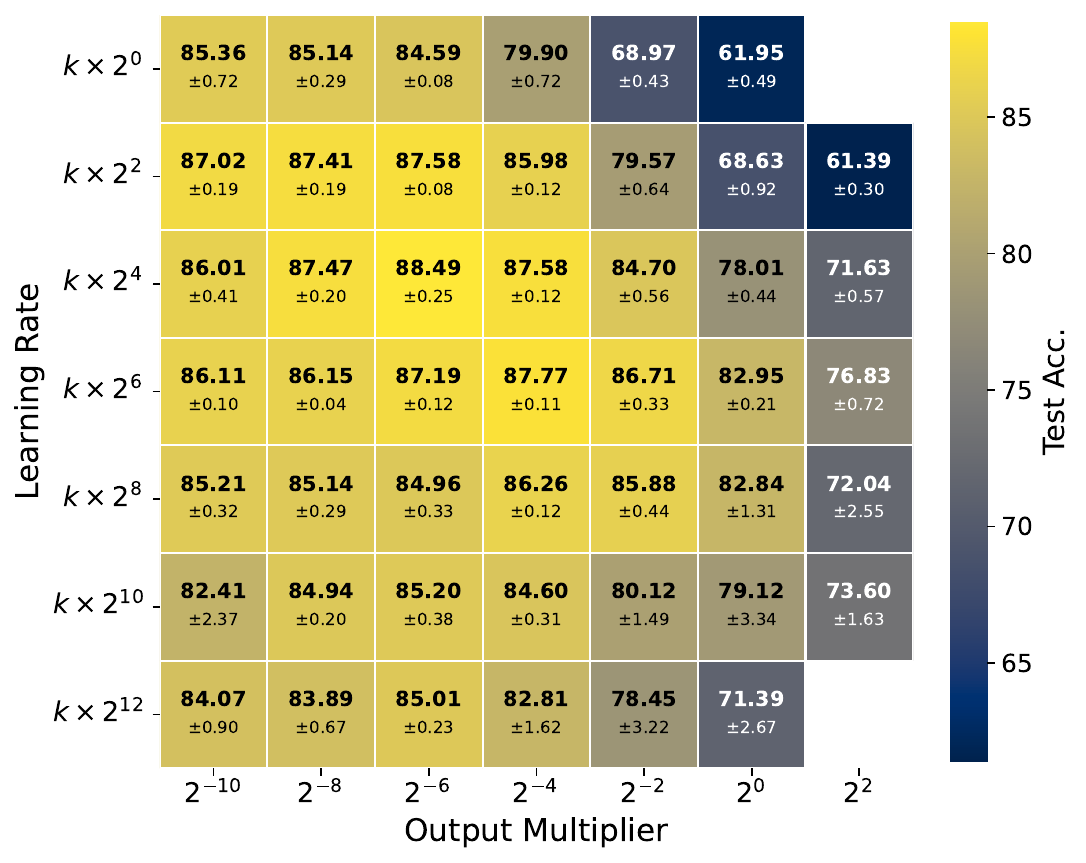}
        \caption{ResNet50}
    \end{subfigure}

    \caption{Peak test accuracy (CIFAR-10).}
    \label{fig:app_cifar10_testacc}
\end{figure*}

\begin{figure*}[!htbp]
    \centering
    \begin{subfigure}[t]{0.25\textwidth}
        \centering
        \includegraphics[width=\linewidth]{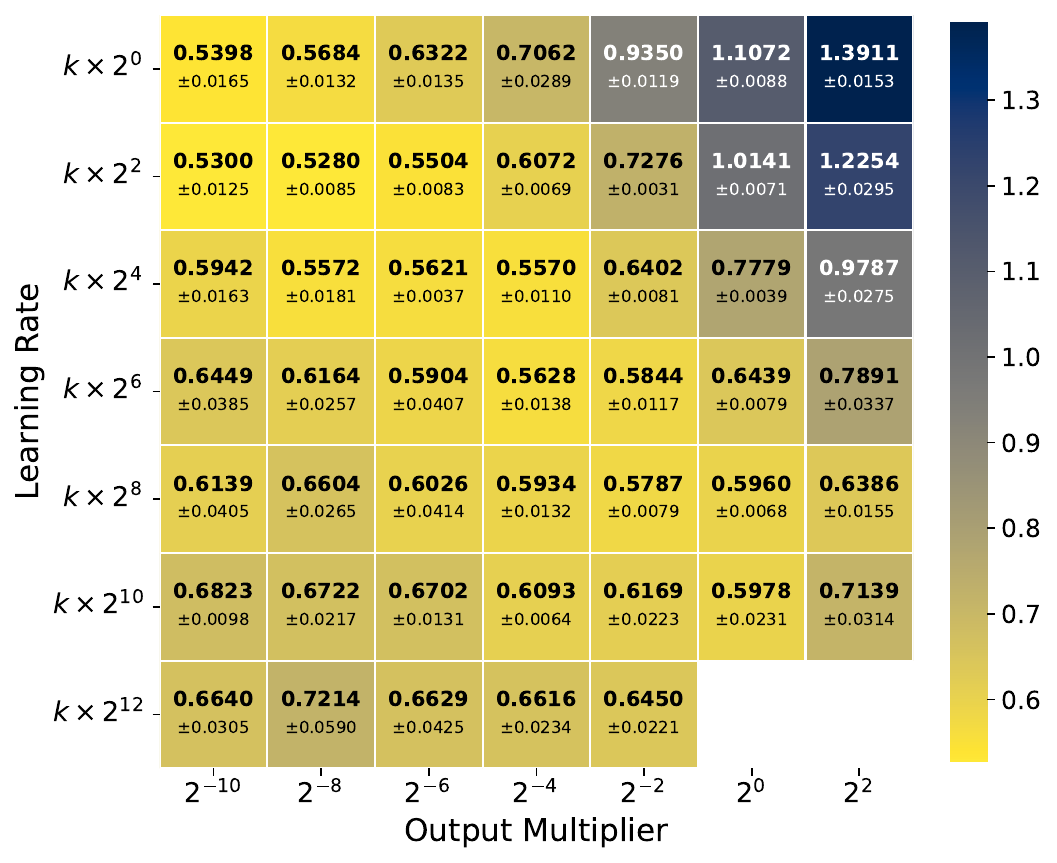}
        \caption{VGG19}
    \end{subfigure}\hfill
    \begin{subfigure}[t]{0.25\textwidth}
        \centering
        \includegraphics[width=\linewidth]{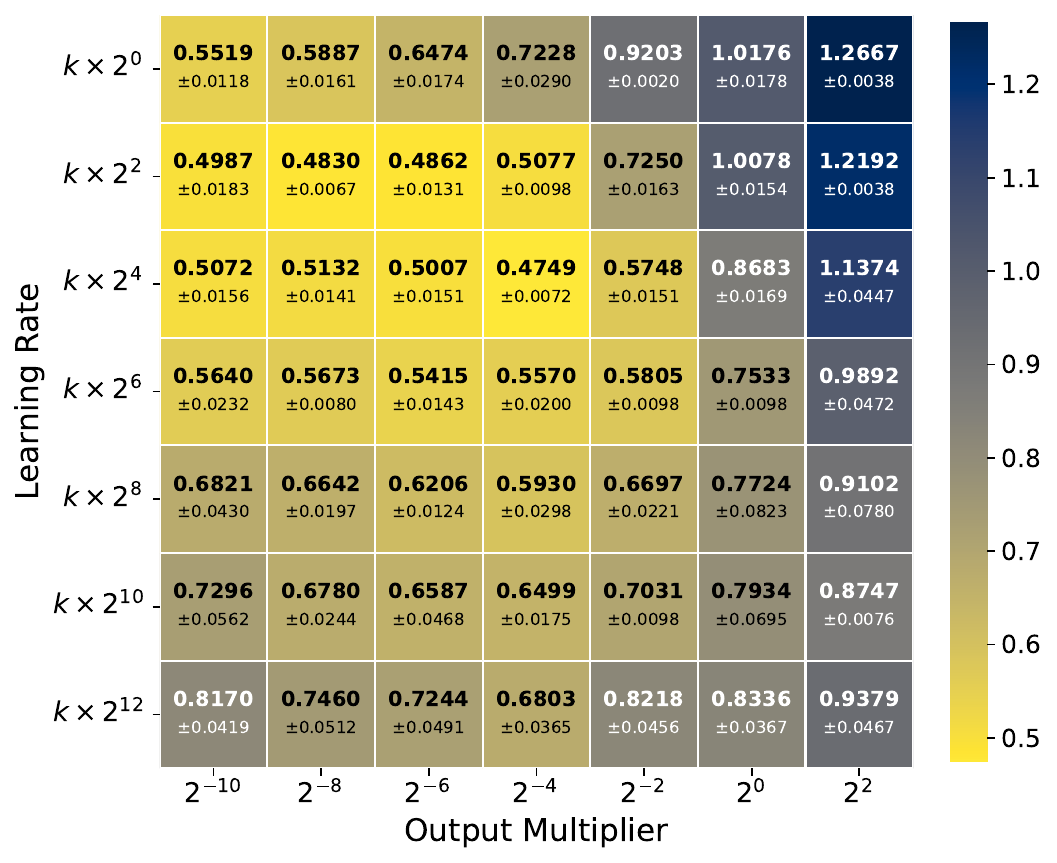}
        \caption{ResNet18}
    \end{subfigure}\hfill
    \begin{subfigure}[t]{0.25\textwidth}
        \centering
        \includegraphics[width=\linewidth]{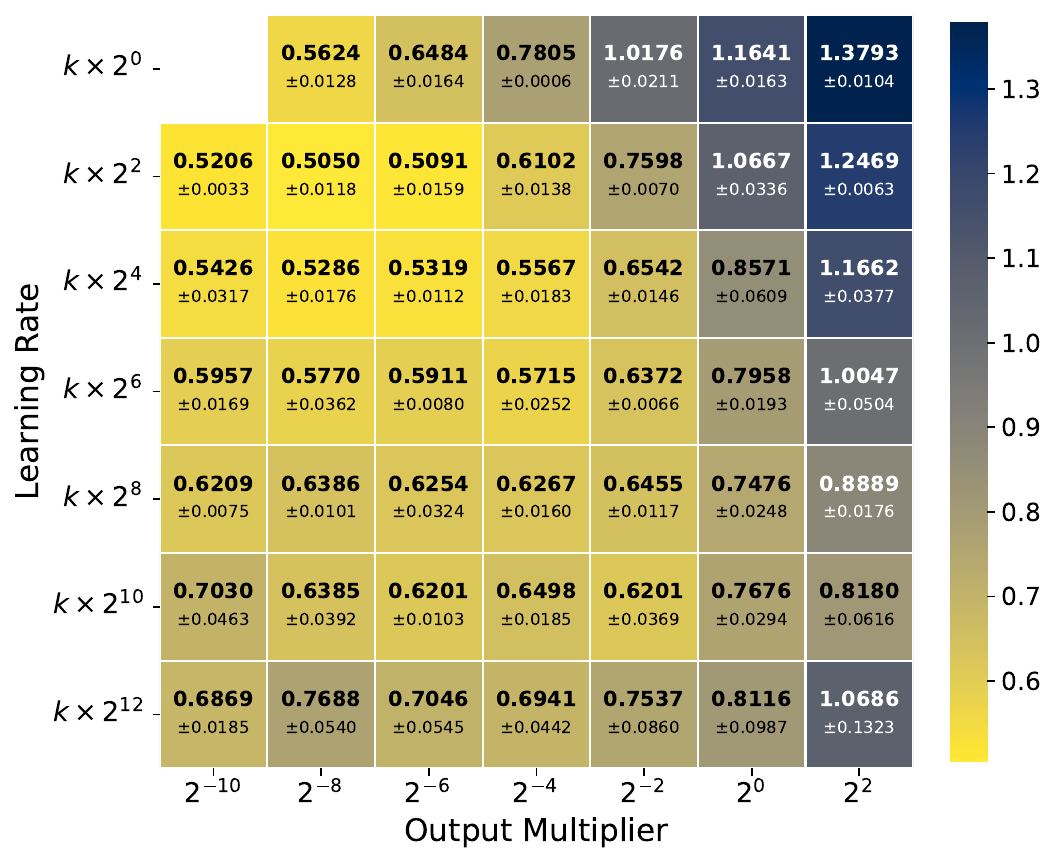}
        \caption{ResNet34}
    \end{subfigure}\hfill
    \begin{subfigure}[t]{0.25\textwidth}
        \centering
        \includegraphics[width=\linewidth]{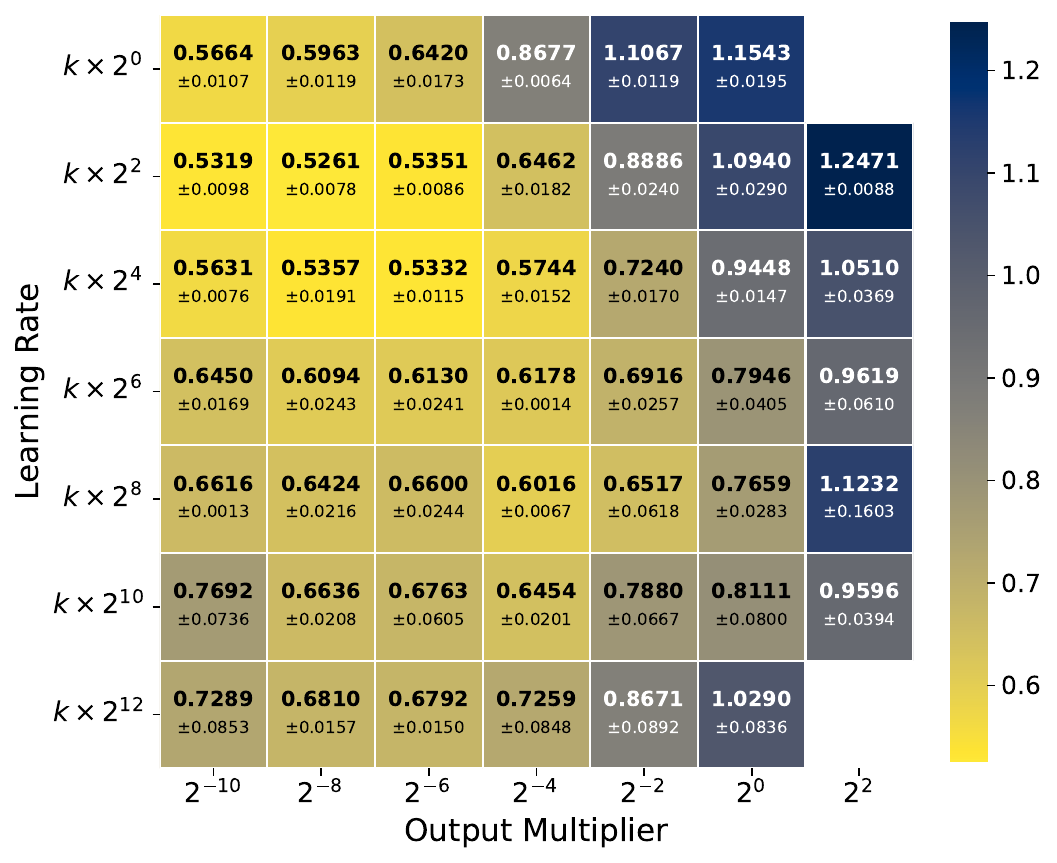}
        \caption{ResNet50}
    \end{subfigure}

    \caption{Best test loss (CIFAR-10).}
    \label{fig:app_cifar10_testloss}
\end{figure*}

\newpage
\subsection{Experiments on Synthetic Image Datasets}\label{app:biggans}
In this subsection, we present additional results on BigGAN-generated dataset with varying effective dimensionality, as a follow-up to \cref{fig:main_biggan} in the main paper. Specifically, we report the best test accuracy and the best (i.e., lowest) test loss achieved during training for VGG19 and ResNet18.

\begin{figure*}[!htbp]
    \centering
    \begin{subfigure}[t]{0.33\textwidth}
        \centering
        \includegraphics[width=\linewidth]{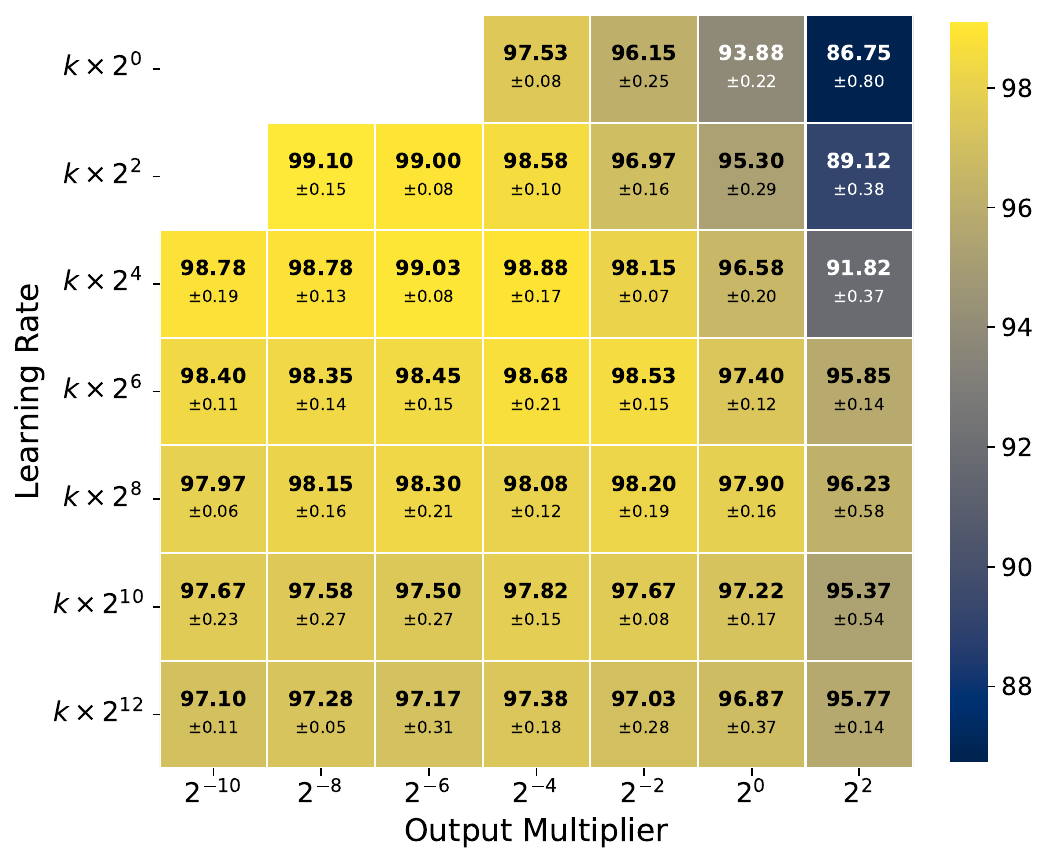}
        \caption{Effective dimension: 32}
    \end{subfigure}\hfill
    \begin{subfigure}[t]{0.33\textwidth}
        \centering
        \includegraphics[width=\linewidth]{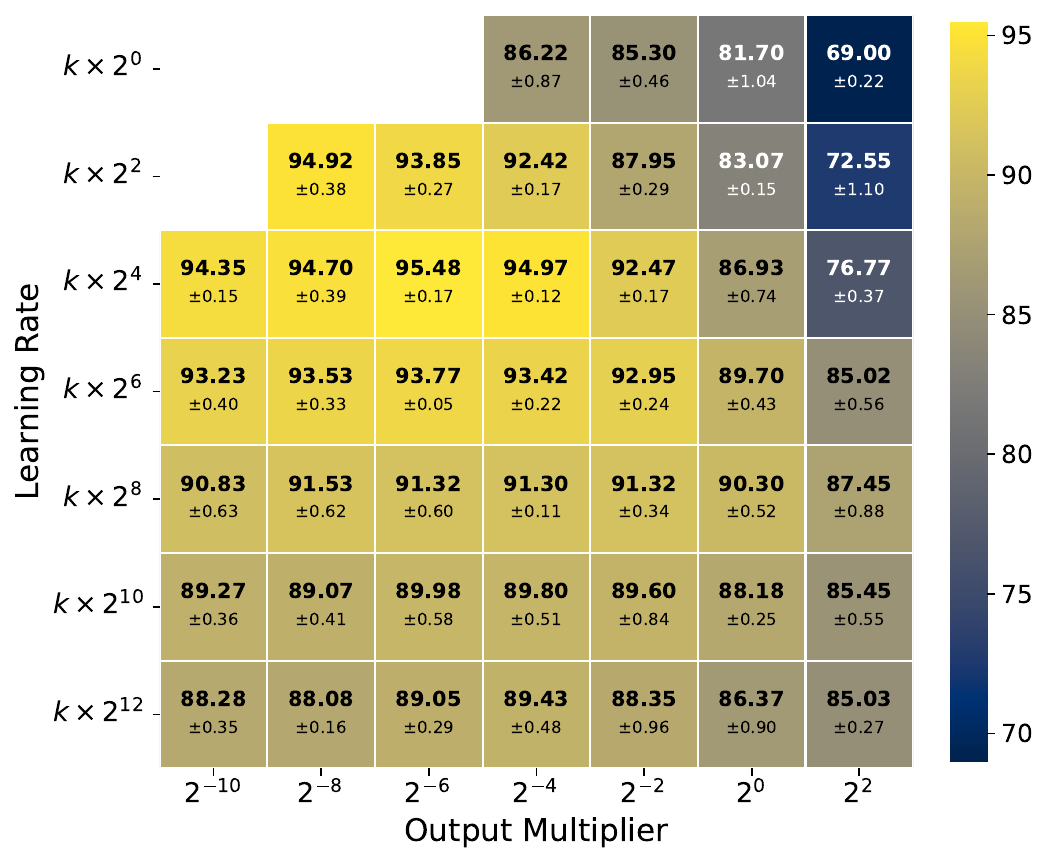}
        \caption{Effective dimension: 64}
    \end{subfigure}\hfill
    \begin{subfigure}[t]{0.33\textwidth}
        \centering
        \includegraphics[width=\linewidth]{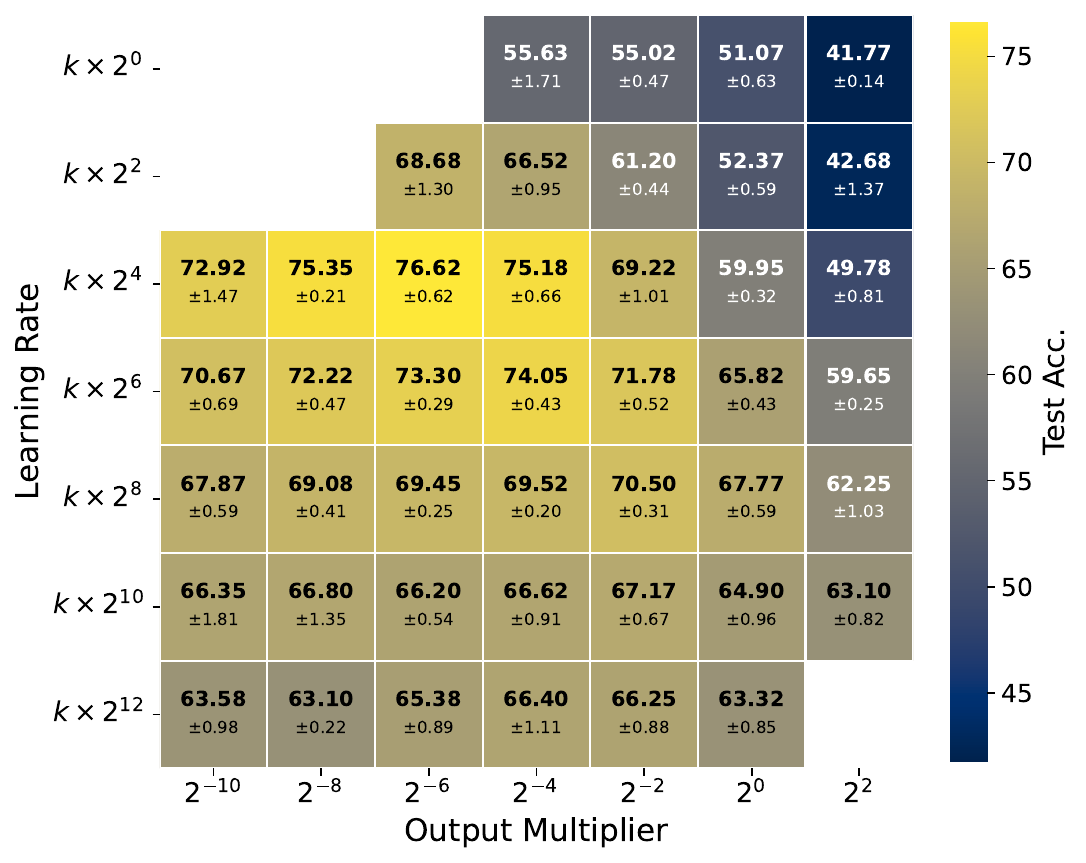}
        \caption{Effective dimension: 128}
    \end{subfigure}

    \caption{\textbf{Peak test accuracy for ResNet18} trained on a BigGAN-generated dataset, with varying effective dimensionality.}
    \label{fig:resnet18_biggan_testacc}
\end{figure*}

\begin{figure*}[!htbp]
    \centering
    \begin{subfigure}[t]{0.33\textwidth}
        \centering
        \includegraphics[width=\linewidth]{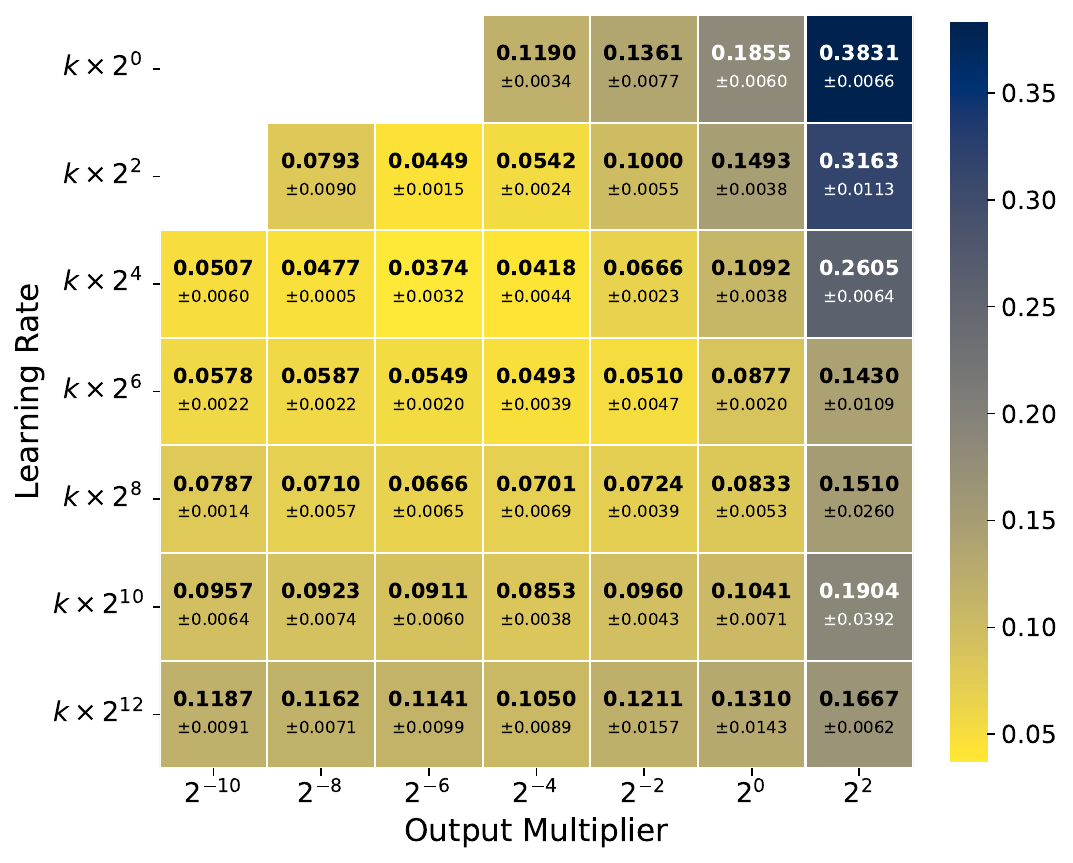}
        \caption{Effective dimension: 32}
    \end{subfigure}\hfill
    \begin{subfigure}[t]{0.33\textwidth}
        \centering
        \includegraphics[width=\linewidth]{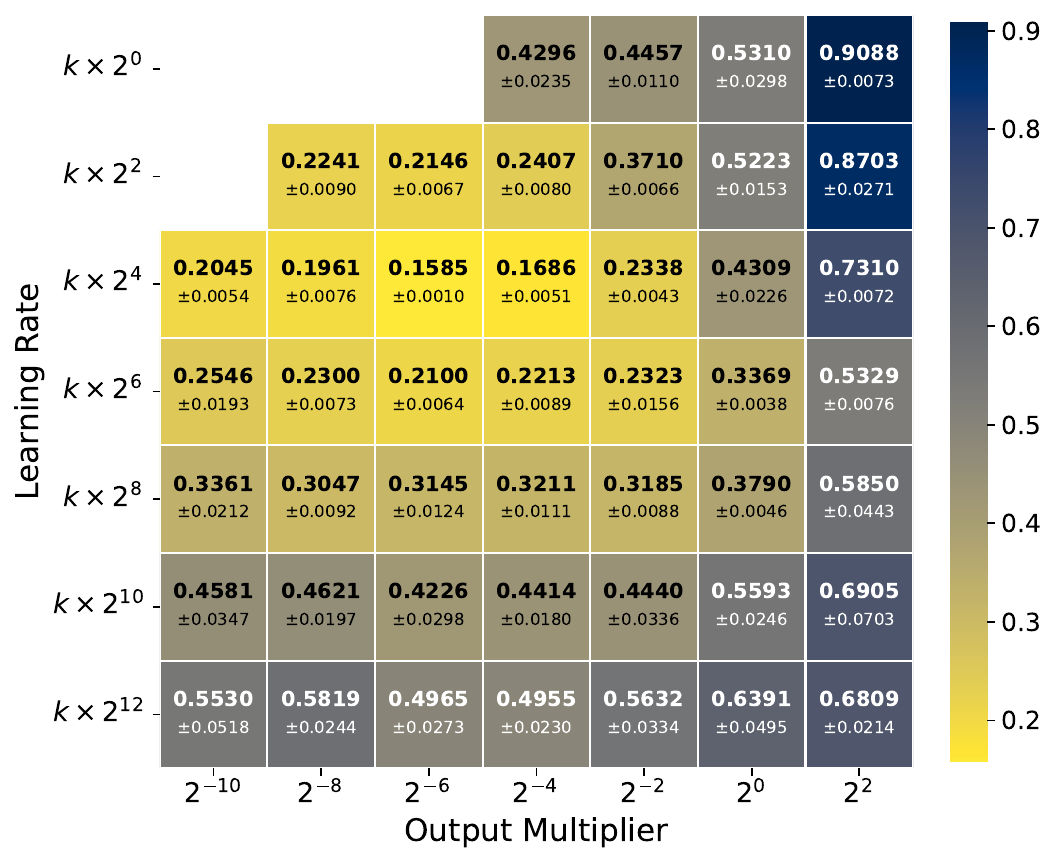}
        \caption{Effective dimension: 64}
    \end{subfigure}\hfill
    \begin{subfigure}[t]{0.33\textwidth}
        \centering
        \includegraphics[width=\linewidth]{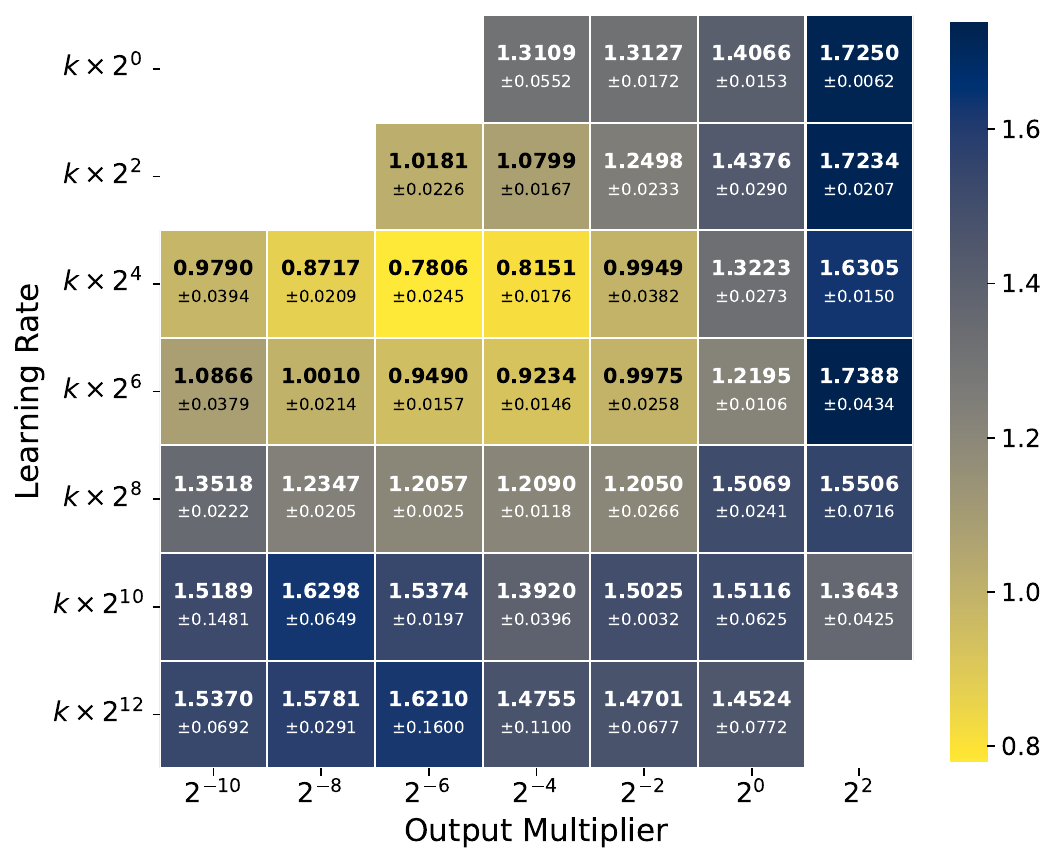}
        \caption{Effective dimension: 128}
    \end{subfigure}

    \caption{\textbf{Best test loss for ResNet18} trained on a BigGAN-generated dataset, with varying effective dimensionality.}
    \label{fig:resnet18_biggan_testloss}
\end{figure*}

\begin{figure*}[!htbp]
    \centering
    \begin{subfigure}[t]{0.33\textwidth}
        \centering
        \includegraphics[width=\linewidth]{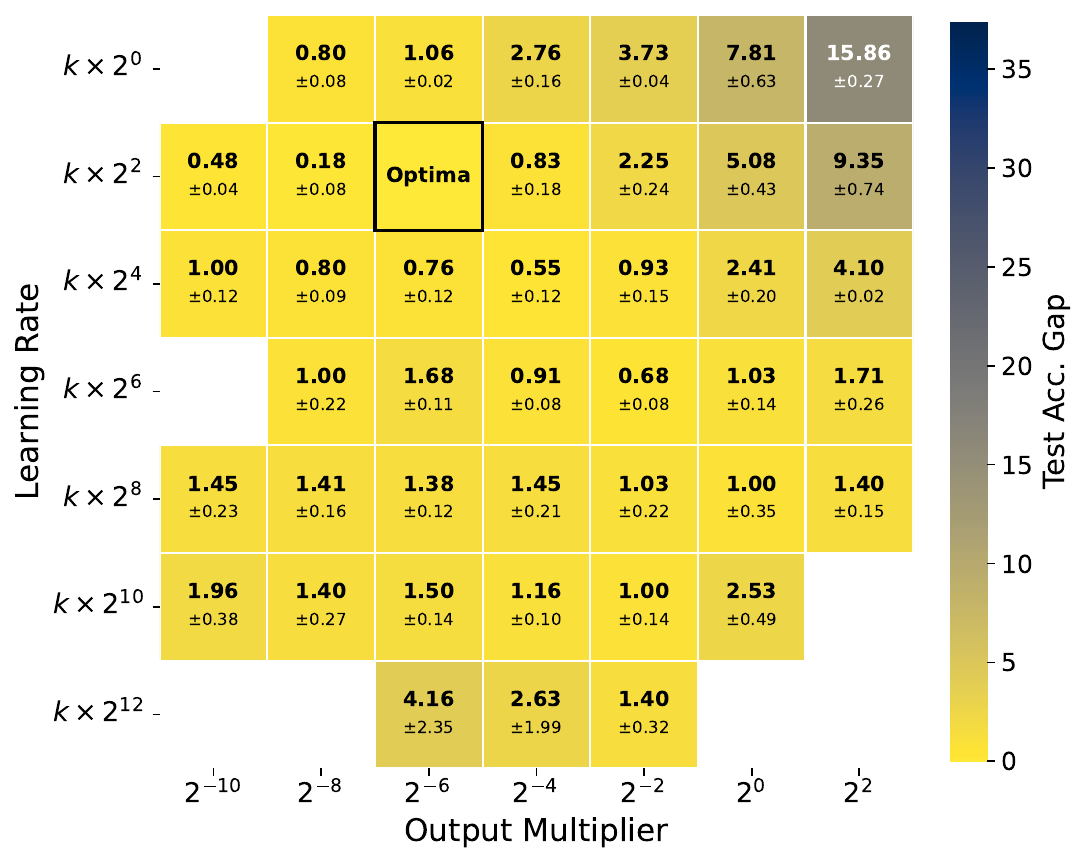}
        \caption{Effective dimension: 32}
    \end{subfigure}\hfill
    \begin{subfigure}[t]{0.33\textwidth}
        \centering
        \includegraphics[width=\linewidth]{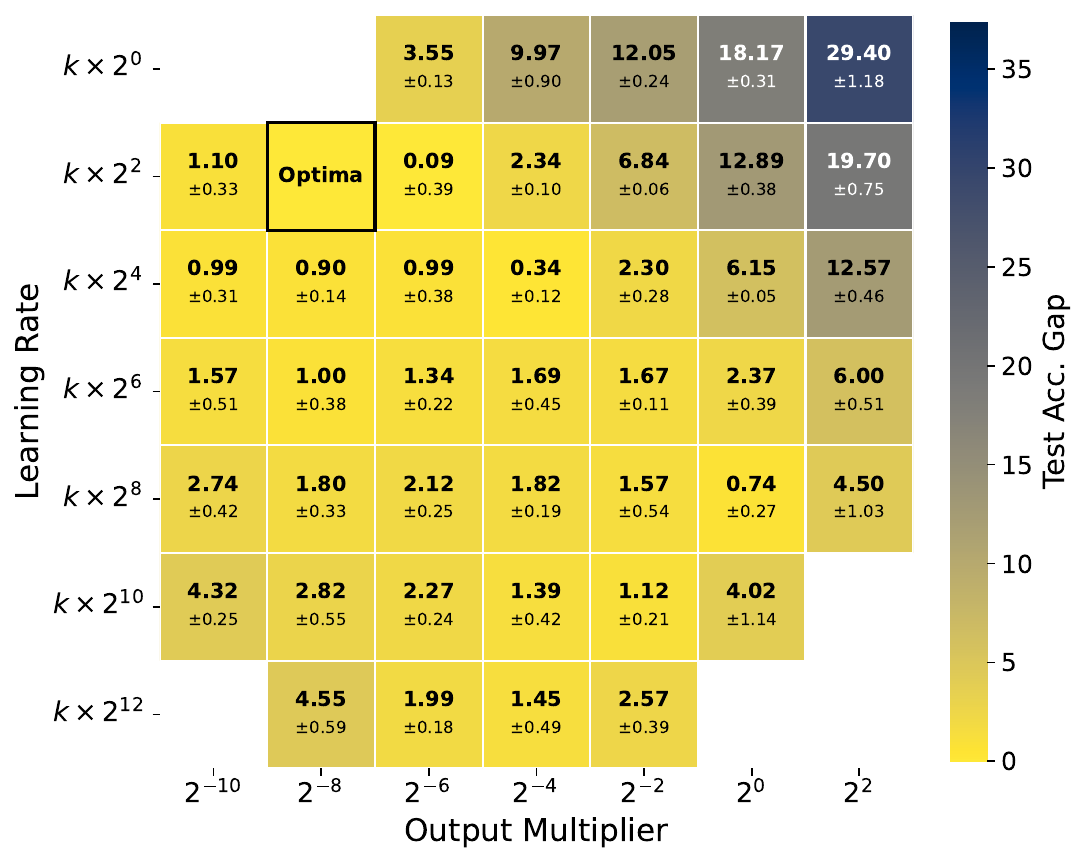}
        \caption{Effective dimension: 64}
    \end{subfigure}\hfill
    \begin{subfigure}[t]{0.33\textwidth}
        \centering
        \includegraphics[width=\linewidth]{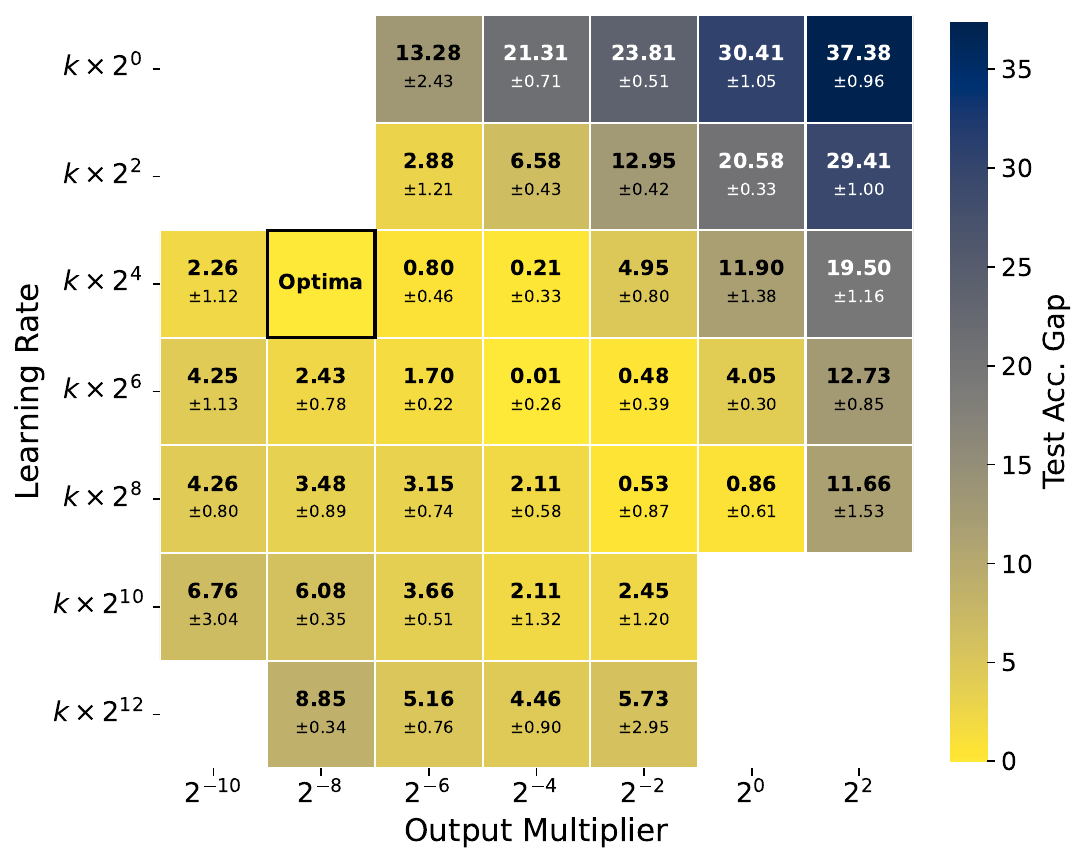}
        \caption{Effective dimension: 128}
    \end{subfigure}

    \caption{\textbf{Gap of the peak test for VGG-19} trained on a BigGAN-generated dataset, with varying effective dimensionality.}
    \label{fig:vgg19_testacc_gap}
\end{figure*}

\begin{figure*}[!htbp]
    \centering
    \begin{subfigure}[t]{0.33\textwidth}
        \centering
        \includegraphics[width=\linewidth]{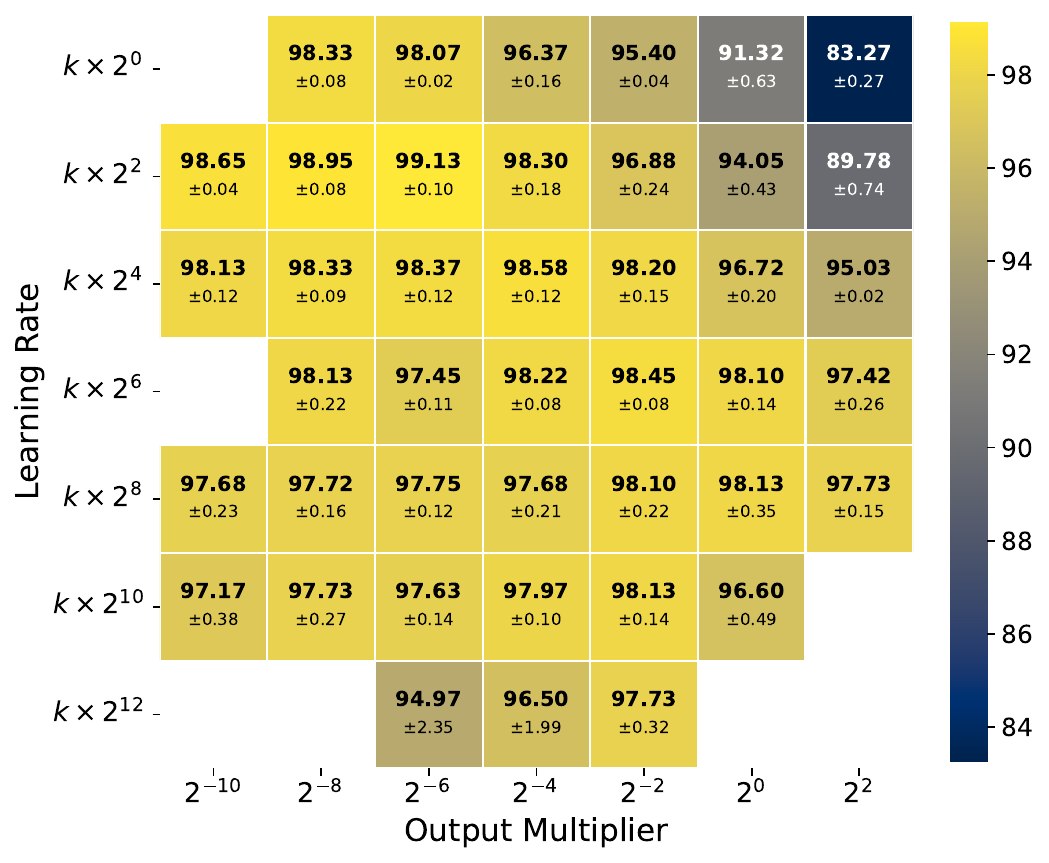}
        \caption{Effective dimension: 32}
    \end{subfigure}\hfill
    \begin{subfigure}[t]{0.33\textwidth}
        \centering
        \includegraphics[width=\linewidth]{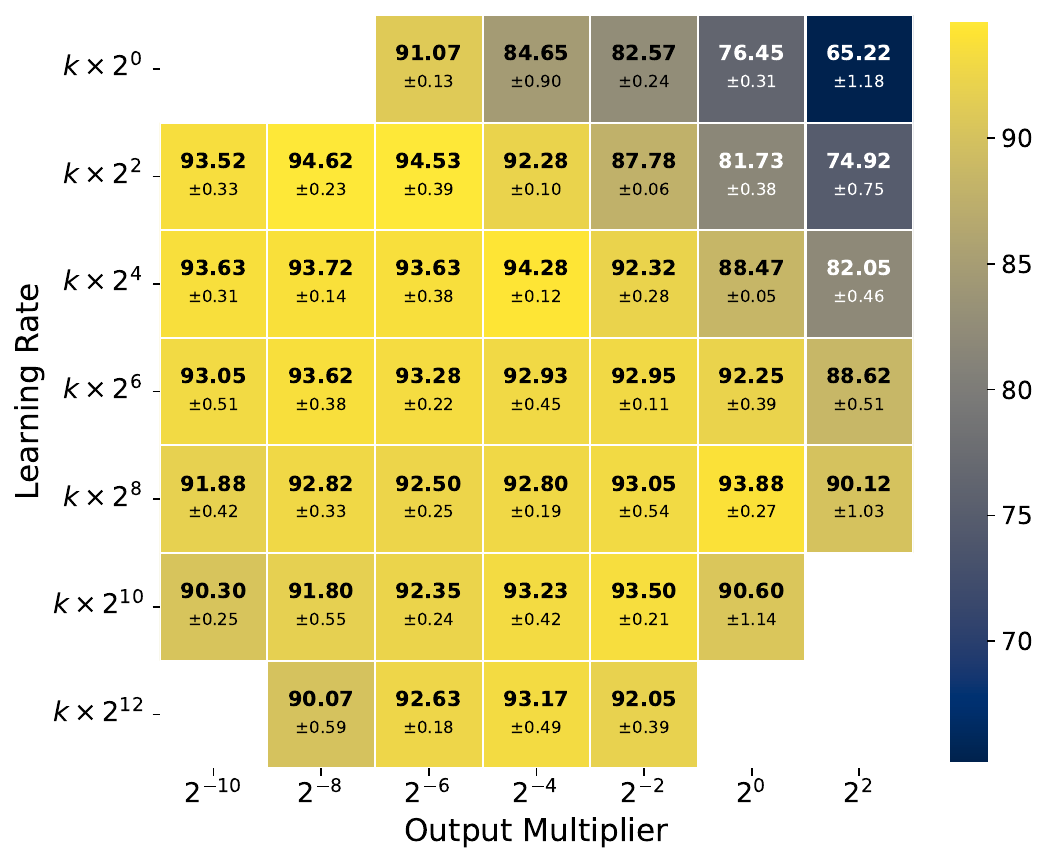}
        \caption{Effective dimension: 64}
    \end{subfigure}\hfill
    \begin{subfigure}[t]{0.33\textwidth}
        \centering
        \includegraphics[width=\linewidth]{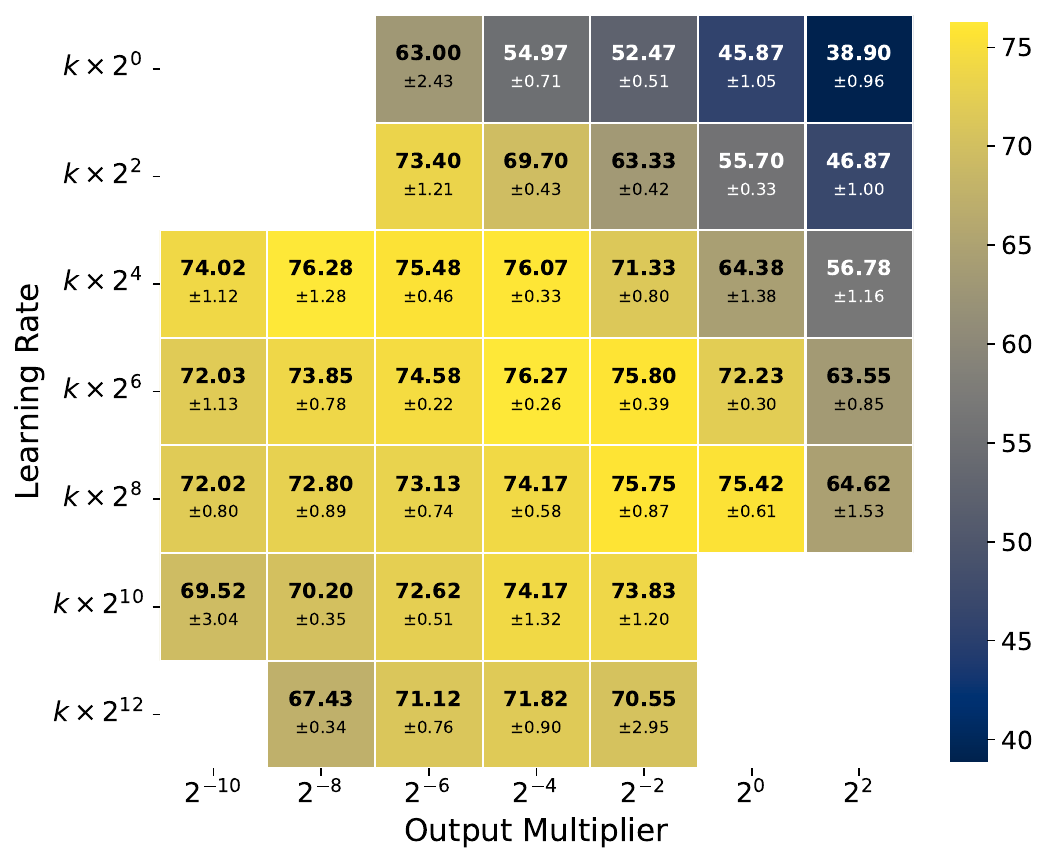}
        \caption{Effective dimension: 128}
    \end{subfigure}

    \caption{\textbf{Peak test accuracy for VGG19} trained on a BigGAN-generated dataset, with varying effective dimensionality.}
    \label{fig:vgg19_testacc}
\end{figure*}

\begin{figure*}[!htbp]
    \centering
    \begin{subfigure}[t]{0.33\textwidth}
        \centering
        \includegraphics[width=\linewidth]{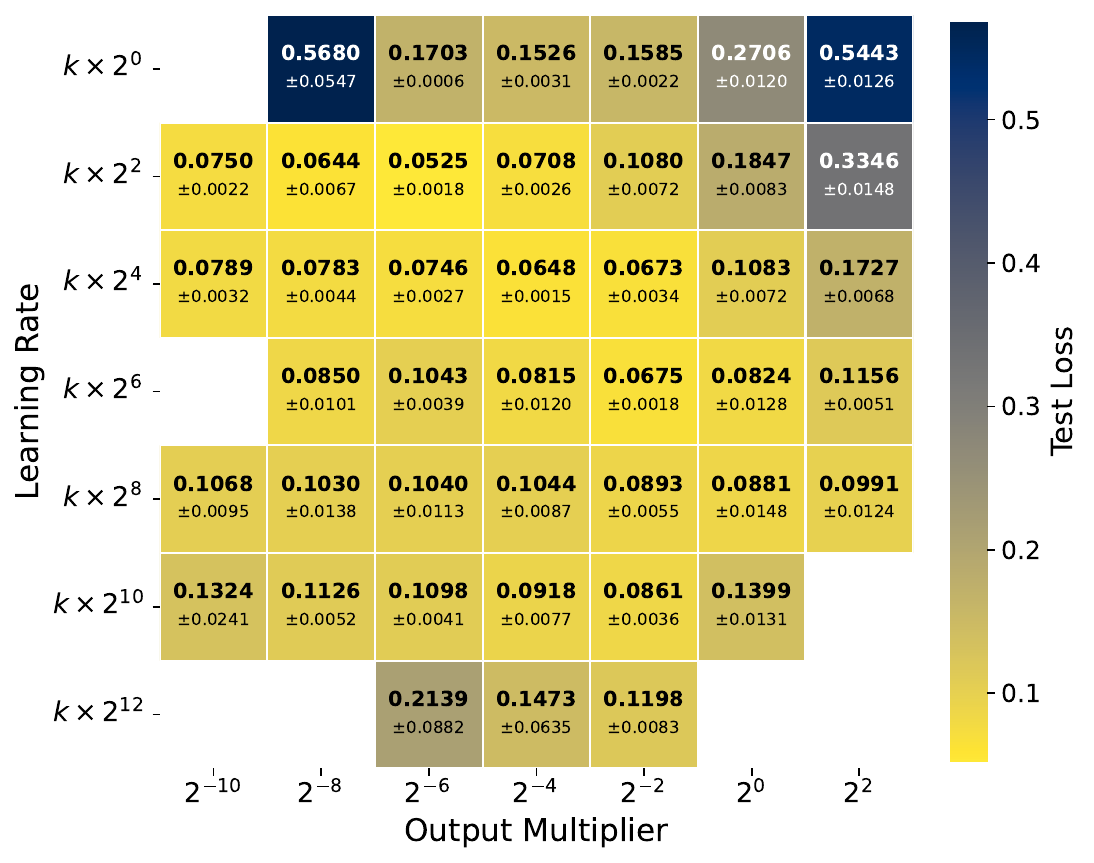}
        \caption{Effective dimension: 32}
    \end{subfigure}\hfill
    \begin{subfigure}[t]{0.33\textwidth}
        \centering
        \includegraphics[width=\linewidth]{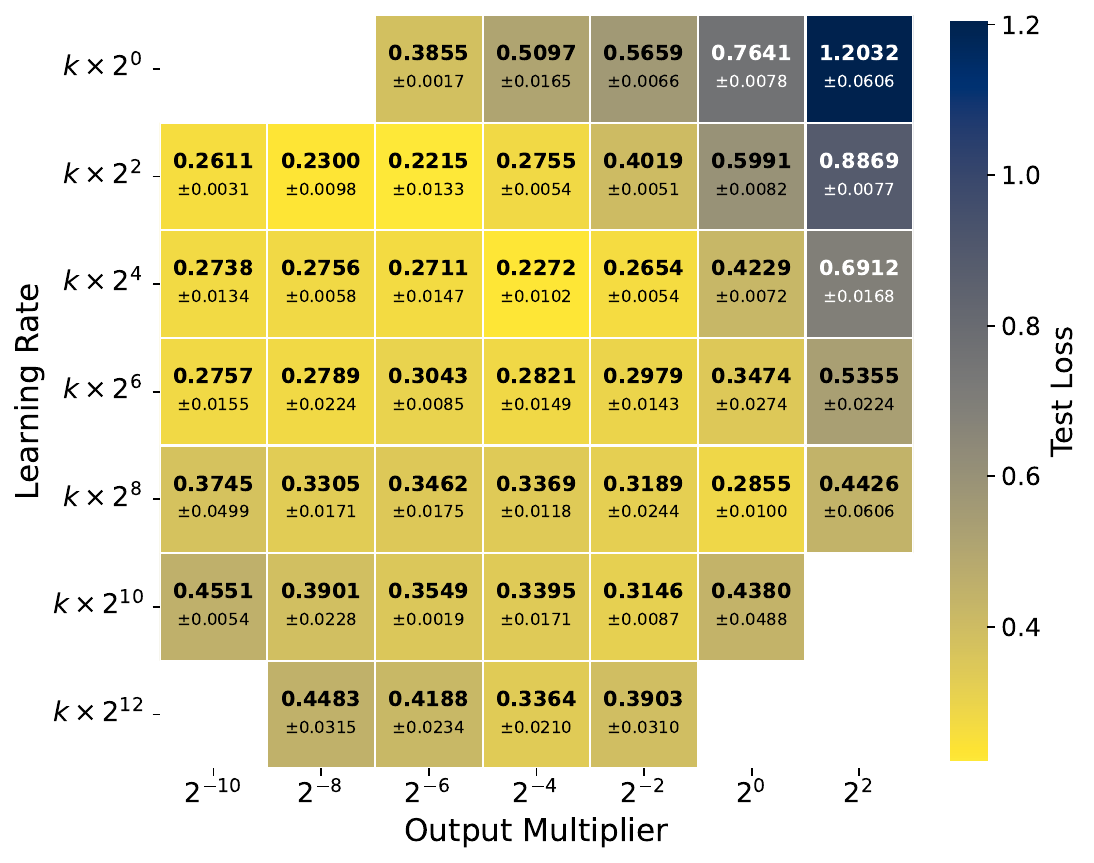}
        \caption{Effective dimension: 64}
    \end{subfigure}\hfill
    \begin{subfigure}[t]{0.33\textwidth}
        \centering
        \includegraphics[width=\linewidth]{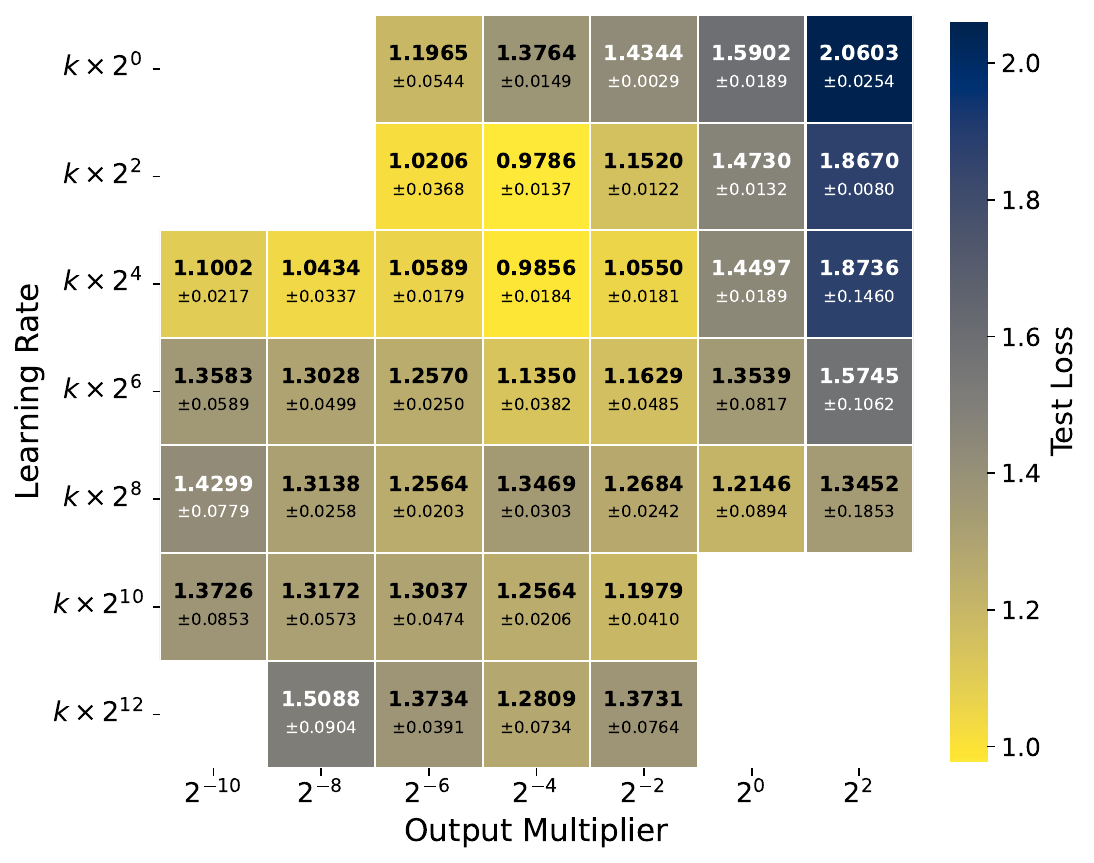}
        \caption{Effective dimension: 128}
    \end{subfigure}

    \caption{\textbf{Best test loss for VGG19} trained on a BigGAN-generated dataset, with varying effective dimensionality.}
    \label{fig:vgg19_testloss}
\end{figure*}

\newpage
\subsection{Details About  Experiments in \cref{sec:theory}}\label{app:simul_detail}
For numerical experiments  (\cref{fig:oa_of} and \cref{fig:ushape}), we use 50 training samples generated from a Gaussian mixture model with $\kappa=1.5$, $\sigma=1$, and 128 input dimension. We use two-layer, bias-free ReLU networks with 64 hidden units, and train all models until the training error reaches $\eta=0.05$. Note that we do not restrict $\lambda$ here in order to reflect a realistic setup.

\subsection{Additional Results in \cref{sec:theory}}\label{app:u_shape}
In this subsection, we plot the excess error (i.e., $\mathsf{OA}(\alpha)+\mathsf{OF}(\alpha)$) from our numerical experiments (\cref{fig:oa_of}). As shown in \cref{fig:ushape}, the excess error exhibits a ``U-shape,'' demonstrating the existence of an optimal FLS.
\begin{figure}[h]
    \centering
    \includegraphics[width=0.38\linewidth]{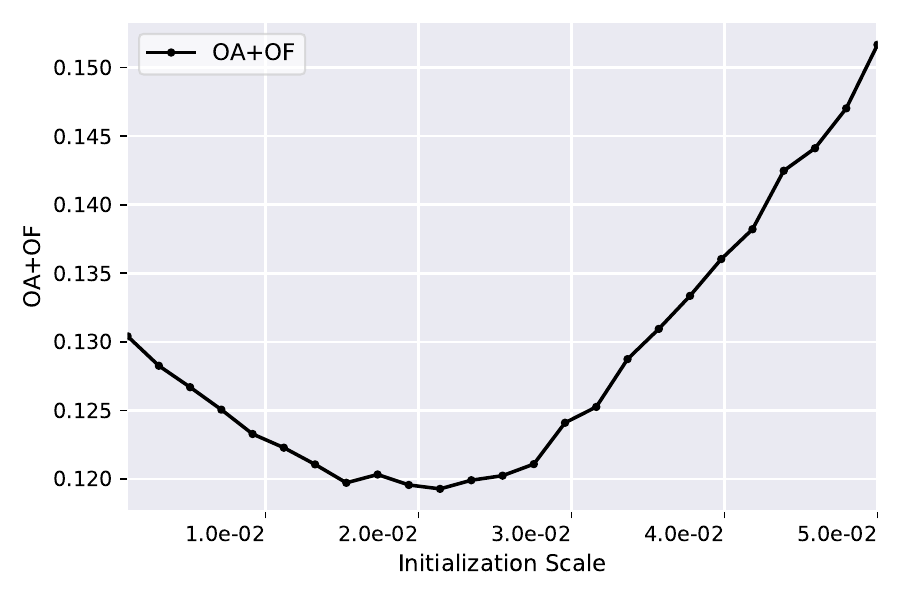}
    \caption{Excess error across different initializations scale ($\eta=0.05$)}
        \label{fig:ushape}
\end{figure}

\textbf{Discussion of the small norm regime.} Our main theorem \cref{thm:rb} theoretically requires $\|\hat{\bfw}_\alpha\|\leq 1$. In \cref{fig:norm}, we present at finite-time training, norm of the predictor indeed tends to be small. If the initialization scale becomes large, the norm can blow up; however, we do not consider this regime in our theorem. However, empirically, we observe that our analysis still holds even with longer training (which also makes the norm blow up) (\cref{fig:further}).

\begin{figure}[h]
    \centering
    \includegraphics[width=0.38\linewidth]{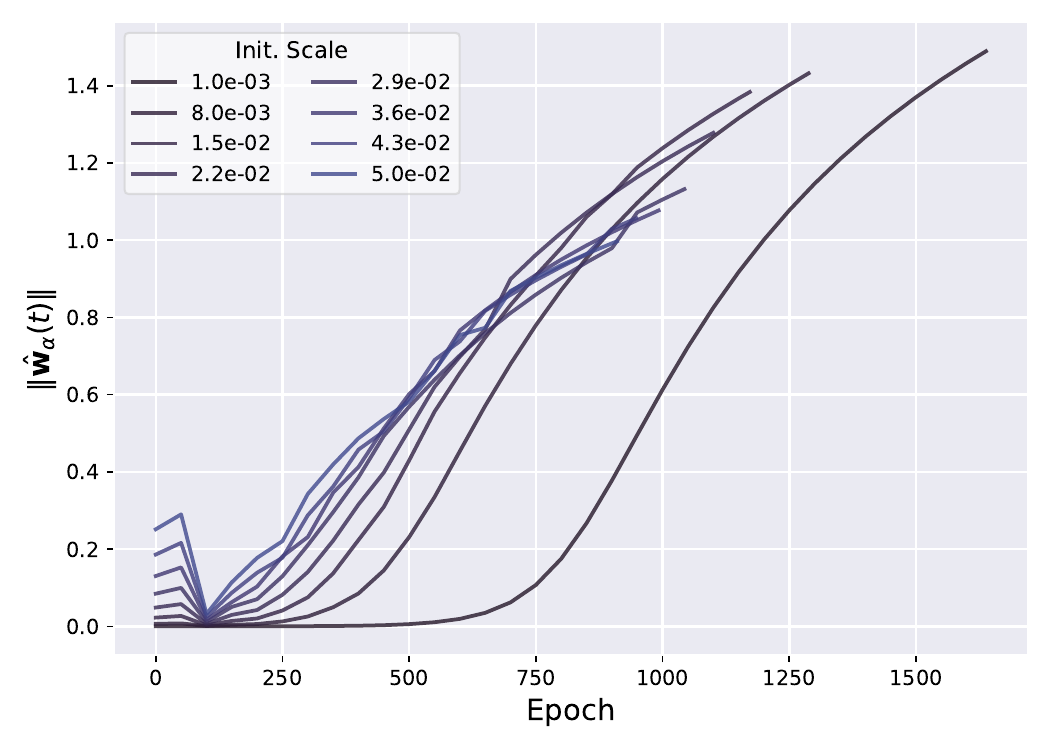}
    \caption{Norm of the effective predictor.}
        \label{fig:norm}
\end{figure}

\begin{figure}[h]
    \centering
    \includegraphics[width=0.38\linewidth]{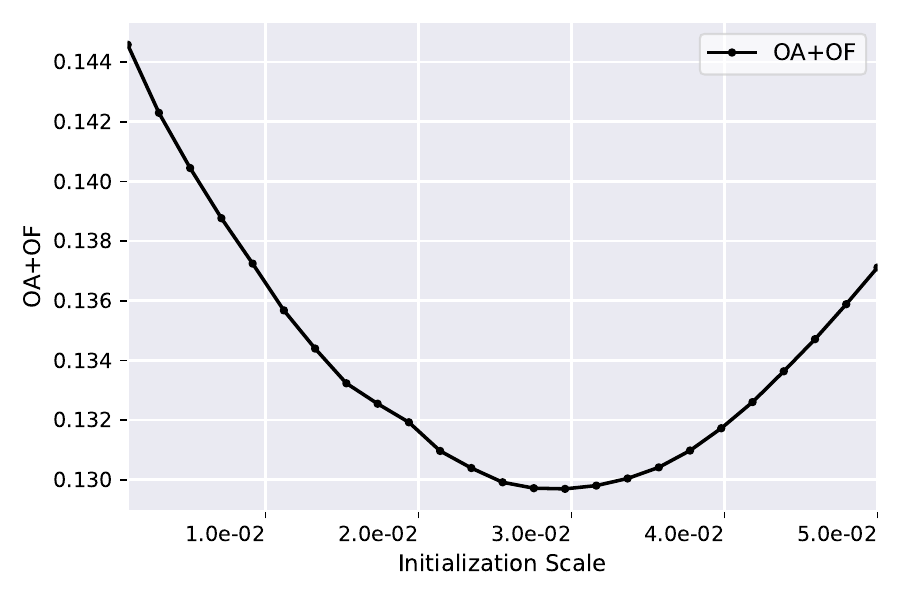}
    \caption{Excess error across different initializations scale  ($\eta=0.01$).}
        \label{fig:further}
\end{figure}
\newpage
\section{Output Scaling vs. Initialization Scaling}\label{app:vs}
In \cref{prop:equivalence}, we show that output scaling with scale-compensated learning rate (used in \cref{sec:dnn}) is exactly equivalent to the initialization scaling (used in \cref{sec:theories}).

\begin{proposition}\label{prop:equivalence} 
Suppose $\bfW=(\bfW_1,\dots,\bfW_L)$ and $\bfW'=(\alpha\bfW_1,\dots,\alpha\bfW_L)$ are the parameters of an $L$-layer bias-free, positively homogeneous network, where $\alpha>0$ is the initialization scaling factor applied to all layers. 
Consider two learning configurations $A_{\bfW}=(\alpha^L,\eta)$ and $A_{\bfW'}=(1,\eta\alpha^2)$, where each pair denotes the output multiplier and the learning rate, respectively. 
Then, under gradient descent (or gradient flow), for all $t\ge 0$ and all $l\in[L]$, we have
$\bfW'_l(t)=\alpha\,\bfW_l(t)$.
In particular, $f_{A_\bfW}(\bfx;\bfW(t))=f_{A_{\bfW'}}(\bfx;\bfW'(t))$ for all $\bfx$ and $t$.
\end{proposition}

\begin{proof}
Note that $f_{\bfW'}(\bfx)=\alpha^L f_{\bfW}(\bfx)$. Moreover, for each $l\in[L]$, we have
\begin{align}
\nabla_{\bfW'_\ell} f_{\bfW'}(\bfx)
=(\bfW'_L\cdots \bfW'_{\ell+1})(\bfW'_{\ell-1}\cdots \bfW'_1 \bfx)^\top
=\alpha^{L-1}\nabla_{\bfW_\ell} f_{\bfW}(\bfx).
\end{align}
Let $\gamma=\alpha^L$ and define losses $ \hat{L}_\gamma(\bfW)=\frac1n\sum_{i=1}^n \ell(\gamma f_{\bfW}(\bfx_i),y_i)$
and $\hat{L}_1(\bfW')=\frac1n\sum_{i=1}^n \ell(f_{\bfW'}(\bfx_i),y_i)$.
Since $f_{\bfW'}(\bfx_i)=\gamma f_{\bfW}(\bfx_i)$, we have
$\ell'(f_{\bfW'}(\bfx_i),y_i)=\ell'(\gamma f_{\bfW}(\bfx_i),y_i)$, and thus
\begin{align}
\nabla_{\bfW_\ell}\hat{L}_\gamma(\bfW)
&=\frac1n\sum_{i=1}^n \ell'(\gamma f_{\bfW}(\bfx_i),y_i)\cdot \gamma \cdot \nabla_{\bfW_\ell} f_{\bfW}(\bfx_i),\\
\nabla_{\bfW'_\ell}\hat{L}_1(\bfW')
&=\frac1n\sum_{i=1}^n \ell'(\gamma f_{\bfW}(\bfx_i),y_i)\cdot \nabla_{\bfW'_\ell} f_{\bfW'}(\bfx_i)
=\frac{1}{\alpha}\nabla_{\bfW_\ell}\hat{L}_\gamma(\bfW).
\end{align}

With $\eta'=\eta\alpha^2$, the GD update gives
\begin{align}
\bfW'_\ell({t+1})&=\bfW'_\ell(t)-\eta'\nabla_{\bfW'_\ell}\hat{L}_1(\bfW'(t))\\
&=\alpha \bfW_\ell(t)-\eta\alpha^2\cdot\frac1\alpha \nabla_{\bfW_\ell}\hat{L}_\gamma(\bfW(t)) \\
&=\alpha \bfW_\ell({t+1}), \label{eq:equiv}
\end{align}
and the same calculation applies to gradient flow. Since $\bfW'(0)=\alpha \bfW(0)$, induction yields
$\bfW'_\ell(t)=\alpha \bfW_\ell(t)$ for all $t\ge0$ and $\ell\in[L]$, hence
$f_{A_\bfW}(\bfx;\bfW(t))=f_{A_{\bfW'}}(\bfx;,\bfW'(t))$ for all $\bfx$ and $t$, and this completes the proof.
\end{proof}

\begin{remark}
    For $L=2$, analyzing $A_{\bfW}=(\alpha^2, \eta/\alpha^2)$ is equivalent to analyzing $A_{\bfW'}=(1,\eta)$, which matches our setup. Note that in \cref{prop:equivalence}, the distribution of each layer's weights is unconstrained. In our manuscript, we consider a setup where the distribution of the second layer weights is conditioned on the sampled first layer weights, which is included as a special case of \cref{prop:equivalence}.
\end{remark}

\vspace{40pt}
\subsection{Feature Learning Strength and Maximal Update Parametrization}\label{app:flsvsmup}
Maximal update parametrization \citep{yang2021tpiv}, also known as $\mu$P, controls the FLS through the output multiplier $a:=(b\sqrt{\mathrm{width}})^{-1}$. Since scalar $b>0$ is a \textit{free variable} (see \citet{karkada2024lazy} for more details), this can equivalently be rewritten in terms of the output multiplier $c=(b\sqrt{\mathrm{width}})^{-1}$ as in our experiments. Therefore, the two formulations ultimately capture the same underlying notion.



\newpage
\section{Proofs in \cref{sec:theories}}
In this section, we write $\mathcal{I}_+:=\{i\in[n]:y_i =+1\} $.

\subsection{Proof of \cref{lem:phase1_lb}}\label{ssec:pf_phase1_lb}
Consider the quantity $\bfx_a(\bfw) := \sum_{i:\langle \bfx_i,\bfw\rangle > 0} y_i\mathbf{x}_i$. Then, for all $t\in [t_1, t_{\alpha}]$, we know that
\begin{align}
\bfx_a(\bfw_j(t))=\sum_{i:\langle \bfx_i,\bfw_j(t)\rangle>0}y_i\bfx_i=\sum_{i\in\mathcal{I}_+} \bfx_i=\bfx_+.
\end{align}
Hence, $\left\langle {\bfx_+}/{\|\bfx_+\|}, {\bfx_a(\bfw_j)}/{\|\bfx_a(\bfw_j)\|}\right\rangle=1$.
Then, by \citet[Lemma~10]{min2024early}, we have
\begin{align}
    \left| \frac{d}{dt}\nang(t) - \left(1-\nang^2(t)\right) \|\bfx_+\| \right|
    \leq 2n \bfx_{\max} \max_{i} \left|f(\bfx_i; \mathbf{W}(t), \mathbf{v}(t))\right|. \label{eq:phase1_ode}
\end{align}
This yields
\begin{align}
    \frac{d}{dt}\nang(t)   
    &\geq \left(1-\nang^2(t)\right) \|\bfx_+\| -2n \bfx_{\max} \max_{i} \left|f(\bfx_i;\mathbf{W}(t), \mathbf{v}(t))\right| \\
    & \geq \left(1-\nang^2(t)\right) \|\bfx_+\| - \nu\alpha\|\bfx_+\|, \label{eq:9_fork} \\
    & = \|\bfx_+\| \left(1-\nu\alpha - \nang^2(t) \right), \label{eq:dummy1}
\end{align}
where \cref{eq:9_fork} is due to \cref{lem:lem3_min}. Now, we consider two cases:
\begin{itemize}[leftmargin=*,topsep=0pt,parsep=0pt]
\item \textbf{Case 1.} We have  $\nang(t) < \sqrt{1-\nu\alpha}$ for all $t \in [t_1, t_{\alpha}]$.
\item \textbf{Case 2.} There exists some $t \in [t_1, t_{\alpha}]$ such that $\nang(t) \geq \sqrt{1-\nu\alpha}$.
\end{itemize}

For \textbf{case 1}, \cref{lem:techlem1} implies that
\begin{align}
\|\mathbf{x}_+\| \sqrt{1-\nu\alpha} \le \frac{d}{dt} \arctanh\left(\frac{\nang(t)}{\sqrt{1-\nu\alpha}}\right).
\end{align}
Integrating both sides over $[t_1,t_\alpha]$, we get
\begin{align}
(t_\alpha-t_1)\|\mathbf{x}_+\| \sqrt{1-\nu\alpha}  &\le \arctanh\left(\frac{\nang(t_\alpha)}{\sqrt{1-\nu\alpha}}\right) - \arctanh\left(\frac{\nang(t_1)}{\sqrt{1-\nu\alpha}}\right)\\
&\le \arctanh\left(\frac{\nang(t)}{\sqrt{1-\nu\alpha}}\right)
\end{align}
where the second inequality holds as $\nang(t_1)$ is nonnegative. Taking $\tanh(\cdot)$ on both sides and scaling, we get the claim.

For \textbf{case 2}, we cannot directly apply \cref{lem:techlem1}, as the $\arctanh(x)$ is undefined for $x \ge 1$. Instead we consider
\begin{align}
\widetilde{\nang}(t):=\min\left\{\nang(t),\sqrt{1-\nu\alpha}\right\}.
\end{align}
Then, we know that the differential inequality in \cref{lem:techlem1} holds almost everywhere. Indeed, if $\nang(t) < \sqrt{1-\nu\alpha}$, then $\widetilde{\nang}(t) = \nang(t)$; otherwise, both the LHS and the RHS equals to zero. Noticing $\widetilde{\nang}(t) \le \nang(t)$, we get the first claim.

Moreover, sign preservation (\cref{lem:gf_property}) induced by \cref{assm:ortho_sep}, we have $v_j(t_{\alpha}) \geq 0$ for all $j \in V_+$. Since the effective predictor lies in the conical hull of $\mathbf{W}(t_{\alpha})$, a property that holds independent of the specific non-negative values of $v_j(t_{\alpha})$. Thus, we obtain the same lower bound for the effective predictor.

\subsection{Proof of \cref{cor:phase1_angle}}\label{ssec:pf_phase1_angle}
We can proceed as:
\begin{align}
\sin^2(\angle(\bfw_j(t_\alpha),\bfx_+ )) &= 1- \psi_j(t_{\alpha})^2\\
&\le 1- \left({1-\nu\alpha}\right)\;
\tanh^2\left((t_{\alpha}-t_1)\|\bfx_+\|\sqrt{1-\nu\alpha}\right) \\
&= {\nu\alpha} + \left({1-\nu\alpha}\right) \left(1-\tanh^2\left((t_{\alpha}-t_1)\|\bfx_+\|\sqrt{1-\nu\alpha}\right)\right) \\
&= \nu\alpha + \left({1-\nu\alpha}\right)\mathrm{sech}^2\left((t_{\alpha}-t_1)\|\bfx_+\|\sqrt{1-\nu\alpha}\right) \\
& \le \nu\alpha + 4\left({1-\nu\alpha}\right) \exp\left(-2(t_{\alpha}-t_1)\|\bfx_+\|\sqrt{1-\nu\alpha}\right)\\ 
&= \nu\alpha + 4\left({1-\nu\alpha}\right)\exp \left(2t_1\|\bfx_+\| \sqrt{1-\nu\alpha}\right) \exp \left(-2t_{\alpha}\|\bfx_+\| \sqrt{1-\nu\alpha} \right) \\
&= \nu\alpha + 4\left({1-\nu\alpha}\right)\exp \left(2t_1\|\bfx_+\| \sqrt{1-\nu\alpha}\right)\cdot (h\alpha)^{\frac{\|\bfx_+\| \sqrt{1-\nu\alpha}}{4n \bfx_{\max}}}.\label{eq:rate}
\end{align}
Here, the first inequality follows from \cref{lem:phase1_lb}, the second inequality follows from the fact that $\text{sech}^2(x) \leq 4\exp(-2x)$. Thus, we have
\begin{align}
\angle(\alpha) &\le \arcsin\left(\sqrt{\nu\alpha + 4\left({1-\nu\alpha}\right)\exp \left(2t_1\|\bfx_+\| \sqrt{1-\nu\alpha}\right)\cdot (h\alpha)^{\frac{\|\bfx_+\| \sqrt{1-\nu\alpha}}{4n \bfx_{\max}}}}\right)\\
&\le \frac{\pi}{2}\sqrt{\nu\alpha} + \pi\sqrt{{1-\nu\alpha}}\exp \left(t_1\|\bfx_+\| \sqrt{1-\nu\alpha}\right)\cdot (h\alpha)^{\frac{\|\bfx_+\| \sqrt{1-\nu\alpha}}{8n \bfx_{\max}}}\\
&\le C_1\sqrt{\alpha} + C_2 \alpha^{\|\bfx_+\|/8n\bfx_{\max}}
\end{align}
for some $C_1, C_2 > 0$.
In terms of $\psi_j(t_{\alpha})$ itself, we have $\psi_j(t_{\alpha})=1-O(\alpha-\alpha^{k})\approx 1-O(\alpha)$, where $k=\|\bfx_+\|/4n\bfx_{\max}$.

\newpage

\newpage

\newpage
\subsection{Proof of \cref{lem:phase2_bound}}

Since $t\geq t_\alpha$, we have $\psi_j(t) \geq \lambda$. For any $t \geq t_\alpha$, we have
\begin{align}
    \frac{d}{dt} \psi_j(t) &\geq \left(\left\langle \frac{\bfx_+}{\|\bfx_+\|}, \frac{\sum_{i\in \mathcal{I}_+}c_i(t)\bfx_i}{\|\sum_{i\in \mathcal{I}_+}c_i(t)\bfx_i\|}\right\rangle - \psi_j(t)\right)\left\|\sum_{i\in \mathcal{I}_+}c_i(t)\bfx_i \right\| \\
    & \geq \left(\lambda- \psi_j(t)\right)\left\|\sum_{i\in \mathcal{I}_+}c_i(t)\bfx_i \right\|. \label{eq:ang_to_sep}
\end{align}
where $c_i(t)=-\nabla\ell\left(y_i, f(\bfx_i, \theta_t)\right)$ denotes the loss gradient of $i$th sample at time $t$, and we have used Lemma 5 of \citet{min2024early} for the last inequality. From now on, we write $\bfx_c (t)=\sum_{i\in \mathcal{I}_+}c_i(t)\bfx_i $ for notational simplicity.

Now, let $\delta(t)=\psi(t)-\lambda$. Then, multiplying $G(t):=\int_{t_\alpha}^t\|\bfx_c(\tau)\|d\tau$ for both sides and differentiation gives
\begin{align}
    \frac{d}{dt}(\delta(t)\cdot\exp(G(t))&=\exp(G(t))\frac{d}{dt}\delta(t) + \delta(t) \exp(G(t))\frac{d}{dt}G(t) \\
    &=\exp(G(t))\frac{d}{dt}\delta(t) + \delta(t) \exp(G(t))\|\bfx_c(t)\|\\
    & \geq 0,
\end{align}
where we use \cref{eq:ang_to_sep} for the last inequality. Integrating both sides from $t_\alpha$ to $t_{\eta,\alpha}$, we get
\begin{align}
\delta(t_{\eta,\alpha})\exp(G(t_{\eta,\alpha})) -\delta(t_{\alpha}) \geq 0,
\end{align}
where we use $\exp(G(t_\alpha))=1$. Rewriting the terms, we have
\begin{align}
    \psi(t_{\eta,\alpha}) \ge \lambda + (\psi(t_\alpha)-\lambda)\exp(-G(t_{\eta,\alpha})).
\end{align}
Plugging the results of \cref{lem:dummy2}, we get what we want.

\newpage
\begin{lemma}\label{lem:dummy}
    We have
    \begin{align}
        \|\bfx_c(t)\| \leq \bfx_{\max}n \hat{L}(t).
    \end{align}
\end{lemma}
\begin{proof}
For $y_i=+1$, logistic loss satisfies
\begin{align}
c_i(t)=-\partial_{f}\ell(+1,f(\bfx_i))=\frac{1}{1+\exp(f_{\theta_t}(\bfx_i))}=\frac{u_i(t)}{1+u_i(t)}.
\end{align}
where we let $u_i(t):=\exp(-f_{\theta_t}(\bfx_i))\ge 0$. Since $\log(1+u_i(t)) \geq u_i(t)/(1+u_i(t))$, we have
\begin{align}
    \frac{d}{du}\left(\log(1+u_i(t)) - \frac{u_i(t)}{1+u_i(t)}\right)= \frac{u_i(t)}{(1+u_i(t))^2} \geq 0,
\end{align}
and now we have $c_i(t) \leq \ell(+1, f_{\theta_t}(\bfx_i))$. Then, we finally have
\begin{align}
    \|\bfx_c(t)\| &= \left\|\sum_{i\in \mathcal{I}_+}c_i(t)\bfx_i\right\| \\
    &\leq \sum_{i\in \mathcal{I}_+}c_i(t)\left\|\bfx_i\right\| \\&\leq \bfx_{\max}\sum_{i\in \mathcal{I}_+}\ell(+1, f_{\theta_t}(\bfx_i))\\
    &=\bfx_{\max} n \hat{L}(\theta_t),
\end{align}
and we get the claim.
\end{proof}

\begin{lemma}\label{lem:dummy2}
    We have
    \begin{align}
        G(t_{\eta,\alpha}) \leq \bfx_{\max}n\left((t_2-t_\alpha)\hat L(t_\alpha)+\frac{1}{\beta}\log\frac{\hat L(t_2)}{\eta}\right),
    \end{align}
    where $\beta:=(\lambda\bfx_{\min})^2/(32\bfx_{\max})$.
\end{lemma}
\begin{proof}
    By the definition of $G(t)$ and the results of \cref{lem:dummy}, we have
\begin{align}
    G(t_{\eta,\alpha}) &\leq \bfx_{\max}n\int_{t_\alpha}^{t_{\eta,\alpha}}\hat {L}(t)dt\\
     &= \bfx_{\max}n\left(\int_{t_\alpha}^{t_2}\hat L(t)dt+\int_{t_2}^{t_{\eta,\alpha}}\hat L(t)dt\right),\\
     & \leq \bfx_{\max}n\left((t_2-t_\alpha)\hat{L}(t_\alpha)+\int_{t_2}^{t_{\eta,\alpha}}\frac{\hat{L}(t_2)}{1+\beta\hat{L}(t_2)(t-t_2)}dt\right) \label{eq:t2}\\
     & = \bfx_{\max}n\left((t_2-t_\alpha)\hat L(t_\alpha)+\frac{1}{\beta}\log\left(1+\beta\hat L(t_2)(t_{\eta,\alpha}-t_2)\right)\right)\\
&\leq \bfx_{\max}n\left((t_2-t_\alpha)\hat L(t_\alpha)+\frac{1}{\beta}\log\frac{\hat L(t_2)}{\eta}\right).
\end{align}
Here, \cref{eq:t2} is due to monotonic decreasing property of the training risk and the results from \citet[Appendix~D.2]{min2024early}, and $\beta:=(\lambda\bfx_{\min})^2/(32\bfx_{\max})$.
\end{proof}

\newpage
\newpage
\subsection{Proof of \cref{thm:rb}}

For the proof, we will derive the upper bounds of the each term, $\mathsf{OA}(\alpha)$ and $\mathsf{OF}(\alpha)$, each, and then combine the terms. For notational simplicity, we will write $\bar{\bfx}_+ = \bfx_+ / \|\bfx_+\|$.

\textbf{Exact formula for the over-alignment, $\mathsf{OA(\alpha)}$.} 
Thanks to the closed-form expression of the zero-one error under the Gaussian mixture (\cref{lem:gm_risk}), we have
\begin{align}
    \inf_{\bfv \in H(\alpha)} \mathcal{E}(\mathbf{v}) = \Phi\left(-{\sigma}^{-1}{\bfv_*^{\top} \mathbf{s}_+}\right), \quad \mathrm{where} \quad \mathbf{v}_* = \arg\max_{\bfv\in H(\alpha)} \bfv^\top \mathbf{s}_+.
\end{align}
Moreover, $\mathcal{E}^*$ denotes the Bayes error, i.e.,
\begin{align}
    \mathcal{E}^*= \inf_{\bfv \in \bbR^d}\mathcal{E}(\bfv) = \Phi\left(-\sigma^{-1}\right).
\end{align} 
Combining these terms, we have
\begin{align}
    \mathsf{OA}(\alpha) = \Phi\left(-{\bfv_*^{\top} \mathbf{s}_+}\sigma^{-1}\right) -\Phi\left(-\sigma^{-1}\right), \label{eq:OA}
\end{align}
and we get what we want.

\textbf{Bounds for the over-fitting.} For the overfitting term, we derive upper bounds based on both Rademacher and Gaussian complexities, and take their minimum to obtain a tighter bound.

By the scale invariance of the zero-one error (\cref{lem:gm_risk}), we have
\begin{align}
    \mathsf{OF}(\alpha)  &= \mathcal{E}(\hat{\bfw}_\alpha) -  \inf_{\bfv \in H(\alpha)} \mathcal{E}(\bfv)  \\
    & = \mathcal{E}(\bar{\bfw}_\alpha) -  \inf_{\bfv \in H(\alpha)} \mathcal{E}(\bfv) \\
    & = \Phi \left(-{\sigma^{-1}}{\bar{\bfw}_\alpha^\top \mathbf{s}_+}\right) - \inf_{\bfv \in H(\alpha)} \Phi \left( 
    -{\sigma^{-1}}{\bfv^\top \mathbf{s}_+}\right) \\
    &= \Phi \left(-{\sigma^{-1}}{\bar{\bfw}_\alpha^\top \mathbf{s}_+}\right) - \Phi \left(-{\sigma}^{-1}{\sup_{\bfv \in H(\alpha)} \bfv^\top \mathbf{s}_+}\right),
\end{align}
where  $\bar{\bfw}_\alpha := \hat{\bfw}_\alpha/\|\hat{\bfw}_\alpha\|$. Since $\Phi'(t)=\exp({-t^2/2})/\sqrt{2\pi} \leq 1/\sqrt{2\pi}$, the mean value theorem implies that there exists some $z \in \left[-{\sigma}^{-1}{\sup_{\bfv \in H(\alpha)} \bfv^\top \mathbf{s}_+},-{\sigma}^{-1}{\bar{\bfw}_\alpha^\top \mathbf{s}_+}\right]$ such that
\begin{align}
    \Phi\left(-{\sigma^{-1}}{\bar{\bfw}_\alpha^\top \mathbf{s}_+}\right) - \Phi\left(-{\sigma^{-1}}{\sup_{\bfv \in H(\alpha)} \bfv^\top \mathbf{s}_+}\right)& = \Phi'(z)\left(-{\sigma^{-1}}{\bar{\bfw}_\alpha^\top \mathbf{s}_+}+{\sigma^{-1}}{\sup_{\bfv \in H(\alpha)} \bfv^\top \mathbf{s}_+}\right)\\ &\leq \frac{1}{\sqrt{2\pi}} \left(\frac{\sup_{\bfv \in H(\alpha)} \bfv^\top \mathbf{s}_+ - \bar{\bfw}_\alpha^\top \mathbf{s}_+}{\sigma}\right). \label{eq:dummy12}
\end{align}
By \cref{lem:logistic_pop}, we have $L(\bar{\bfw}_{\alpha})=\bbE_{G\sim \mathcal
N(0,1)}[\ell(\bar{\bfw}_{\alpha}^\top \mathbf{s}_+ + \sigma G)]$. Moreover, by the Cauchy-Schwarz inequality, we have $\bar{\bfw}_{\alpha}^\top \mathbf{s}_+\in [-\|\mathbf{s}_+\|, \| \mathbf{s}_+\|]$.

Now define the scalar function $\tilde{L}$ by $\tilde{L}(t) := L(\bar{\bfw}_{\alpha})$, where $t := \bar{\bfw}_{\alpha}^\top \mathbf{s}_+$. Note that the logistic loss function $\ell(u)$ is monotonically decreasing with respect to $u\in \mathbb{R}$. Then, since $t\leq \|\mathbf{s}_+\|$, we obtain
\begin{align}
    \tilde{L}'(t) &= -\bbE_{G\sim \mathcal
N(0,1)}\left[\frac{1}{1+\exp(t+\sigma G)}\right] \\
    &\leq -\bbE_{G\sim \mathcal
N(0,1)}\left[\frac{1}{1+\exp(\|\mathbf{s}_+\|+\sigma G)}\right]\\
    &\leq -\bbE_{G\sim \mathcal
N(0,1)}\left[\frac{1}{1+\exp(\|\mathbf{s}_+\|+\sigma G)}\mathbb{1}[G\leq0]\right] \\
    & \leq -\Pr(G\leq 0)\cdot \frac{1}{1+\exp(\|\mathbf{s}_+\|)} \label{eq:dummy10}\\
    & = -\frac{1}{2(1+\exp(\|\mathbf{s}_+\|))}.
\end{align}
Here, \cref{eq:dummy10} we use monotonicity of $1/(1+e^t)$. 

Using mean value theorem again for $\tilde{L}(t)$, we have 
\begin{align}
    &\tilde{L}\left(\sup_{\bfv \in H(\alpha)} \bfv^\top \mathbf{s}_+ \right) - \tilde{L}\left(\bar{\bfw}_\alpha^\top \mathbf{s}_+\right) \leq  -\frac{1}{2(1+\exp(\|\mathbf{s}_+\|))}\left(\sup_{\bfv \in H(\alpha)} \bfv^\top \mathbf{s}_+ -  \bar{\bfw}_\alpha^\top \mathbf{s}_+\right),
    \end{align}
which directly implies
    \begin{align}
    \sup_{\bfv \in H(\alpha)} \bfv^\top \mathbf{s}_+ -  \bar{\bfw}_\alpha^\top \mathbf{s}_+ \leq 2(1+\exp(\|\mathbf{s}_+\|))\left( \tilde{L}\left(\bar{\bfw}_\alpha^\top \mathbf{s}_+\right) - \tilde{L}\left(\sup_{\bfv \in H(\alpha)} \bfv^\top \mathbf{s}_+ \right)\right) .\label{eq:dummy11}
\end{align}
Combining \cref{eq:dummy12} and \cref{eq:dummy11}, we proceed as
\begin{align}
    \mathsf{OF}(\alpha) &\leq \frac{2(1+\exp(\|\mathbf{s}_+\|))}{\sigma\sqrt{2\pi}} \left( \tilde{L}\left(\bar{\bfw}_\alpha^\top \mathbf{s}_+\right) - \tilde{L}\left(\sup_{\bfv \in H(\alpha)} \bfv^\top \mathbf{s}_+ \right)\right)  \\
    & = \frac{2(1+\exp(\|\mathbf{s}_+\|))}{\sigma\sqrt{2\pi}} \left( \tilde{L}\left(\bar{\bfw}_\alpha^\top \mathbf{s}_+\right) - \inf_{\bfv\in H(\alpha)}\tilde{L}\left(\bfv^\top \mathbf{s}_+ \right)\right) \\
    &=\frac{2(1+\exp(\|\mathbf{s}_+\|))}{\sigma\sqrt{2\pi}} \left( {L}\left(\bar{\bfw}_\alpha\right) - \inf_{\bfv\in H(\alpha)}{L}\left(\bfv \right)\right) \\
    & \leq\frac{2(1+\exp(\|\mathbf{s}_+\|))}{\sigma\sqrt{2\pi}} \left(2\sup_{\bfv \in H(\alpha)} \left|L(\bfv) - \hat{L}(\bfv) \right| +\hat{L}(\bar{\bfw}_\alpha) - \inf_{\bfv\in H(\alpha)} \hat{L}(\bfv)\right)  \label{eq:risk_decomp} \\
    & \leq \frac{2(1+\exp(\|\mathbf{s}_+\|))}{\sigma\sqrt{2\pi}} \left(2\sup_{\bfv \in H(\alpha)} \left|L(\bfv) - \hat{L}(\bfv) \right|  + \eta \right). \label{eq:risk_decomp_2}
\end{align}
In \cref{eq:risk_decomp}, we use classical error decomposition technique. Moreover, \cref{eq:risk_decomp_2} is due to:
\begin{align}
    \hat{L}(\bar{\bfw}_{\alpha}) - \inf_{\bfv \in H(\alpha)} \hat{L}(\bfv) \leq\hat{L}\left(\hat{\bfw}_{\alpha}/\|\hat{\bfw}_{\alpha}\|\right) \leq \hat{L}\left(\hat{\bfw}_{\alpha}\right)=\eta,
\end{align}
where the inequalities are from $\|\hat{\bfw}_{\alpha}\|\leq 1$ and $\langle \hat{\bfw}_{\alpha}, \bfx_i\rangle\geq0$ for all $i\in\mathcal{I}_+$.

We now proceed to derive two different upper bounds for the uniform deviation, based on Rademacher and Gaussian complexities \citep{bartlett2002rademacher}, and take the minimum.  To bound this term, it is important to note that the logistic loss is sample-wise unbounded due to the unbounded support of Gaussian mixtures, which makes it tricky to apply classical Rademacher complexity-based bounds involving McDiarmid's inequality \citep{mjt_dlt}. Instead, we can utilize the Lipschitz property of the logistic loss (i.e., 1-Lipschitz with respect to margin), which is studied from  \citet{maurer2021concentration}.  Note that, since we are handling only the positive-class data, for simplicity, we can let $\bfa_i:= y_i\bfx_i$, which follows that $y_i\bfv^\top\bfx_i = \bfv^\top \bfa_i$ and $\bfa_i$ is the i.i.d. random variable from $\mathcal{N}(\bfs_+, \sigma^2 \Id)$. 

First, we define a \textit{data-independent} cone $\bar{H}(r)$
\begin{align}
    \bar{H}(r) :=\{\bfv \in \bbS^{d-1} : \langle \bfs_+,  \bfv \rangle \geq r \}
\end{align}
and the corresponding function class $\bar{\mathcal{H}}_r$
\begin{align}
    \bar{\mathcal{H}}_r:=\{\bfa\mapsto \ell(\bfv^\top \bfa)  : \bfv\in \bar{H}(r) \}
\end{align}
Moreover, we define a fixed grid $G_\epsilon$ as follows:
\begin{align}
    G_\epsilon := \left\{ -1+k\epsilon: k=0,1,\cdots, \lceil2/\epsilon \rceil\right\} \cap
    \left[-1,1\right].
\end{align}
For the proof, we set $\epsilon\in (0,0.25)$.
For each fixed $r\in G_\epsilon$, we first derive the uniform deviation to $\bar{\mathcal H}_r$.

\newpage
Let i.i.d. random variables $\bfa_1,\cdots,\bfa_n\sim \mathcal{N}(\bfs_+, \sigma^2 \Id)$ and the empirical covariance matrix
$\widehat{\Sigma}_n := \frac{1}{n} \sum_{i=1}^n \bfa_i \bfa_i^\top
$. Then, the upper bound of the (one-side) uniform deviation is
\begin{align}
 \sup_{\bfv \in \bar{H}(r)} \left(L(\bfv) - \hat{L}(\bfv) \right) &\leq \bbE_\bfa [R(\bar{\mathcal{H}}_{r}, (\bfa_1,\bfa_2, \cdots, \bfa_n))]+16e\cdot\mathrm{Lip}(\ell)\cdot\| \|\bfa_1\| \|_{\psi_1} \sqrt{\frac{\log (3/\delta\epsilon)}{n}} \\
        &\leq \bbE_{\bfa}\bbE_{\epsilon} \left[\frac{2}{n} \sup_{\bfv \in \bar{H}(r)}  \left(\sum_{i=1}^n \epsilon_i\ell(\bfv^\top \bfa_i)\right)\right] + 16e\| \|\bfa_1\| \|_{\psi_1} \sqrt{\frac{\log (3/ \delta\epsilon)}{n}} \label{eq:dummy103}\\
        & \leq \sqrt{\frac{\pi}{2}}\cdot\bbE_{\bfa}\underbrace{\bbE_{g\sim 
        \mathcal{N}(0,1)} \left[\frac{2}{n} \sup_{\bfv \in \bar{H}(r)}  \left(\sum_{i=1}^n g_i\ell(\bfv^\top \bfa_i)\right)\right]}_{=:\hat{\mathfrak G}_n\left(\bar{\mathcal H}_r\right)} + 16e\| \|\bfa_1\| \|_{\psi_1} \sqrt{\frac{\log (3/ \delta\epsilon)}{n}}\label{eq:dummy101}. \\
        & \leq  \sqrt{2\pi}\left(1+\sigma\left(1+\sqrt{\frac{d}{n}}\right)\right)\cdot 
\sqrt{
\frac{d\left(1-r^2\right)}{n}
}+ C_1(1+\sigma\sqrt{d})\sqrt{\frac{\log(3/\delta\epsilon)}{n}} \label{eq:gc_bound},
\end{align}
for some $C_1>0$. Here, $R(\cdot,\cdot)$ denotes the empirical Rademacher complexity and $\epsilon_i$ denotes the Rademacher random variable. Note that \cref{eq:dummy101} is from standard comparison inequality between Rademacher and Gaussian complexities \citep{ledoux1991probability}.


Taking a union bound for (1) $\bar{\mathcal{H}}_{r}$ and $-\bar{\mathcal{H}}_{r}$ and (2) over all $r\in G_\epsilon$, we obtain, with probability at least $1-\delta$,
\begin{align}
    \sup_{\bfv\in\bar H(r)}
    \left|
    L\left(\bfv\right)-\hat L\left(\bfv\right)
    \right|
    \leq \sqrt{2\pi}\left(1+\sigma\left(1+\sqrt{\frac{d}{n}}\right)\right)\cdot 
\sqrt{
\frac{d\left(1-r^2\right)}{n}
}+ C_1(1+\sigma\sqrt{d})\sqrt{\frac{\log(3/\delta\epsilon)}{n}} 
\end{align}

Now, after the training samples are realized, define $\tau(\alpha):=\inf_{\bfv\in H\left(\alpha\right)} \left\langle \bfs_+,\bfv\right\rangle$ and choose
\begin{align}
    r(\alpha)
    &:=
    \max
    \left\{
    r\in G_\epsilon:
    0<r\leq \tau(\alpha)
    \right\}.
\end{align}
Then, we have $r(\alpha)\in\left[\tau_\alpha-\epsilon,
\tau_\alpha\right]$.
Moreover, since $r(\alpha)\leq \tau(\alpha)$, we have
\begin{align}
    H\left(\alpha\right)
    &\subseteq
    \bar H\left(r(\alpha)\right).
\end{align}
Therefore, we obtain
\begin{align}
    \sup_{\bfv\in H\left(\alpha\right)}
    \left|
    L\left(\bfv\right)-\hat L\left(\bfv\right)
    \right|
    &\leq
    \sup_{\bfv\in \bar H\left(r(\alpha)\right)}
    \left|
    L\left(\bfv\right)-\hat L\left(\bfv\right)
    \right|
    \\
    &\leq \sqrt{2\pi}\left(1+\sigma\left(1+\sqrt{\frac{d}{n}}\right)\right)\cdot 
\sqrt{
\frac{d\left(1-r^2\right)}{n}
}+ C_1(1+\sigma\sqrt{d})\sqrt{\frac{\log(6/\delta\epsilon)}{n}} \label{eq:final_udev}.
\end{align}
Plugging \cref{eq:final_udev} into \cref{eq:risk_decomp_2}, we get the upper bound for $\mathsf{OF}(\alpha)$.

\paragraph{Combining $\mathsf{OA}$ and $\mathsf{OF}$.}
 As a final step, combining the derived upper bound for $\mathsf{OF}(\alpha)$ and \cref{eq:OA},
 we get what we want.

\newpage
\subsubsection{Detailed derivation of \cref{eq:gc_bound}}
Here, we derive a upper bound for the empirical Gaussian complexity $\hat{\mathfrak G}_n\left(\bar{\mathcal H}_r\right)$. By defining $\hat{\Sigma}_n:=\frac{1}{n}\sum_{i=1}^n \bfa_i \bfa_i^\top$, we have
\begin{align}
\bbE_\bfa \hat{\mathfrak G}_n\left(\bar{\mathcal H}_r\right)
&:=
\bbE_\bfa \mathbb E_g
\left[
\frac{2}{n}
\sup_{\bfv\in\bar H\left(r\right)}
\sum_{i=1}^n
g_i\ell\left(\bfv^\top \bfa_i\right)
\,\middle|\,
\bfa_1,\ldots,\bfa_n
\right] \\
&=
\bbE_\bfa \mathbb E_g
\left[
\frac{2}{n}
\sup_{\bfv\in\bar H\left(r\right)}
\sum_{i=1}^n
g_i\left(
\ell\left(\bfv^\top \bfa_i\right)-\ell(0)
\right)
\,\middle|\,
\bfa_1,\ldots,\bfa_n
\right] \\
&\leq
\bbE_\bfa \mathbb E_g
\left[
\frac{2}{n}
\sup_{\bfv\in\bar H\left(r\right)}
\sum_{i=1}^n
g_i \bfv^\top \bfa_i
\,\middle|\,
\bfa_1,\ldots,\bfa_n
\right] \\
&=
\frac{2}{n} \bbE_\bfa
\mathbb E_g
\left[
\sup_{\bfv\in\bar H(r)}
\left\langle
\bfv,
\sum_{i=1}^n g_i\bfa_i
\right\rangle
\,\middle|\,
\bfa_1,\ldots,\bfa_n
\right] \\
&=
\frac{2}{\sqrt n} \bbE_\bfa
\mathbb E_{\bfg:=[g_1,\cdots,g_n]^\top}
\sup_{\bfv\in\bar H(r)}
\left\langle
\hat\Sigma_n^{1/2}\bfg,\bfv
\right\rangle \\
&\leq
\frac{2}{\sqrt n} \bbE_\bfa
\left\|\hat\Sigma_n\right\|^{1/2}
\mathbb E_{\bfg}
\sup_{\bfv\in\bar H(r)}
\left\langle
\bfg,\bfv
\right\rangle \label{eq:sf}\\
&\leq 
\frac{2}{\sqrt n} \bbE_\bfa
\left\|\hat\Sigma_n\right\|^{1/2}
\sqrt{d(1-r^2)}\label{eq:width} \\
& \leq \frac{2}{\sqrt n}
\left(1+\sigma\left(1+\sqrt{\frac{d}{n}}\right)\right)
\sqrt{d(1-r^2)} \label{eq:cov}
\end{align}
Here, the detailed derivations are as follows: We use Sudakov-Fernique inequality in \cref{eq:sf}. 

For the \cref{eq:width}, we decompose $\bfv$ as 
\begin{align}
 \bfv=\sqrt{1-\|\bfz\|^2}\bfs + \bfz\quad \mathrm{where}  \quad \bfz \perp \bfs, \quad\|\bfz\|\leq\sqrt{1-r^2}.
\end{align}
In a similar manner, we can write the Gaussian vector $\bfg$ as
\begin{align}
    \bfg = \langle\bfg, \bfs \rangle\bfs + \bfg' \quad \mathrm{where} \quad \bfg' \perp \bfs. \label{dummy301}
\end{align}
From \cref{dummy301}, we have $ \langle\bfg, \bfs \rangle \sim \mathcal{N}(0,1)$ and $\bfg_{\perp} \sim \mathcal{N}(0, \mathbf{I}_{d-1})$. Then, we obtain
\begin{align}
    \mathbb E_{\bfg}
\sup_{\bfv\in\bar H(r)}
\left\langle
\bfg,\bfv
\right\rangle &= \bbE \sup_{\bfz \perp \bfs,\|\bfz\|\leq \sqrt{1-r^2}} \left[\langle\bfg, \bfs \rangle \left(\sqrt{1-\|\bfz\|^2}\right) + \langle \bfg', \bfz \rangle\right] \\
& = \bbE \sup_{\bfz \perp \bfs,\|\bfz\|\leq \sqrt{1-r^2}} \left[\langle\bfg, \bfs \rangle \left(\sqrt{1-\|\bfz\|^2}-1\right) + \langle \bfg', \bfz \rangle\right] \\
& \leq  \sqrt{1-r^2}\sqrt{d-1} + \frac{1}{\sqrt{2\pi}}(1-r) \\
& \leq \sqrt{1-r^2}\sqrt{d-1} + \frac{1}{\sqrt{2\pi}}\sqrt{1-r^2} \\
& =  \sqrt{1-r^2} \left(\sqrt{d-1}+\frac{1}{\sqrt{2\pi}}\right) \\
& \leq  \sqrt{d(1-r^2)}.
\end{align}
\newpage
It remains to bound the spectral norm of the empirical covariance to dervie \cref{eq:cov}. 
Let $\bfA:=[\bfa_1,\cdots, \bfa_n]$, then we have $\widehat{\Sigma}_n = \frac{1}{n}\bfA\bfA^\top$. We proceed as 
\begin{align}
    \bbE_\bfa\|\widehat{\Sigma}_n\|^{1/2}_2 &= \frac{1}{\sqrt{n}}\bbE_\bfa\|\bfA\|_2 \\
    & = \frac{1}{\sqrt{n}}\bbE_\bfa\|\bfs \mathbf{1}_n^\top + \sigma \mathbf{Z}\|_2 \\
    & \leq \frac{1}{\sqrt{n}} \bbE_\bfa \left[ \sqrt{n} + \sigma \|\mathbf{Z}\|_2\right] \\
    & \leq \frac{1}{\sqrt{n}} \bbE_\bfa \left[ \sqrt{n} + \sigma (\sqrt{n} +\sqrt{d})\right] \\
    & = 1+\sigma\left(1+\sqrt{d/n}\right).
\end{align}
where $\mathbf{Z}$ has each entry as $\mathcal{N}(0,1)$. 

As a final step, we apply \cref{lem:subexp_norm} to bound $\| \|\bfa_1\| \|_{\psi_1}$, 
and we get what we want.

\newpage
\subsection{Proof of \cref{prop:propor}}
For the empirical mean, we have
\begin{align}
    \bar{\bfx}_+ = \frac{1}{n} \sum_{i=1}^n \bfx_i = \kappa \bfs_+ + \sigma \frac{1}{n}\sum_{i=1}^n \bfz_i.
\end{align}
We decompose $\bar{\bfz}:=\frac{1}{n}\sum_{i=1}^n \bfz_i$ into
\begin{align}
    \bar{\bfz} = \langle \bar{\bfz}, \bfs_+ \rangle\bfs_+ + \Pi^\perp_{\bfs_+} \bar{\bfz},
\end{align}
where $\Pi^\perp_\bfs \bar{\bfz}$ is component of $\bar{\bfz}$ orthogonal to $\bfs_+$. Then, in terms of $\phi=\angle(\bar{\bfx}_+, \bfs_+)$, we can write as
\begin{align}
    \tan \phi  = \frac{\sigma \|\Pi^\perp_{\bfs_+} \bar{\bfz}\|}{\kappa + \sigma \langle \bar{\bfz}, \bfs_+ \rangle}.
\end{align}
Now, for $\delta\in(0,1)$, let $t= \log(4/\delta)$. Since $\langle \bar{\bfz}, \bfs_+ \rangle\sim \mathcal{N}(0,1/n)$, using standard Gaussian tail bound, we have
\begin{align}
    \Pr\left(|\langle \bar{\bfz}, \bfs_+ \rangle| \leq \sqrt{2t/n}\right) \geq 1- \delta/2. \label{eq:dummy201}
\end{align}
Next, by the fact that $n\|\Pi^\perp_{\bfs_+} \bar{\bfz}\|^2 \sim \chi^2_{d-1}$, from \citet{laurent2000adaptive}, with probability at least $1-\delta/2$, we have
\begin{align}
    d-1-2\sqrt{(d-1)t} \leq n\|\Pi^\perp_{\bfs_+} \bar{\bfz}\|^2 \leq d-1+2\sqrt{(d-1)t} + 2t. \label{eq:dummy202}
\end{align}
Taking a union bound, we have, with probability at least $1-\delta$, \cref{eq:dummy201,eq:dummy202} hold. On this event, if $\kappa/\sigma >\sqrt{2t/n}$, we proceed as
\begin{align}
&\frac{
\max\left\{0,d-1-2\sqrt{\left(d-1\right)t}\right\}
}{
n\left(\kappa/\sigma+\sqrt{2t/n}\right)^2
}
\leq
\tan^2\phi
\leq
\frac{
d-1+2\sqrt{\left(d-1\right)t}+2t
}{
n\left(\kappa/\sigma-\sqrt{2t/n}\right)^2
}
\\
\implies\quad
&
\frac{\sqrt{\gamma_1}}{\gamma_2}
\cdot
\frac{
\max\left\{0,d-1-2\sqrt{\left(d-1\right)t}\right\}
}{
d\left(
1+\sqrt{
\frac{2t}{n\gamma_2\sqrt{d\log n}}
}
\right)^2
}
\leq
\tan^2\phi
\leq
\frac{\sqrt{\gamma_1}}{\gamma_2}
\cdot
\frac{
d-1+2\sqrt{\left(d-1\right)t}+2t
}{
d\left(
1-\sqrt{
\frac{2t}{n\gamma_2\sqrt{d\log n}}
}
\right)^2
},
\label{eq:three}
\end{align}
where we use $\frac{d/n}{\kappa^2/\sigma^2}=\frac{\sqrt{\gamma_1}}{\gamma_2}$.

Consider the regime where $d,n\to\infty$ and
$\gamma_2\to\gamma_{2,\infty}\in\left(0,\infty\right)$, we have
\begin{align}
\frac{
\max\left\{0,d-1-2\sqrt{\left(d-1\right)t}\right\}
}{d}\to
1, \qquad
\frac{
d-1+2\sqrt{\left(d-1\right)t}+2t
}{d}
\to
1, \qquad
\sqrt{
\frac{2t}{n\gamma_2\sqrt{d\log n}}
}\to
0.
\end{align}
Therefore, we have
\begin{align}
\tan^2\phi
\to
\frac{\sqrt{\gamma_1}}{\gamma_2}
.
\end{align}

Now we consider three different regimes introduced in \cref{prop:propor}.

\textbf{1. Data-abundant regime.} If $\gamma_1\to0$, then $\tan^2\phi\to0$, and $\phi
\to
0$, with high probability.

\textbf{2. Moderate regime.} If $\gamma_1\to\gamma_{1,\infty}\in\left(0,\infty\right)$, then $\tan^2\phi
\to
{\sqrt{\gamma_{1,\infty}}}/{\gamma_{2,\infty}}$
with high probability. Therefore, we obtain
$\phi
\to
\arctan\left(
{\gamma_{1,\infty}^{0.25}}/{\gamma_{2,\infty}^{0.5}}
\right)$.

\textbf{3. High-dimensional regime.}
If $\gamma_1\to\infty$, then $\tan^2\phi\to
\infty$ with high probability, and hence $
\phi \to
{\pi}/{2}$,

and this completes the proof.

\newpage
\section{Technical Lemmata and Known Results}

\begin{lemma}[Gradient flow properties]\label{lem:gf_property} For any $j\in[h]$ and $t\geq0$, we have
\begin{itemize}[leftmargin=*,topsep=0pt,parsep=0pt]
    \item (Balancedness, from \citet{du2018algorithmic}.) $v_j(t)^2 - \|\bfw_j(t)\|^2 = 0$.
    \item (Sign preservation, from \citet{boursier2022gradient}.) $\mathrm{sign}(v_j(t))=\mathrm{sign}(v_j(0))$.
\end{itemize}
\end{lemma}
\begin{proof} See each paper for the proof.
\end{proof}
\vspace{80pt}

\begin{lemma}\label{lem:techlem1}
Let $\phi(t)$ be a differentiable function satisfying $\dot{\phi}(t) \ge b\left(c^2 - \phi^2(t) \right)$ for some $b,c > 0$. Then, we have:
\begin{equation}
    \frac{d}{dt} \arctanh\left(\phi(t) / c\right) \ge bc.
    \label{eq:conclusion}
\end{equation}
\end{lemma}

\begin{proof}
First, note that the derivative of the $\arctanh(x)$ is $1/(1-x^2)$. Then, by the chain rule,
\begin{align}
    \frac{d}{dt} \arctanh(\phi(t)/c) 
    \quad=\quad \frac{c^2}{c^2 - \phi^2(t)} \cdot \frac{d}{dt}(\phi(t)/c) \quad=\quad \frac{c}{c^2 - \phi^2(t)} \dot{\phi}(t) \quad\ge\quad bc
\end{align}
where the last inequality follows from the assumption.
\end{proof}

\vspace{80pt}
\begin{lemma}[Lemma 3 and 4 of \citet{min2024early}]\label{lem:lem3_min} Consider (sub)gradient flow optimization as specified in \cref{ssec:setup} and the two-layer ReLU network is initialized with  scale $\alpha\leq \frac{1}{4\sqrt{h}\bfx_{\max}\mathbf{W}^2_{\max}}$. For any $t\leq t_{\alpha}:=\frac{1}{4n \mathbf{x}_{\max}}\log \frac{1}{\sqrt{h}\alpha} $ and $i\in \mathcal{I}_+$, we have 
\begin{align}
    \left\| \frac{d}{dt}\frac{\bfw_j(t)}{\|\bfw_j(t)\|} - \sign  (\bfv_j(0)) \left(\mathbf{I}_h - \frac{\bfw_j(t) \bfw_j(t)^\top}{\|\bfw_j(t)\|^2} \right) \left(\sum_{i=1}^n \mathbf{x}_i y_i \sigma'(\langle \bfx_i, \bfw_j(t)\rangle) \right)\right\| \leq 2n\mathbf{x}_{\max} \max_i |f(\mathbf{x}_i;\mathbf{W}(t), \mathbf{v}(t))|. \label{eq:alignment_gf}
\end{align}
\end{lemma}
\begin{proof}
    See \citet{min2024early} for the proof.
\end{proof}


\newpage
\begin{lemma}\label{lem:gm_risk} Let $(\bfx, y)$ be the drawn sample from the data model from \cref{ssec:setup}, for any unit vector (predictor) $\bfw \in \mathbb{S}^{d-1}$ with scalar multiplier $c>0$, we have
\begin{align}
    \mathrm{Pr}(\mathrm{sgn}(c\bfw^\top \bfx) \neq y) =\Phi\left(-\frac{\bfw^\top \mathbf{s}}{\sigma}\right).
\end{align}
\end{lemma}
\begin{proof}
An error occurs iff $y c\bfw^\top \bfx \le 0$. By the assumption, we have $\bfx=y \mathbf{s} + \sigma \bfz$,
\begin{align}
    y c\bfw^\top \bfx
= y c\bfw^\top (y \mathbf{s} + \sigma \bfz)
= c\bfw^\top \mathbf{s} + \sigma y c\bfw^\top \bfz.
\end{align}
Since $\bfz\sim\mathcal{N}(0,\Id)$ and $\|\bfw\|=1$, we have $\bfw^\top \bfz\sim\mathcal{N}(0,1)$.
Moreover, $y$ is independent of $\bfz$ and $y\in\{\pm1\}$, so $y\bfw^\top \bfz \stackrel{d}{=}  \bfw^\top \bfz$.
Hence we may write
\begin{align}
\Pr\left(\mathrm{sign}(c\bfw^\top \bfx)\neq y\right)
&= \Pr\left( \bfw^\top \mathbf{s} + \sigma G \le 0\right),
\qquad G\sim\mathcal{N}(0,1)    \\
&= \Pr\left(G \le -\frac{\bfw^\top \mathbf{s}}{\sigma}\right)\\
&= \Phi\left(-\frac{\bfw^\top \mathbf{s}}{\sigma}\right),
\end{align}
and this completes the proof. Note that, since we are interested in positive-class data, this can be handled in exactly the same way.
\end{proof}

\vspace{50pt}
\begin{lemma}\label{lem:logistic_pop} Suppose the data follow the setup specified in \cref{ssec:setup}. Let $L(\cdot)$ be the population logistic risk and $\ell(\cdot)$ be the logistic loss function. Then, whenever $\|{\bfw}\|=1$, we have
\begin{align}
    L({\bfw}) = \bbE_{G\sim \mathcal{N}(0,1)} \left[\ell({\bfw}^\top \bfs +\sigma G )\right].
\end{align}
\end{lemma}
\begin{proof}
    We have $y{\bfw}^\top \bfx = \bfw^\top \bfs + \sigma \bfw^\top \bfz$. Since $\bfz \sim \mathcal{N}(\mathbf{0}, \Id)$, we have $\bfw^\top \bfz \sim \mathcal{N}(0,1)$. Substituting the term, we get the claim.
\end{proof}

\vspace{50pt}
\begin{theorem}[Theorem 9 of \citet{maurer2021concentration}]\label{thm:maurer_thm9} Let $X=(X_1, \cdots, X_n)$ be i.i.d. random variables with values in a Banach space $(\mathcal{X}, \|\cdot\|)$ and $\mathcal{H}=\{h:\mathcal{X}\to \bbR\}$ such that $h(\cdot)$ is $L$-Lipschitz for all $\bfx,y\in \mathcal{X}$ and $h\in \mathcal{H}$. If $n\geq \log(1/\delta)$ then with probability at least $1-\delta$, we have
\begin{align}
    \sup_{h\in \mathcal{H}}\left(\frac{1}{n}\sum_{i=1}^n h(X_i) - \bbE[h(X)] \right) \leq \bbE[R(\mathcal{H}, X)] +16eL \| \|X_1\|\|_{\psi_1} \sqrt{\frac{\log(1/\delta)}{n}},
\end{align}
where $R(\mathcal{H}, X):=\bbE\left[\frac{2}{n}\bbE\left[\sup_{h\in \mathcal{H}}\sum_{i}\epsilon_ih(X_i)|X\right]\right]$ denotes the empirical Rademacher complexity and $\|\cdot\|_{\psi_1}$ denotes the sub-exponential norm.
\end{theorem}
\begin{proof}
    Check the original paper for the proof.
\end{proof}

\newpage
\begin{lemma}\label{lem:subexp_norm} Let $\bfx = \bfs + \sigma \bfz$, where $\|\bfs\|=1$, $\sigma>0$, and $\bfz\sim \mathcal{N}(\mathbf{0},\mathbf{I}_d)$.
Then
\begin{align}
    \|\|\bfx\|\|_{\psi_1} = O\left(1 +\sigma\sqrt{d}\right).
\end{align}
\end{lemma}
\begin{proof} 
At first, by the triangle inequality of the Euclidean norm, we have $\|\bfx\| \leq \|\bfs\| +\sigma\|\bfz\| = 1+\sigma\|\bfz\|$. Then, using the triangle inequality for the sub-exponential norm, we obtain
\begin{align}
    \| \|\bfx\| \|_{\psi_1} \leq \|1\|_{\psi_1} + \sigma \|\|\bfz\|\|_{\psi_1}=\frac{1}{\ln2} + \sigma \|\|\bfz\|\|_{\psi_1}.
\end{align}
Thus, to derive the bound, it suffices to upper bound $\|\|\bfz\|\|_{\psi_1}$. By \citet[Theorem~3.1.1]{vershynin2018high}, we have
\begin{align}
    \left\|\|\bfz\| -\sqrt{d}\right\|_{\psi_2} \leq C \max_{i}\|\bfz_i\|_{\psi_2},
\end{align}
for some $C>0$. Here, $\|\cdot\|_{\psi_2}$ denotes the sub-Gaussian norm. Since $\max_{i}\|\bfz_i\|_{\psi_2}=O(1)$ \citep{vershynin2018high}, it follows that $\left\|\|\bfz\| -\sqrt{d}\right\|_{\psi_2} = O(1)$. By the inequality: $\|\|\bfz\| -\sqrt{d}\|_{\psi_1}\lesssim \|\|\bfz\| -\sqrt{d}\|_{\psi_2}$, we have
\begin{align}
      \quad \left\|\|\bfz\| -\sqrt{d}\right\|_{\psi_1} = O(1).
\end{align}
Using triangle inequality yields
\begin{align}
    \| \|\bfz\| \|_{\psi_1} \leq \|\sqrt{d}\|_{\psi_1} + \| \|\bfz\| - \sqrt{d}\|_{\psi_1} = \frac{\sqrt{d}}{\ln 2}+ O(1) = O(\sqrt{d}).
\end{align}
Combining the terms, we get the claim~\footnote{Note that the bound can be refined to $\Theta(\cdot)$, but we defer doing so for the sake of simplicity.}.
\end{proof}

    

    

\newpage
\newpage
\section{Supplementary Theoretical Results}\label{app:sec:rb_details}
In this section, we provide probabilistic bounds for orthogonal separability under the Gaussian mixture model and present additional results that complement the theoretical analyses in \cref{sec:theory}.

\subsection{Proof about \cref{assm:ortho_sep}}
\begin{proposition}\label{prop:assm_just} There exist universal constants $c_1, c_2, c_3 > 0$ such that the following holds:
Consider the data model from \cref{ssec:setup}. If $d \ge c_1 \log n$ and $\kappa^2 \ge c_2 \sigma^2 \sqrt{d \log n}$,
then, with probability at least $1 - 4/n^2$, the training dataset is orthogonally separable, i.e., $\lambda>0$.
\end{proposition}

\begin{proof}
    To show the orthogonal separability of the data, it suffices to show $\langle \bfz_i, \bfz_j\rangle  >0$, for all $i,j\in [n]$, where $\bfz_i := \bfy_i \bfx_i$.

    First, let $a:=\sqrt{6\log n}$. Then, for each $i$, we have $\Pr(|\bfs^\top \zeta_i|>a) \leq 2/n^{3}$, where $\zeta_i\sim\mathcal{N}(0,I_d)$ are the noise vectors satisfying $\bfz_i=\kappa\bfs+\sigma\zeta_i$. Taking a union bound, we obtain
    \begin{align}
        \Pr\left(\max_{i\in[n]} |\bfs^\top \zeta_i| > a \right) \leq 2/n^2.
    \end{align}
    Using 1-Lipschitzness of $\|\cdot\|$ and standard Gaussian concentration inequality yield $\Pr(\|\zeta_i\|>\sqrt{d}+a)\leq 1/n^3$. Again, with union bound
    \begin{align}
        \Pr\left(\max_{i\in[n]} \| \zeta_i\| > \sqrt{d}+a \right) \leq 1/n^2.
    \end{align}
    Next, we are interested in pariwise inner product, i.e., $\langle \zeta_i, \zeta_j \rangle$. Using Bernstein's inequality, we have $\Pr\left(|\langle \zeta_i, \zeta_j \rangle| > c_3(\sqrt{d\log n}+\log n)\right)\leq 2/n^6$. Applying union bound, we get
    \begin{align}
        \Pr\left(\max_{i\neq j}|\langle \zeta_i, \zeta_j \rangle| > c_3(\sqrt{d\log n}+\log n)\right)\leq 1/n^2.
    \end{align}
Now, we aggregate the results. Let $E$ be the event on which all bounds from above hold simultaneously. Then we have $\Pr(E)\geq 1-4/n^2$. By the definition of $\bfz_i$, we proceed as follows. For some $C>0$, we have
\begin{align}
    {\langle \bfz_i, \bfz_j \rangle} &= {\kappa^2 +\kappa \sigma\bfs^\top \left(  \zeta_i +\zeta_j\right) +\sigma^2 \langle \zeta_i, \zeta_j\rangle} \\
    &\geq {\kappa^2-2\kappa\sigma \sqrt{6 \log n}-
c_3\sigma^2(\sqrt{d\log n}+\log n)} \\
& \geq {\kappa^2-C\kappa\sigma \sqrt{\log n}-
C\sigma^2(\sqrt{d\log n}+\log n)} \label{eq:dummy20}.
\end{align}
Since $d\geq c_1 \log n$, we have $\log n \lesssim \sqrt{d\log n}$. Moreover, using Young's inequality yields the numerator of \cref{eq:dummy20} to be positive: 
\begin{align}
    \kappa^2 \geq c_2 \sigma^2 \sqrt{d\log n},
\end{align}
for some $c_2>0$. This concludes the proof.
\end{proof}

\begin{remark}
    \cref{prop:assm_just} suggests that, under the data model in \cref{ssec:setup}, if the data dimension $d$ and the signal strength $\kappa$ is sufficiently large relative to the noise level $\sigma$, then with high probability, orthogonal separability satisfied.
\end{remark}

\newpage

\newpage
\subsection{Experiments on Phase 2 Dynamics}\label{app:p2_discuss}

\begin{figure*}[!htbp]
    \centering
    \begin{subfigure}[t]{0.33\textwidth}
        \centering
        \includegraphics[width=\linewidth]{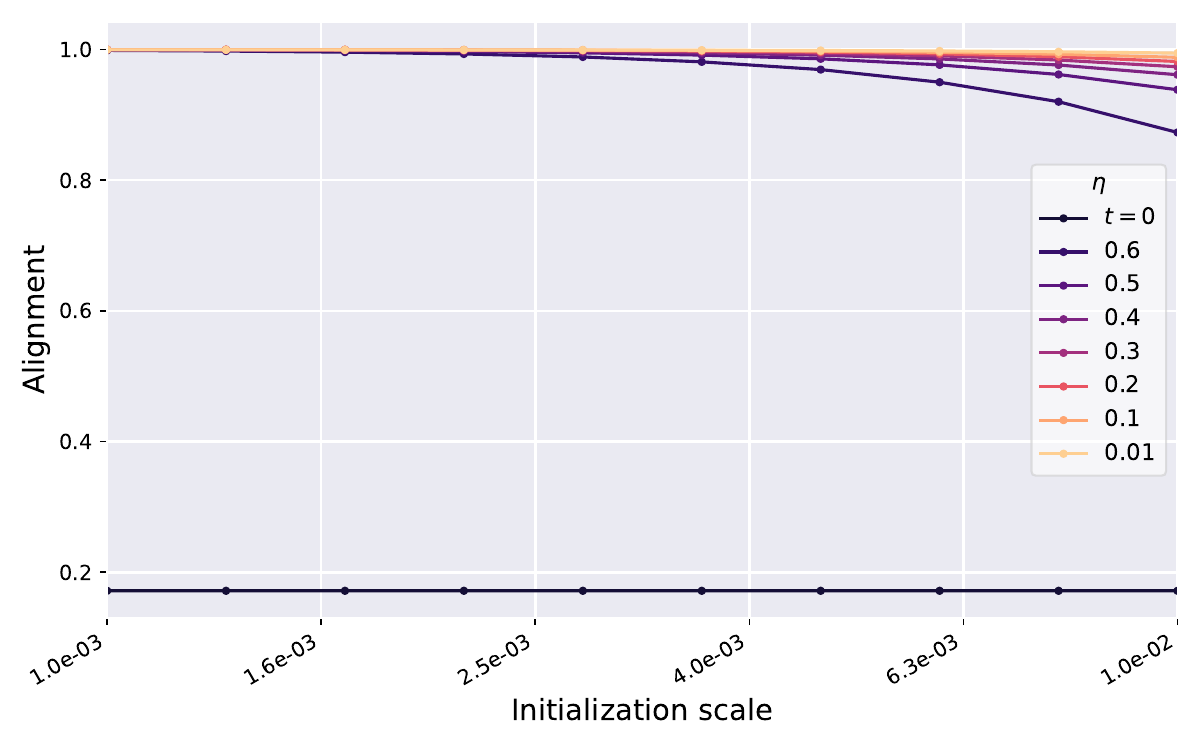}
        \caption{Initialization scale vs. $\Psi_+(t)$}
        \label{fig:psi_vs_alpha}
    \end{subfigure}\hfill
    \begin{subfigure}[t]{0.33\textwidth}
        \centering
        \includegraphics[width=\linewidth]{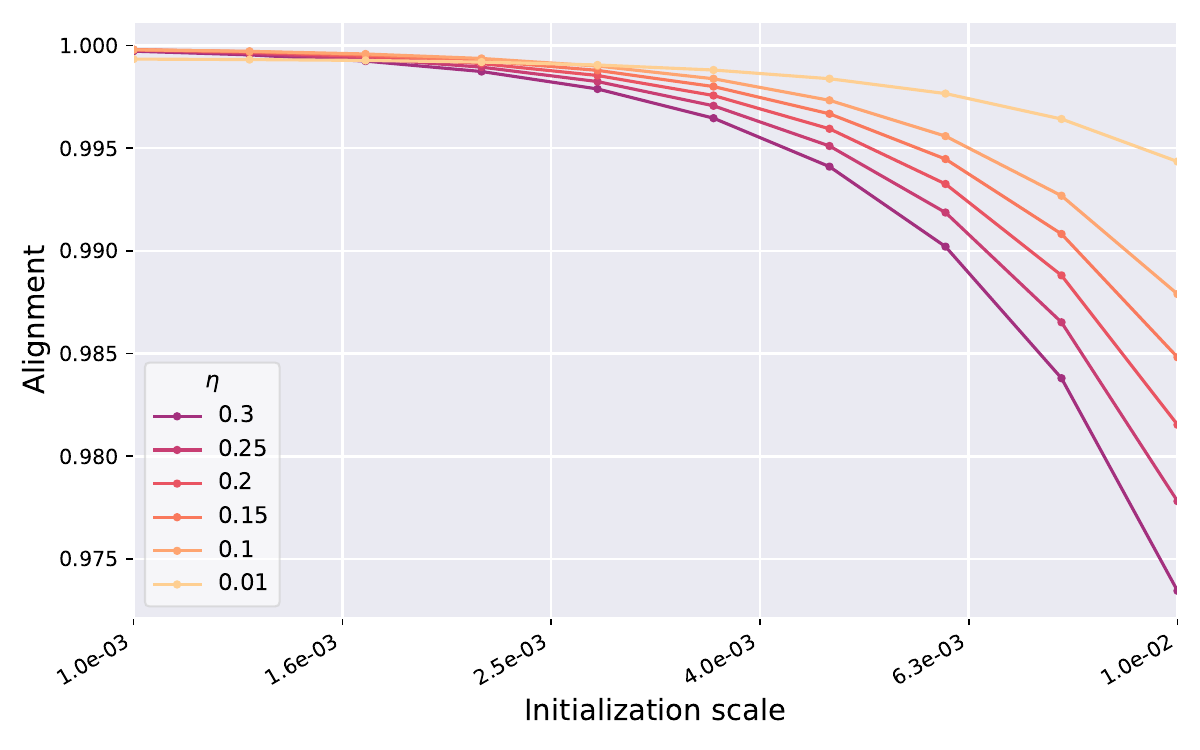}
        \caption{Initialization scale vs. $\Psi_+(t)$: Small $\eta$}
        \label{fig:psi_vs_alpha_small_eta}
    \end{subfigure}\hfill
    \begin{subfigure}[t]{0.33\textwidth}
        \centering
        \includegraphics[width=\linewidth]{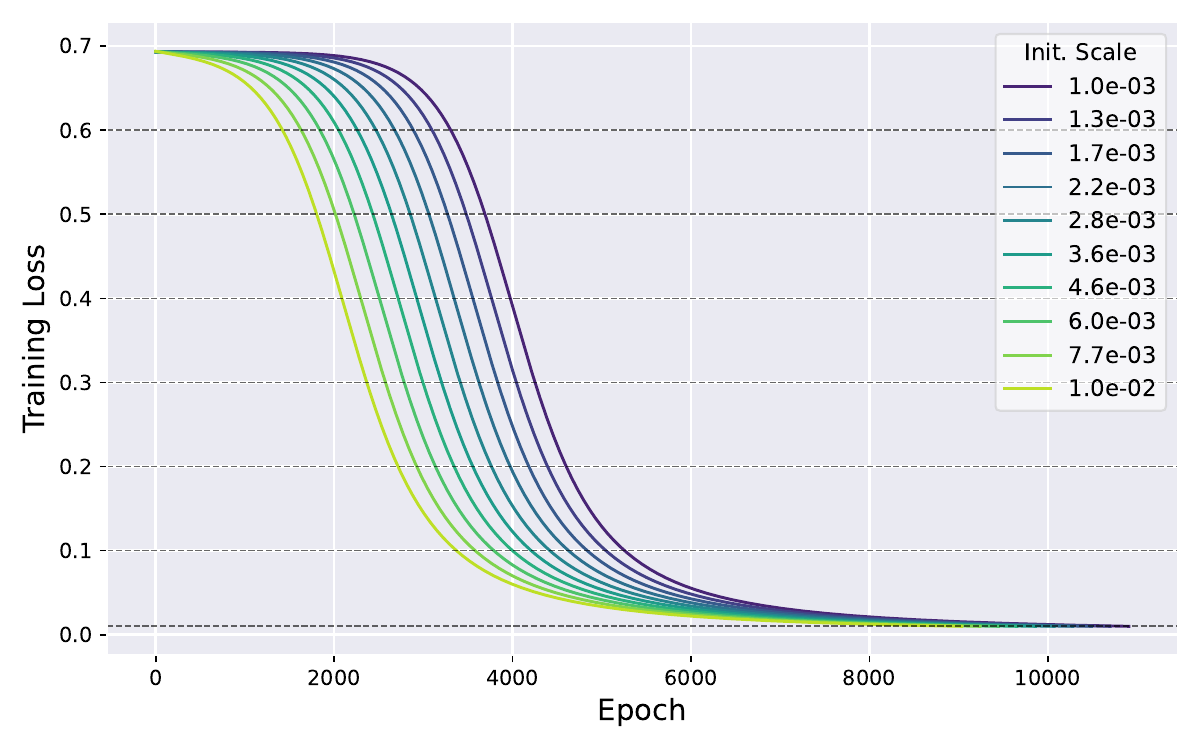}
        \caption{Epoch vs. training loss}
        \label{fig:epoch_vs_loss}
    \end{subfigure}
    \caption{\textbf{Phase 2 results.}}
    \label{fig:p2_results}
\end{figure*}

Following \cref{ssec:phase2}, this section presents the alignment results obtained in Phase 2. For completeness, we recall \cref{lem:phase2_bound}:

\begin{lemma}[\cref{lem:phase2_bound} in the paper] Let $\beta:=(\lambda\bfx_{\min})^2/(32\bfx_{\max})$ and $t_2 = O(\log(1/\alpha)/n)$. Then, for any risk threshold $\eta >0$, we have
    \begin{align}
        \psi_j(t_{\eta,\alpha}) \geq \lambda +  m(\alpha)\exp(-g(\alpha)),
    \end{align}
    where $m(\alpha):=\psi_j(t_\alpha)-\lambda$ and $g(\alpha)\leq\bfx_{\max}n_+\left((t_2-t_\alpha)\hat L(t_\alpha)+\frac{1}{\beta}\log\frac{\hat{L}(t_2)}{\eta}\right)$.
\end{lemma}
From the results of Phase 1, we have that $m(\alpha)$ increases as $\alpha$ decreases. Thus, our main interest is in the term $g(\alpha)$. By the definition of $t_2$, which the time that loss significantly decrease, we regard $t_2\approx t_\alpha$, hence $g(\alpha)\leq \bfx_{\max}n_+(\log( \hat{L}(t_2)/\eta))/\beta)\approx O(1)$, with respect to $\alpha$, thus, we have, in a approximate sense,
\begin{align}
    \psi_j(t_{\eta,\alpha}) \approx \psi_j(t_{\alpha}). \label{eq:samealignment}
\end{align}

\textbf{Interpretation.} \cref{eq:samealignment} suggests that after the alignment phase (Phase 1), the alignment changes very little. To validate this, we examine $\Psi(t_{\eta,\alpha})$—the value of $\Psi(t)$ when each network (for various initialization scales) reaches a target loss level $\eta$. The results are shown in \cref{fig:psi_vs_alpha}. As can be seen, across all initialization scales, $\Psi$ is already large by the time the loss starts to decrease (see also \cref{fig:epoch_vs_loss}). More precisely, although all initialization scales start from nearly the same alignment value at $t=0$, the alignment jumps to a large value immediately after the alignment phase (i.e., at $\eta=0.6$) and then remains at similar values thereafter. We also observe that $\Psi(t)$ decreases sub-linearly with respect to $\alpha$, which is consistent with \cref{cor:phase1_angle}. This tendency persists even as the target $\eta$ decreases (see \cref{fig:psi_vs_alpha_small_eta}), indicating that \cref{eq:samealignment} holds approximately.

\textbf{Details of \cref{fig:p2_results}.} For training, we use 300 training samples generated from a Gaussian mixture model with $\kappa=2$, $\sigma=1$, and $\lambda=0$. We use two-layer, bias-free ReLU networks with 64 hidden units, and train all models until the training risk reaches $\eta=0.01$.

\end{document}